\documentclass[10pt, journal]{IEEEtran}

\IEEEoverridecommandlockouts

\usepackage{amsthm}
\usepackage{times}
\usepackage{multicol}
\usepackage[bookmarks=true]{hyperref}
\usepackage{xcolor}
\usepackage{hyperref}
\usepackage{amsmath, amssymb}
\usepackage{amsfonts}
\usepackage{graphicx}
\usepackage{siunitx}
\usepackage{standalone}
\usepackage{booktabs}
\usepackage[ruled,vlined,linesnumbered,noend]{algorithm2e}
\usepackage{mdframed}
\usepackage{fancyvrb,multirow}
\usepackage{soul}
\usepackage{dsfont,mathabx}
\usepackage{array, booktabs}
\usepackage{subfigure}
\usepackage{booktabs}
\usepackage{makecell}
\usepackage{threeparttable}
\usepackage{hyperref}
\hypersetup{
    colorlinks=true,
    linkcolor=blue,
    filecolor=blue,
    urlcolor=blue,
    citecolor=blue,
}

\newtheorem{theorem}{Theorem}

\newtheorem{lemma}[theorem]{Lemma}
\theoremstyle{definition}

\newtheorem{definition}{Definition}
\newtheorem{remark}{Remark}
\newtheorem{example}{Example}
\newtheorem{problem}{Problem}

\title{
  HECTOR: Human-centric Hierarchical Coordination and Supervision of Robotic Fleets
  \\ under Continual Temporal Tasks
}

\author{Shen Wang$^1$, Yinhang Luo$^1$, Jie Li$^2$ and Meng Guo$^1$
\thanks{
	The authors are with
	$^1$ the School of Advanced Manufacturing and Robotics, Peking University, Beijing 100871, China;
	and $^2$National University of Defense Technology, Hunan 410073, China.
    This work was supported by the National Natural Science Foundation
    of China (NSFC) under grants U2241214, T2121002.
    Corresponding author: Meng Guo, {\tt\small meng.guo@pku.edu.cn}.
}
}

\begin{document}
\maketitle
\thispagestyle{empty}
\pagestyle{empty}

\begin{abstract}
{Robotic fleets can} be extremely efficient when working concurrently
and collaboratively,
e.g., for delivery, surveillance, search and rescue.
However, it can be demanding or even impractical for an operator
to directly control each robot.
Thus, autonomy of the fleet and its online interaction
with the operator are both essential,
particularly in dynamic and {partially unknown} environments.
The operator might need to add new tasks,
cancel some tasks, change priorities and modify planning results.
How to design the procedure for these
{interactions and efficient algorithms to fulfill these needs have}
been mostly neglected in the related literature.
Thus, this work proposes a human-centric coordination and supervision scheme (HECTOR) for
large-scale robotic {fleets under continual and uncertain temporal tasks}.
It consists of three hierarchical layers:
(I) the bidirectional and {multimodal} protocol of online human-fleet interaction,
{where the operator interacts with and supervises} the whole fleet;
(II) the rolling assignment of currently-known tasks to teams within a certain horizon,
and (III) the dynamic coordination within a team given the detected subtasks
during online execution.
The overall mission can be as general as temporal logic formulas over collaborative actions.
Such hierarchical structure allows human interaction and supervision
at different granularities and triggering conditions,
to both improve computational efficiency and reduce human effort.
Extensive {human-in-the-loop simulations} are performed over heterogeneous fleets
under {various temporal tasks and environmental uncertainties}.
\end{abstract}

\def\abstractname{Note to Practitioners}
\begin{abstract}
  This work is motivated by {the practical need for an operator} to coordinate
  a large number of heterogeneous and autonomous robots
  such as unmanned aerial and ground vehicles, to accomplish complex collaborative
  tasks such as patrol, delivery, search and rescue.
  Existing approaches often assume that the fleet size is relatively small,
  and the tasks are known beforehand, while neglecting the human interaction
  during online execution.
  This paper proposes a human-centric and hierarchical framework for online coordination
  and supervision of potentially dozens of robots.
  It allows the operator to interact online with the fleet via lean communication,
  e.g., to specify complex temporal missions,
  approve the formation of teams with assigned tasks,
  monitor the progress of task execution,
  potentially add new tasks and remove old tasks, {and adjust the priority of some tasks}.
  The hierarchical planning scheme decouples the formation of teams, the task assignment among teams
  and the task coordination within each team.
  {Human-in-the-loop simulations} suggest that the framework can be applied to practical relief missions
  after disasters, where the operator can coordinate and supervise
  large-scale fleets efficiently.
\end{abstract}


\vspace{-10pt}

\section{Introduction}\label{sec:intro}

\begin{figure}[t!]
  \centering
  \includegraphics[width=0.96\hsize]{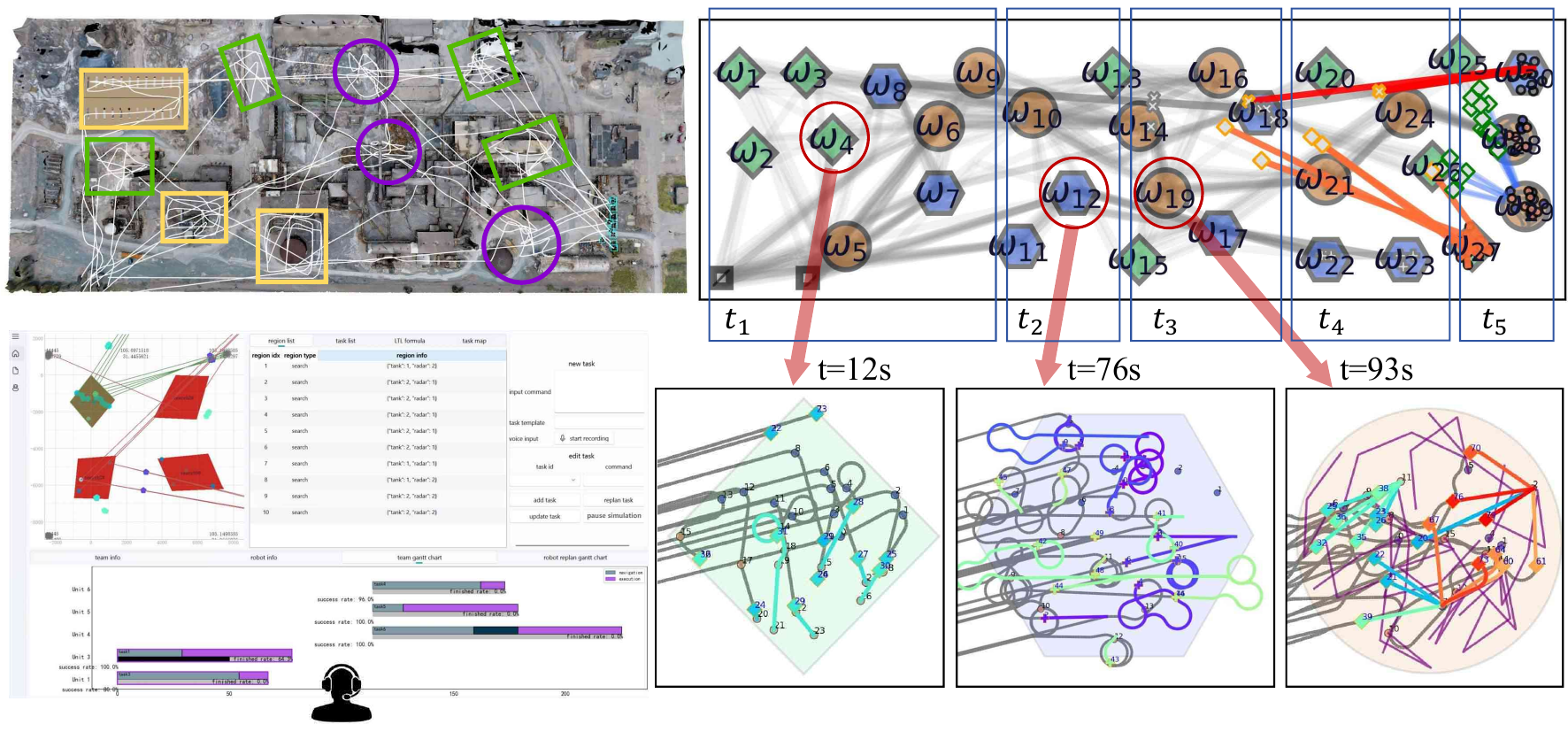}
  \vspace{-0.1in}
  \caption{The considered scenario where a human operator
    coordinates, interacts, and supervises a large robotic
    fleet via the proposed framework.
    {Online requests are specified via the designed interface and protocol
      (\textbf{bottom-left}),
    such as adding or cancelling new temporal tasks,
    different priorities and deadlines, or direct assignment of certain robots.
    The fleet reacts autonomously to fulfill different tasks at different regions,
    both at the global fleet-level (\textbf{top})
    and the local team-level (\textbf{bottom-right}).}
    In the simulated case,~$30$ tasks and~$450$ subtasks
    are accomplished by~$80$ robots. }
  \label{fig:show}
  \vspace{-0.1in}
\end{figure}

\IEEEPARstart{H}{eterogeneous} robotic fleets, combining ground and aerial vehicles,
are increasingly adopted for missions that exceed the capabilities
of individual robots~\cite{ferreira2021survey, krizmancic2022cooperative_ag}.
Concurrent operation improves efficiency, while collaboration enables
functions such as formation, cooperative transport, and
coverage~\cite{varava2017herding, tuci2018cooperative_transport}.
Despite these advantages, coordination of large fleets remains a
central challenge, especially for missions with distributed subtasks
requiring temporal and spatial constraints~\cite{kantaros2020stylus}.
The task allocation problem grows combinatorially with fleet size and
mission length~\cite{choudhury2017dynamics}.
Many existing approaches compute assignments offline for static task
sets~\cite{smith2009dynamic}, but such assumptions are impractical
for deployments where objectives evolve or are canceled in real time.
Dynamic environments demand continual replanning, yet conventional
methods rapidly become intractable and unstable, often leading to
oscillatory allocations and degraded team performance~\cite{hoque2022fleetdagger}.

\begin{figure*}[t!]
  \centering
  \includegraphics[width=0.9\hsize]{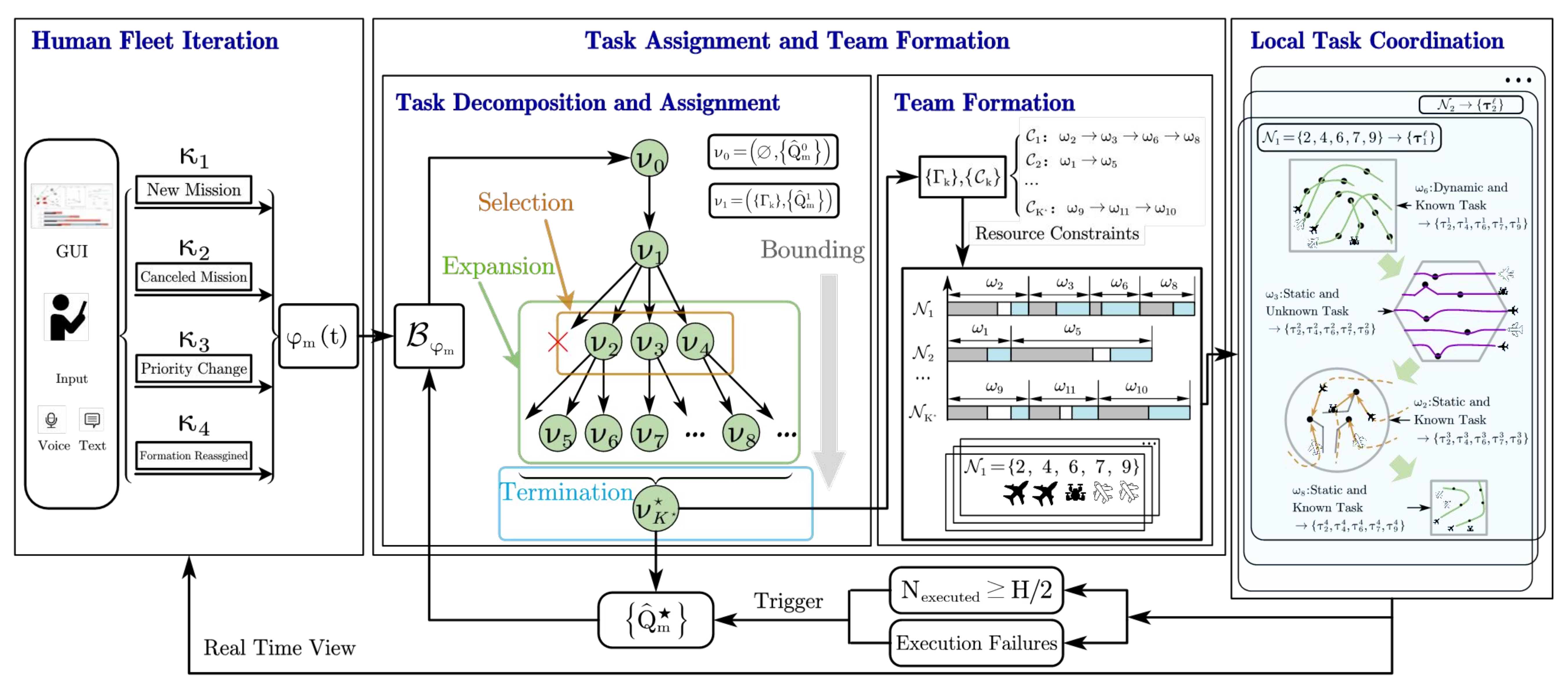}
  \vspace{-0.1in}
  \caption{    {The proposed human-centric framework for
    hierarchical coordination and supervision of robotic fleets,
    consisting of \textbf{three} main components:
    (I) the interaction protocol and interface for four types
    of online requests (\textbf{left});
    (II) the automata-guided task assignment and {team formation}
    (\textbf{middle});
    (III) three types of local coordination strategy
    for different tasks (\textbf{right}).
    Human requests, environment state and execution status
    are all updated online, for which adaptations
    for tasks and motions are
    triggered at different granularities and conditions.}}
  \label{fig:overall}
  \vspace{-0.1in}
\end{figure*}

Moreover, in practice, robotic fleets rarely operate in full autonomy.
As shown in Fig.~\ref{fig:show},
human operators frequently supervise high-level decision-making,
including mission specification, cancellation of outdated objectives,
and adjustment of task priorities~\cite{ferreira2021survey,dahiya2023survey}.
During execution, the operator may need to monitor mission progress,
inspect robot status, reassign subtasks, or intervene
under uncertainty~\cite{swamy2020scaledautonomy}.
Efficient procedures for human–fleet interaction are therefore crucial
to ensure responsiveness and safety during operation.
However, most existing research either emphasizes fully autonomous
coordination~\cite{chen2021decentralized} or restricts human–robot
interaction to very small teams~\cite{dahiya2023survey}.
Recent efforts highlight scalable approaches to interactive fleet
learning~\cite{hoque2022fleetdagger}, but protocols and interfaces for
large heterogeneous fleets under dynamic and partially unknown
conditions remain largely unexplored.

To address these challenges, this work introduces HECTOR,
a unified framework for human-centric supervision and coordination
of large-scale fleets in partially unknown and dynamic environments.
Team-wise collaborative missions are specified and released online
in response to observations during execution,
so the distribution and requirements of tasks are revealed only at runtime.
{As illustrated in Fig.~\ref{fig:overall}, the framework employs a hierarchical
hybrid architecture with three layers:
(I) a human–fleet interaction layer that converts online observations
into team-wise missions;
(II) a task-assignment layer that decomposes missions into collaborative
tasks and allocates them to composed teams using a receding-horizon strategy; and
(III) a local coordination layer that manages intra-team execution through
appropriate distributed strategies.}
The interaction protocol supports bidirectional
and {multimodal requests
between the operator and the fleet}, including real-time visualization
and supervision.
Collaborative missions are expressed as linear temporal logic formulas
over parameterized utility functions, providing both generality
and formality in specification.
The framework is fully online and adapts continuously to operator inputs,
external events, and runtime observations.
Its efficiency and robustness are validated through
practical {human-in-the-loop
simulations} with large heterogeneous systems.

The main contributions of this work are twofold:
(I) a flexible and scalable framework enabling operators to interact with
and coordinate large-scale robotic fleets in uncertain and dynamic settings;
and (II) explicit treatment of practical aspects such as heterogeneous
capabilities, parameterized collaboration utilities, temporal task
constraints, and variability in the number and distribution of subtasks.

\section{Related Work}\label{subsec:intro-related}

\subsection{Task Planning for Robotic Fleets}
Task planning in robotic fleets addresses the decomposition of missions into
subtasks and their allocation across robots or teams. Surveys on multi-robot
task allocation (MRTA)~\cite{gini2017multi, torreno2017cooperative}
have established taxonomies of problem structures, including vehicle routing,
scheduling, and coalition formation~\cite{dai2024dynamic}.
Centralized optimization methods such as MILP and heuristic
search~\cite{ji2021multi} generate globally near-optimal
solutions and allow precise encoding of resource and time constraints, but the
combinatorial complexity makes them unsuitable for large fleets or missions with
dynamic updates. Decentralized approaches improve scalability and resilience by
distributing decision-making, such as market-based auctions
\cite{zhao2021market} that trade computational effort for speed, and distributed
optimization frameworks~\cite{fioretto2018distributed} that enable concurrent
negotiation among agents. Approximate methods, including evolutionary
algorithms~\cite{liu2019multi} and genetic search~\cite{martin2021mrta}, reduce
computational overhead and can adapt to heterogeneous capabilities, but they
often sacrifice guarantees of feasibility or temporal optimality.
{Recent efforts emphasize hierarchical decomposition and reactive
coordination to improve scalability under dynamic task streams,
including decomposition-based hierarchical planning~\cite{luo2024decomposition,leahy2021scalable}
and real-time reactive allocation under temporal logic~\cite{chen2025real}.}
Despite these advances, most approaches assume static and fully known tasks at the planning
stage. Efficient solutions that adapt to the continual online task generation,
uncertainties in subtask number and distribution,
and online operator interventions remain open.

\subsection{Complex Tasks as Temporal Logic Formulas}
Temporal logic has emerged as a principled way to specify robotic missions that
combine safety, temporal order, and collaboration requirements. Probabilistic
temporal logics (PCTL)~\cite{lahijanian2015temporal} enable reasoning about
uncertainty in outcomes, while LTL provides a rich language for describing
sequential objectives such as repeated patrolling, collaborative capture, or
persistent surveillance~\cite{kantaros2020stylus,  leahy2021scalable, schillinger2018simultaneous,
luo2021temporal}. Counting LTL (cLTL)~\cite{sahin2019multirobot}
extends this by capturing quantitative requirements, for instance requiring a
minimum number of robots at specific subtasks.
Capability LTL~\cite{li2025task,cardona2024planning} address the time-dependent capabilities of robots
during execution.
Hierarchical LTL is proposed in~\cite{luo2024decomposition} {to
further introduce structured analyses and synthesis to improve efficiency}.
Centralized formulations ensure completeness and optimality through sampling-based search
\cite{kantaros2020stylus}, automaton–system synchronization
\cite{schillinger2018simultaneous}, and MILP encoding of task constraints
\cite{leahy2021scalable, luo2021temporal, sahin2019multirobot, kurtz2021more}, but they face
double-exponential growth with task and fleet size. Decentralized methods
address scalability through local coordination~\cite{lindemann2019coupled},
decision trees~\cite{chen2025real}, partial-order analyses~\cite{liu2024time},
reactive synthesis~\cite{kantaros2022perception}, and hierarchical decoupling~\cite{luo2025simultaneous}.
More recent works emphasize dynamic reactivity by combining logic
specifications with scalable heuristics~\cite{leahy2021scalable,chen2025real, cardona2024planning}.
{As summarized in Table~\ref{table:rw_comparison},
  scalability in complex missions has been addressed through mechanisms
  such as sampling~\cite{kantaros2020stylus}, capacity-based optimization~\cite{leahy2021scalable},
  counting constraints~\cite{sahin2019multirobot}, decision tree~\cite{chen2025real},
  and flow-based construction~\cite{gosrich2025online}.
Nevertheless, most aforementioned work overlook the online interaction with human operators,
and different types of local tasks contained within the temporal mission,
which in itself requires coordination and adaptation.}
Such nested coordination raises key challenges in retaining tractable
computational complexity while enabling coherent adaptation across multiple
spatial, temporal and organizational scales.

{As the most relevant to this work, HULK~\cite{luo2025hulk} facilitates continual task
planning by modeling missions constraints as posets. However, this approach necessitates
reconstruction whenever missions are modified online, limiting its online responsiveness.
Furthermore, it fails to account for the interplay between
high-level coordination and low-level motion feasibility, as well as the
integration of human supervision during execution. The proposed method addresses these
shortcomings by incorporating a human interaction protocol and an online adaptation
algorithm, further integrating dynamic constraints into motion feasibility. Additionally,
it enhances computational efficiency by eliminating the need for poset reconstruction
and employing an automata-guided search instead.}

\begin{table}[t!]
  \centering
  {
\caption{\textbf{Comparison with Related Work}}
\label{table:rw_comparison}
\vspace{-0.05in}
\setlength{\tabcolsep}{3pt}
\renewcommand{\arraystretch}{1.5}
\resizebox{\columnwidth}{!}{
\begin{tabular}{l|ccccccc}
\toprule
\textbf{Method} &
\makecell[c]{\scriptsize Task \\ \scriptsize Specification} &
\makecell[c]{\scriptsize Online \\ \scriptsize Task} &
\makecell[c]{\scriptsize Uncertain/ \\ \scriptsize Unknown \\ \scriptsize Subtasks} &
\makecell[c]{\scriptsize Human \\ \scriptsize Interaction} &
\makecell[c]{\scriptsize Dynamic \\ \scriptsize Constraints} &
\makecell[c]{\scriptsize Local \\ \scriptsize Coordination} \\
\midrule
STyLuS*~\cite{kantaros2020stylus} &
\makecell[c]{\small LTL} &
\makecell[c]{\small $\times$} &
\makecell[c]{\small $\times$} &
\makecell[c]{\small $\times$} &
\makecell[c]{\small $\surd$} &
\makecell[c]{\small $\times$} \\
ScRATCHeS~\cite{leahy2021scalable} &
\makecell[c]{\small CaTL} &
\makecell[c]{\small $\times$} &
\makecell[c]{\small $\times$} &
\makecell[c]{\small $\times$} &
\makecell[c]{\small $\surd$} &
\makecell[c]{\small $\times$} \\
DecTree~\cite{chen2025real} &
\makecell[c]{\small LTL} &
\makecell[c]{\small $\surd$} &
\makecell[c]{\small $\surd$} &
\makecell[c]{\small $\times$} &
\makecell[c]{\small $\times$} &
\makecell[c]{\small $\times$} \\
cLTL$+$~\cite{sahin2019multirobot} &
\makecell[c]{\small cLTL} &
\makecell[c]{\small $\times$} &
\makecell[c]{\small $\times$} &
\makecell[c]{\small $\times$} &
\makecell[c]{\small $\surd$} &
\makecell[c]{\small $\times$} \\
HieraTL~\cite{luo2025simultaneous} &
\makecell[c]{\small hLTL} &
\makecell[c]{\small $\surd$} &
\makecell[c]{\small $\surd$} &
\makecell[c]{\small $\times$} &
\makecell[c]{\small $\times$} &
\makecell[c]{\small $\times$} \\
Flow-based~\cite{gosrich2025online} &
\makecell[c]{\small Ordered} &
\makecell[c]{\small $\times$} &
\makecell[c]{\small $\surd$} &
\makecell[c]{\small $\times$} &
\makecell[c]{\small $\times$} &
\makecell[c]{\small $\times$} \\
HULK~\cite{luo2025hulk} &
\makecell[c]{\small LTL} &
\makecell[c]{\small $\surd$} &
\makecell[c]{\small $\surd$} &
\makecell[c]{\small $\times$} &
\makecell[c]{\small $\surd$} &
\makecell[c]{\small $\surd$} \\
\midrule
\textbf{HECTOR (Ours)} &
\makecell[c]{\small LTL} &
\makecell[c]{\small \textbf{$\surd$}} &
\makecell[c]{\small \textbf{$\surd$}} &
\makecell[c]{\small \textbf{$\surd$}} &
\makecell[c]{\small \textbf{$\surd$}} &
\makecell[c]{\small \textbf{$\surd$}} \\
\bottomrule
\end{tabular}
}}
\vspace{-3mm}
\end{table}

\subsection{Human-fleet Interaction}
{Beyond automated task planning, effective human-swarm collaboration under
uncertainty remains a central bottleneck, especially in deciding when and how
operators should validate and intervene online~\cite{ferreira2021survey,dahiya2023survey}.
In practice, no universal metric reliably predicts
whether an autonomous plan will remain operationally
sound once execution begins, so human expertise is essential for validating
reasoning quality and mission viability.
Human–swarm interaction emphasizes interface design, situational awareness, and
scalable mechanisms for directing collectives~\cite{kolling2016human}.
Mixed-initiative approaches~\cite{gombolay2017computational} dynamically shift
decision authority between humans and autonomy, supporting conflict resolution
and efficient replanning in uncertain missions.
Similar approaches are proposed in~\cite{tumova2014maximally,ulusoy2013optimality}
to compute least violating plans.
Scaled-autonomy frameworks~\cite{swamy2020scaledautonomy} extend this by allowing
robots to modulate the
level of assistance provided to the operator based on task load.
Interactive fleet learning~\cite{hoque2022fleetdagger} integrates demonstration,
online supervision, and data aggregation to refine fleet policies over time.
However, many existing paradigms provide limited support for online
validation and intervention, often restricting operators to static displays or
post-hoc plans. In dynamic settings, this can force a fallback to teleoperation, which scales
poorly with the fleet size due to increasing workload and latency. Consequently,
interaction protocols that enable sparse and high-level oversight while preserving
scalable autonomy remain insufficiently addressed.}

\section{Preliminary}\label{sec:preliminary}

\subsection{Linear Temporal Logic (LTL)}\label{subsec:ltl}
Linear Temporal Logic (LTL) formulas are composed of a set of atomic
propositions $AP$ together with Boolean and temporal operators. Atomic
propositions are Boolean variables representing elementary facts in the system,
which can be either true or false depending on the state. The syntax
\cite{baier2008principles} is given as follows:
\[
\varphi \triangleq \top \;|\; p \;|\; \varphi_1 \wedge \varphi_2 \;|\; \neg \varphi
\;|\; \bigcirc \varphi \;|\; \varphi_1 \,\textsf{U}\, \varphi_2,
\]
where $\top \triangleq \texttt{True}$, $p \in AP$, $\bigcirc \triangleq$
\emph{next}, $\textsf{U} \triangleq$ \emph{until}, and
$\bot \triangleq \neg \top$. Other operators such as
$\Box \triangleq \emph{always}$, $\Diamond \triangleq \emph{eventually}$,
and $\Rightarrow \triangleq \emph{implication}$ can be derived as
abbreviations. LTL is widely used in robotics for specifying high-level missions
such as ``eventually visit region $A$ and then always avoid region $B$.''

Formally, an infinite word $w$ over the alphabet $2^{AP}$ is defined as an
infinite sequence $W \triangleq \sigma_1\sigma_2\cdots$, with
$\sigma_i \in 2^{AP}$. The language of a formula $\varphi$ is defined as the
set of words that satisfy it, namely
$\mathcal{L} \triangleq Words(\varphi) = \{W \mid W \models \varphi\}$,
where $\models$ denotes the satisfaction relation. A particularly useful
subclass is the \emph{co-safe} formulas, which can be satisfied by a finite
sequence of words. They involve only the operators $\bigcirc$, $\textsf{U}$,
and $\Diamond$, and are expressed in positive normal form. This property is
advantageous in planning problems, since many practical missions can be
verified after a finite execution, avoiding the need for reasoning over
infinite traces.

\subsection{Nondeterministic B\"uchi Automaton}\label{subsec:nba}

Given an LTL formula $\varphi$, one can construct an equivalent automaton that
accepts exactly the set of words satisfying $\varphi$. A common representation
is the Nondeterministic B\"uchi Automaton (NBA), defined as follows.

\begin{definition}[NBA] \label{def:nba}
A NBA $\mathcal{B}$ is a 5-tuple
$\mathcal{B} \triangleq (Q,\,Q_0,\,\Sigma,\,\delta,\,Q_F)$,
where $Q$ is the set of states;
$Q_0 \subseteq Q$ is the set of initial states;
$\Sigma = AP$ is the set of alphabets;
$\delta: Q \times \Sigma \rightarrow 2^{Q}$ is the nondeterministic transition relation;
and $Q_F \subseteq Q$ is the set of \emph{accepting} states. \hfill $\blacksquare$
\end{definition}

Given an infinite word $w \triangleq \sigma_1\sigma_2\cdots$, the resulting
\emph{run} \cite{baier2008principles} within $\mathcal{B}$ is an infinite
sequence $\rho \triangleq q_0q_1q_2\cdots$ such that $q_0 \in Q_0$,
$q_i \in Q$, and $q_{i+1} \in \delta(q_i,\,\sigma_i)$ hold for all
$i \geq 0$. A run is \emph{accepting} if it visits accepting states infinitely
often, i.e., $\textsf{inf}(\rho)\cap Q_F \neq \emptyset$, where
$\textsf{inf}(\rho)$ is the set of states that appear infinitely often in
the sequence of~$\rho$. Such accepting runs are commonly represented in a prefix–suffix form,
where the prefix reaches an accepting state, and the suffix is a cycle that revisits
this state infinitely. This construction is fundamental in temporal-logic
planning, but the size of $\mathcal{B}$ can grow double exponentially with the
length of $\varphi$, which poses scalability challenges in complex missions.

\section{Problem Formulation}\label{sec:problem}

\subsection{Multi-robot Systems}\label{subsec:multi-agent}
Consider a team of~$N$ robots denoted by~$\mathcal{N}\triangleq\{1,\cdots,N\}$,
operating in a shared workspace~$\mathcal{W}\subset\mathbb{R}^3$. Each
robot~$i\in\mathcal{N}$ is characterized by its position~$x_i\in\mathcal{W}$,
a reference velocity~$v_i\in\mathbb{R}^3$, and a set of primitive
actions~$\mathcal{A}_i$. A robot can navigate freely in~$\mathcal{W}$
according to~$v_i$, and can execute one action from~$\mathcal{A}_i$ at a time.
The local plan of robot~$i$ is expressed as a sequence of timed actions:
\begin{equation}\label{eq:tau}
  \tau_i\triangleq (t^1_i,\,p^1_i,\,a^1_i)(t^2_i,\,p^2_i,\,a^2_i)\cdots,
\end{equation}
where~$t_i^\ell\geq 0$ denotes the time instant;
$p_i^\ell\in\mathcal{W}$ the goal position;
and $a_i^\ell\in\mathcal{A}_i$ the action to be executed. Thus,
robot~$i$ navigates to~$p_i^\ell$ with velocity~$v_i$ and initiates
action~$a_i^\ell$ from time~$t_i^\ell$, for all~$\ell\geq 1$. This representation
encodes both mobility and task execution, which allows the fleet to be described
as a set of inter-dependent and timed action sequences that must be coordinated at scale.

\subsection{Online Requests from Human Operator}\label{subsec:human}

During execution, the operator can adapt the behavior of the fleet through
online requests, which may introduce new missions or modify existing ones.
Unlike offline formulations that assume static task sets, here the requests are
treated as dynamic events that evolve over time and directly affect planning
decisions. Four different types of requests are defined with explicit
parameters as follows.

(I) A \emph{new mission} can be released at time~$t>0$ defined as
$\kappa_1(t)\triangleq(\varphi_t,\,t)$, where
\begin{equation}\label{eq:phi}
  \varphi_t\triangleq \text{sc-LTL}(\boldsymbol{\omega}_t),
\end{equation}
is a syntactically co-safe LTL specification over the set of collaborative
tasks $\boldsymbol{\omega}_t\triangleq\{\omega_1,\cdots,\omega_{M_t}\}$.
Each collaborative task~$\omega_m$ is defined as:
\begin{equation}\label{eq:task}
  \omega_m\triangleq \Big(S_m,\,\eta_m,\,\big\{(n_j,\,a_j,\,s_j),
  j=1,\cdots,J_m\big\}\Big),
\end{equation}
where $S_m\subset\mathcal{W}$ is the region of execution;
$(n_j,a_j,s_j)$ is a subtask requiring at least $n_j$ robots to perform
action $a_j$ at location $s_j$; and $J_m$ is the number of subtasks. The
function $\eta_m:\mathbb{N}\times\mathcal{A}\times 2^{\mathcal{N}}\rightarrow
\mathbb{R}^+$ defines the estimated duration, i.e.,
$\eta_m(n_j,a_j,\mathcal{N}_j)$ gives the completion time if subteam
$\mathcal{N}_j\subseteq \mathcal{N}$ executes action $a_j$ at $s_j$. The
logical composition of tasks follows the syntax in
Sec.~\ref{subsec:ltl}, and satisfaction is defined through the
relation~$\models$. Thus, the missions accumulate as
$\boldsymbol{\varphi}_t\triangleq\{{\varphi}_{t_\ell},\ \forall t_\ell\leq t\}$.
Note that uncertainty is inherent: both the number of subtasks~$J_m$
and their locations $\{s_j\}$ may
be unknown at release, requiring redundancy in team formation and adaptive
coordination during execution.

\begin{remark}\label{rm:cltl}
The task definition in~\eqref{eq:task} differs from cLTL-based formulations
in~\cite{sahin2019multirobot} in two aspects: (I) both task locations and the
number of subtasks may be uncertain, and (II) the subtask duration depends on the
assigned robots, rather than being instantaneous~\cite{liu2024time}. These
extensions better reflect real-world missions such as search and rescue, where
the environment reveals subtasks progressively. \hfill $\blacksquare$
\end{remark}

\begin{remark}\label{rm:constraints}
The duration function~$\eta_m(\cdot)$ typically saturates, where the marginal
benefit of adding robots decreases with team size. This property is widely
adopted in generic task models~\cite{ferreira2021survey, krizmancic2022cooperative_ag, apt2009generic},
yet it is often neglected in temporal-logic planning~\cite{kantaros2020stylus,
  luo2024decomposition, leahy2021scalable, chen2025real, luo2021temporal, cardona2024planning}.
Capturing this effect is
crucial to avoid over-allocation and to ensure efficient use of heterogeneous
resources. \hfill $\blacksquare$
\end{remark}

(II) Previous missions \emph{can be cancelled} by~$\kappa_2(t)\triangleq \{\varphi_{t_\ell}\}$,
where~$\varphi_{t_\ell}\in \boldsymbol{\varphi}_t$ is the mission
that has not yet been completed and should be cancelled. This captures the
practical need to revoke outdated or invalid goals without disrupting other
ongoing tasks and plans.

(III) \emph{Deadlines and priorities} of existing missions can be modified by
$\kappa_3(t) \triangleq \{(\varphi_{t_\ell},\,d_\ell^{\star},\,w_\ell^{\star})\}$,
where $\varphi_{t_\ell}\in \boldsymbol{\varphi}_t$ is the mission to be
modified, $d_\ell^{\star}>0$ is the updated deadline,
and~$w_\ell^{\star}>0$ is the revised priority. This reflects the
ability of the operator to shift focus between competing objectives depending
on urgency and available resources.

(IV) \emph{Robots may be reassigned} across missions manually by
$\kappa_4(t) \triangleq \{(\mathcal{N}_\ell,\, \varphi_{t_\ell})\}$,
where the subset of robots~$\mathcal{N}_\ell\subseteq \mathcal{N}$ should be
assigned to mission~$\varphi_{t_\ell}\in \boldsymbol{\varphi}_t$. This allows
explicit operator control over resource allocation in situations where
autonomous assignment may not match human intent.

Thus, the evolving set of operator requests up to time~$t\geq 0$ is denoted by
the set below:
\begin{equation}\label{eq:upsilon}
  \mathcal{K}(t)\triangleq \big\{\kappa_1(t),\kappa_2(t),\kappa_3(t),\kappa_4(t)\big\},
\end{equation}
where fulfilled requests are removed upon completion. This formulation enables
a closed-loop interaction in which operator input and autonomous planning
mutually adapt over time.

\subsection{Problem Statement}\label{subsec:prob}
Given the mission specifications in~\eqref{eq:phi}
and additional operator requests in~\eqref{eq:upsilon},
the overall objective is to synthesize collective plans in~\eqref{eq:tau} such that the
average mission response time is minimized, i.e.,
\begin{equation}\label{eq:objective}
  \boldsymbol{\min}_{\{\tau_i\}}\,
  \frac{\sum_{{\varphi}_{t_\ell}\in{\boldsymbol{\varphi}}_t}
  (t_\ell^{\texttt{f}}-t_\ell)}
  {|\boldsymbol{\varphi}_t|},
\end{equation}
where~$t_\ell$ and $t_\ell^{\texttt{f}}$ denote the release and completion
times of mission~${\varphi}_{t_\ell}$. This objective explicitly measures the
responsiveness of the fleet, which is critical in dynamic and safety-critical
environments such as disaster relief or surveillance. All operator requests in
$\mathcal{K}(t)$ must be satisfied during planning and execution as described
earlier, ensuring that human interventions are seamlessly integrated into the
autonomous coordination framework.

\begin{figure}[t!]
    \centering
    \includegraphics[width=0.9\hsize]{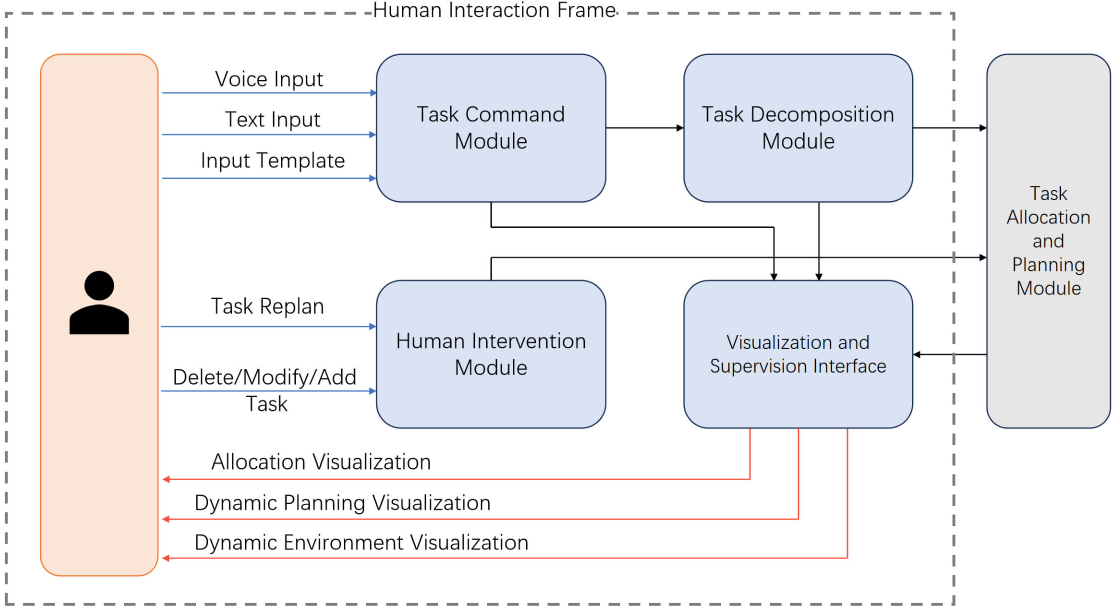}
    \vspace{-0.1in}
    \caption{      The proposed protocols for human-fleet interaction:
      the human command and intervention module (\textbf{left});
      the task decomposition and planning modules (\textbf{right});
      and the visualization module (\textbf{middle}).}\label{fig:interaction}
    \vspace{-0.1in}
\end{figure}

\begin{remark}\label{rm:gap}
The problem combines temporal missions with online operator requests,
via introducing cancellations, priority updates and uncertain subtasks.
Centralized MILP~\cite{torreno2017cooperative, ji2021multi} and
decentralized market-based methods~\cite{zhao2021market, fioretto2018distributed}
typically assume static tasks, while temporal-logic approaches
\cite{kantaros2020stylus, schillinger2018simultaneous, sahin2019multirobot}
focus on fixed specifications. Existing frameworks for human–fleet
interaction~\cite{ferreira2021survey, swamy2020scaledautonomy} also target
simpler settings, leaving the proposed problem insufficiently addressed. The
integration of temporal-logic missions, dynamic operator requests, and
large-scale multi-robot coordination therefore represents a significant gap
that this work aims to fill. \hfill $\blacksquare$
\end{remark}

\section{Proposed Solution}\label{sec:solution}
The proposed solution consists of three main components
as shown in Fig.~\ref{fig:overall}:
(I) the protocol and interface of human-fleet interaction,
along with the hierarchical coordination and communication
in Sec.~\ref{subsec:hfi};
(II) the receding-horizon task planning algorithm that assigns tasks
to subteams of robots given the global mission specification
and constraints on the resources in Sec.~\ref{subsec:task};
(III) the local coordination algorithm that assigns subtasks to robots
during online execution in Sec.~\ref{subsec:act}.
The synergy and adaptation of these components are triggered
by online observations, human requests and execution status.
{For clarity, a table of key variables in this work is provided in the Appendix.}

\subsection{Protocol and Interface of Human-fleet Interaction}\label{subsec:hfi}

This section introduces the communication and coordination framework that
enables online human–fleet interaction. It first presents the protocol that
maps operator requests to system modules, then describes the graphical
interface that supports multimodal inputs and visual supervision, and finally
outlines the hierarchical communication structure between the operator, team
leaders, and team members.

\subsubsection{Protocol for Online Requests}\label{subsubsec:human}

{
The communication protocol links the four operator request types in~\eqref{eq:upsilon}
to the interaction module. As shown in Fig.~\ref{fig:interaction}, when a new mission is issued
through request $\kappa_1$, the input is provided via voice, text, or templates, processed
by the “task command module” and forwarded to the “task decomposition module,” which generates
a sc-LTL specification with spatial-temporal constraints for the planning pipeline. Requests
for cancellation $\kappa_2$, task modification $\kappa_3$, and robot reassignment $\kappa_4$
are processed by the “human intervention module,” which updates active missions and resources,
passing changes to the “task allocation and planning module.” Ongoing executions that cannot
be interrupted are preserved. The protocol ensures closed-loop communication, where upward
flows transmit operator directives to the planning logic and downward flows provide feedback
on planning results, execution progress, and environmental changes. This feedback is visualized
via the “visualization and supervision interface,” displaying allocation maps, task graphs,
and execution timelines. This setup ensures that operator interventions are executed and
monitored in real-time throughout the mission lifecycle.
}

\begin{figure}[t!]
    \centering
    \includegraphics[width=0.9\hsize]{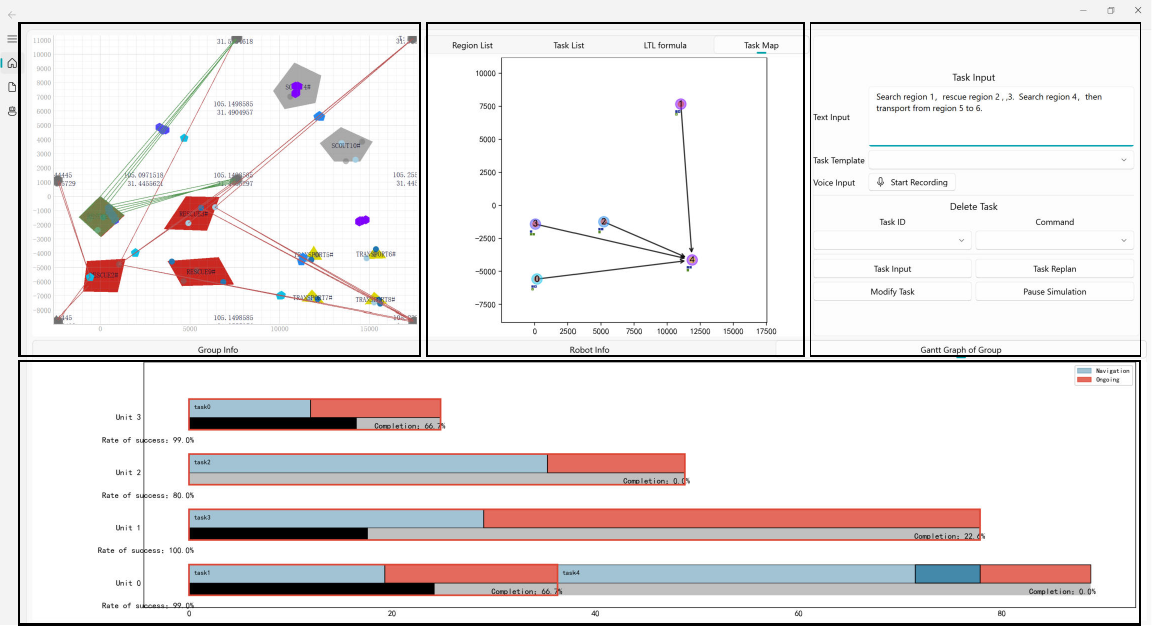}
    \vspace{-0.1in}
    \caption{        Integrated interface for human-fleet interaction,
        which includes scenario visualization (\textbf{top-left}),
        temporal ordering of subtasks (\textbf{top-middle}),
        panel to specify operator requests (\textbf{top-right}),
        local plans and execution progress for teams and robots
        are visualized online (\textbf{bottom}).}\label{fig:gui}
    \vspace{-0.1in}
  \end{figure}

\subsubsection{{Design of Graphical Interface for Online Interactions}}\label{subsubsec:interface}

The graphical interface in Fig.~\ref{fig:gui} integrates the online interaction protocol
with real-time visualization, linking operator inputs to the corresponding modules. {On
the right panel, the operator can release new missions through text, templates or
voice, which are processed by the task command module to fulfill
request~$\kappa_1$. The same panel allows cancellation and priority updates, routed
to the human intervention module for requests~$\kappa_2$ and~$\kappa_3$. Robot reassignment
for request~$\kappa_4$ is also supported. Responses are reflected in
multiple interactive panels}:
the scene map displays task distribution, robot positions and trajectories; the mission
panel shows the decomposed task graph and LTL formulas; and Gantt charts summarize task
allocations and execution progress both for teams and robots.
{This interface closes the loop between human input, intermediate results and fleet execution,
enabling transparent supervision and efficient adaptation.}

\subsubsection{Hierarchical Coordination and Communication}\label{subsubsec:hierarchical}

The communication architecture follows a hierarchical structure that mirrors
the organization of the fleet. At the top level, the operator interacts
with designated team leaders through the task command and intervention modules,
issuing high-level missions, modifying specifications, and
adjusting deadlines or resource allocations~\cite{you2021human}. Each team
leader translates these directives into concrete plans for
its team, forming subgroups when necessary and allocating robots to subtasks
according to capabilities and availability~\cite{kolling2016human}. At the
lower level, team members communicate directly with their leader to receive
local assignments and report execution status, enabling rapid coordination
without overwhelming the operator with low-level details. This
structure improves scalability by limiting communication overhead and enhances
robustness by isolating local replanning within teams, while maintaining
consistency since operator requests are propagated downward and execution
feedback is aggregated upward in a structured manner.

\subsection{Task Assignment and Team Formation}\label{subsec:task}

\subsubsection{Simultaneous Task Decomposition and Team Assignment}

Existing work on task coordination for multi-robot systems under temporal
logic specifications often relies on the synchronized product between the task
automaton~$\mathcal{B}_{\varphi_t}$ and the global system model, which is
constructed as the product of all local robot models~\cite{schillinger2018simultaneous,
  luo2021temporal}. This approach guarantees correctness but suffers from
double-exponential growth in computational complexity.
{To mitigate this
challenge, several methods have been proposed, including
local coordination~\cite{lindemann2019coupled},
decision trees~\cite{chen2025real}, distributed sampling~\cite{kantaros2020stylus},
hierarchical decoupling~\cite{luo2025simultaneous},
and partial-order analyses~\cite{liu2024time}.} Although these
methods improve scalability in static environments, they are not suitable for
open-world scenarios where new tasks may be triggered online. In such settings,
re-computation of the entire system model would be required after each update,
which invalidates previously computed results and leads to inefficiency.

Consider the B\"uchi automaton
$\mathcal{B}_{\varphi_m}\triangleq (Q^m,\,Q^m_0,\,\Sigma^m,\,\delta^m,\,Q^m_F)$
associated with the mission specification~$\varphi_m\in \Phi_t$,
where~$\Phi_t\triangleq \{\varphi_1, \cdots, \varphi_M\}$ is the set of
missions known at time~$t>0$ and~$\mathcal{M}\triangleq \{1,\cdots,M\}$.
Moreover, the fleet is divided into $K$ teams denoted by
$\mathcal{C}_k\subset \mathcal{N}$ and
$\mathcal{K}\triangleq \{1,\cdots,K\}$,
of which the exact value of~$K$ is to be decided.
Each team~$\mathcal{C}_k$ is defined
not by a fixed set of robots but by its composition, i.e., the available
number of each robot type with associated capabilities. Denote by
$\{\Gamma_k,\, k\in \mathcal{K}\}$ the local plans of teams
$\{\mathcal{C}_k\}$ as the sequences of tasks to be accomplished.

\begin{problem} \label{qs:assign_known_K}
  Given the mission specifications $\varphi_m\in \Phi_t$, determine the
  optimal number of teams~$K$, the composition of each team
  $\{\mathcal{C}_k\}$, and the local plans $\{\Gamma_k\}$ such that:
  (I) the tasks can be executed by the assigned team under the temporal-logic
  constraints and the robot capacity constraints; and (II) the execution time
  for missions is minimized as in~\eqref{eq:objective}.
  \hfill $\blacksquare$
\end{problem}

The search structure is organized as a tree
$\mathfrak{T} \triangleq ( \mathcal{V}, \rightarrow)$, where
$\mathcal{V} \triangleq \{\nu\}$ is the set of nodes, and
$\rightarrow \subset \mathcal{V} \times  \mathcal{V}$ defines the edges. Each
node~$\nu \triangleq (\{\Gamma_k,\, k\in \mathcal{K}\}, \{\widehat{Q}_m,\,
m\in \mathcal{M}\})$ contains two components: the partial local plans of all
teams~$\{\mathcal{C}_k\}$ and the sets of current reachable states in
$\mathcal{B}_{\varphi_m}$ for all missions $\varphi_m\in \Phi_t$. The root
node is $\nu_0\triangleq (\emptyset, \{\widehat{Q}^0_m,\, m\in \mathcal{M}\})$.
As summarized in Algorithm~\ref{alg:ss-TaskAssign}
and illustrated in Fig.~\ref{fig:Tree},
the proposed procedure consists of the following four stages.

\begin{figure}[t!]
  \centering
  \includegraphics[width=0.9\hsize]{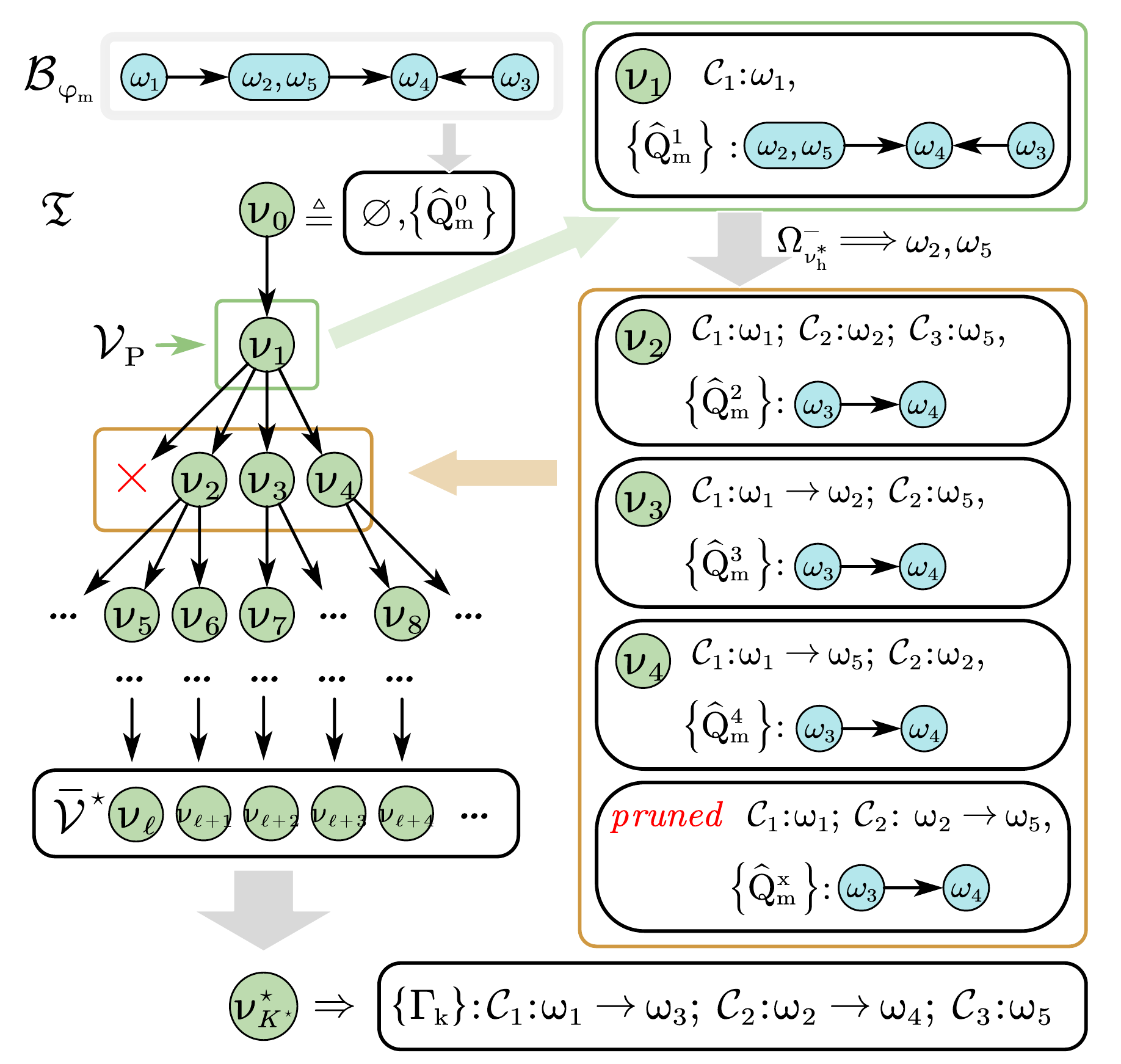}
  \vspace{-0.1in}
  \caption{{Illustration of the automaton-guided search tree,
    where the node expands over not only the
    local plans of each team~$\{\Gamma_k,\, k\in \mathcal{K}\}$,
    but also the progress of each mission~$\{\widehat{Q}_m,\,
    m\in \mathcal{M}\}$.
    The leaf nodes of complete assignments may correspond
    to different number of teams.}}
  \label{fig:Tree}
  \vspace{-0.1in}
\end{figure}

\emph{(I) Selection.}
A set of $P>0$ candidate nodes is selected for parallel expansion,
i.e., $\mathcal{V}_P \triangleq \{\nu_1^\star, \cdots, \nu_P^\star\}$, where
$\nu_p^\star \triangleq
{\textbf{argmin-p}}_{\nu \in  \mathcal{V}} \left\{\chi(\nu)\right\}$ for
$p=1,\cdots,P$. {The value function
$\chi: \mathcal{V} \rightarrow \mathbb{R}^+$ is defined as:
\begin{equation}\label{eq:value}
  \begin{split}
  \chi(\nu) \triangleq &\underset{k\in\mathcal{K}}{\textbf{max}} \{T_k\}
  + \eta_1 \sum_{k\in\mathcal{K}} C_k\\
  &+ \eta_2 \sum_{m\in \mathcal{M}} \mathbf{w}_m \underset{q_m\in \widehat{Q}_m} {\textbf{min}} \psi(q_m,\, Q^m_F),
\end{split}
\end{equation}
where~$\eta_1,\eta_2>0$ are weighting parameters;
$\mathbf{w}_m$ denotes the preference weight associated with mission $\varphi_m$,
reflecting human-specified priorities; $T_k>0$ is the ending time
of the local plan~$\Gamma_k$; $C_k>0$ is the estimated cost of~$\Gamma_k$; and
$\psi(q_m,\,Q^m_F)>0$ returns the length of the shortest path from state~$q_m$
to any final state in~$Q^m_F$.} The first term measures the makespan of the
current assignments, the second term estimates the overall cost, and the third
term accounts for the overall progress of missions.

{

\begin{remark}\label{rm:weights}
Parameters $\eta_1, \eta_2 > 0$ above regularize the search process, i.e.,
$\eta_1$ penalizes nodes with equal makespan based on cost, while $\eta_2$ prioritizes nodes with greater
mission progress via $\psi$. These weights ensure makespan minimization while improving
the search efficiency, with their impact primarily on convergence rate rather than optimality.
\hfill $\blacksquare$
\end{remark}

}

Moreover, given a node~$\nu\in \mathcal{V}_P$ and the associated local plans
$\{\Gamma_k\}$, the capacity constraints of each team~$k\in \mathcal{K}$ under
the updated assignment are given by:
\begin{equation} \label{eq:capacity-team}
  \mathcal{C}_k \triangleq \Big(
  \big\{(\beta^j_k,\, a^j),\, a^j \in \omega^j_k\big\},\,
  \forall \omega^j_k \in \Gamma_k\Big),
\end{equation}
which specifies the minimum number of robots~$\beta^j_k>0$ to perform action
$a^j\in \mathcal{A}$ for each assigned task~$\omega^j_k\in \Gamma_k$. This
requirement provides actions and counts, without binding specific robots, and
thereby decouples assignment from scheduling.
Moreover, if the capacity constraints above exceed the overall fleet capacity,
i.e., the following condition:
\begin{equation}\label{eq:capacity-bound}
  \sum_{k\in \mathcal{K}} \beta^j_k \leq
  \sum_{i\in \mathcal{N}} \mathds{1}(a^j \in \mathcal{A}_i),\; \forall a^j \in \mathcal{A};
\end{equation}
is violated, where the left-hand side is the required capacity and the
right-hand side is the available fleet capacity.
Then this node is marked infeasible and excluded from the set~$\mathcal{V}_P$.

\emph{(II) Expansion.}
Each selected node~$\nu_h^\star \in \mathcal{V}_P$ is expanded by assigning an
additional task to one team. The set of candidate tasks is defined as below:
\begin{equation}\label{eq:task-candidate}
  \Omega^-_{\nu_h^\star} \triangleq
  \bigcup_{m\in \mathcal{M}}
  \Big\{\omega \in \Sigma^m \,\big|\, \exists q_m \in \widehat{Q}_m :
  \delta^m(q_m, \omega) \in Q^m \Big\},
\end{equation}
where $\omega$ is an atomic task symbol from the alphabet~$\Sigma^m$ and
$\delta^m: Q^m \times \Sigma^m \rightarrow 2^{Q^m}$ is the transition function
of the B\"uchi automaton~$\mathcal{B}_{\varphi_m}$. Thus, a candidate task must
enable a valid state transition for at least one automaton
$\mathcal{B}_{\varphi_m}$. For each $\omega \in \Omega^-_{\nu_h^\star}$, a
child node is created by augmenting the local plan~$\Gamma_k$ of team
$\mathcal{C}_k$ with~$\omega$, namely:
\begin{equation}\label{eq:expand}
  \nu^+ \triangleq
  \Big(\{\Gamma_1,\cdots,\Gamma^+_k,\cdots,\Gamma_K\},\;
  \{\widehat{Q}_m^+\}_{m\in\mathcal{M}}\Big),
\end{equation}
where~$\Gamma^+_k$ appends~$\omega$ to~$\Gamma_k$; the task set~$\Omega_k$
within~$\Gamma_k$ is updated by adding $\omega$; and
$\widehat{Q}_m^+ \triangleq \{\delta^m(q_m,\omega)\,|\, q_m \in \widehat{Q}_m\}$
is the updated set of reachable states for mission~$\varphi_m$. Consequently,
the edge~$(\nu_h^\star,\, \nu^+)$ is inserted into the search tree
$\mathfrak{T}$. If no valid and feasible transition exists, then
$\Omega^-_{\nu_h^\star}=\emptyset$ and the node cannot be expanded. In
addition, a new team~$(K+1)$ can be created first by specifying its
composition; a candidate task $\omega\in \Omega^-_{\nu_h^\star}$ is then
assigned to it, and its capacity constraints $\mathcal{C}_{K+1}$ are computed
by~\eqref{eq:capacity-team}. Hence, task assignment and team formation emerge
naturally during the search.

More importantly, after assigning task~$\omega$ to team~$k \in \mathcal{K}$,
denoted by~$\omega^\ell_k$, the associated capacity constraint~$\mathcal{C}_k$
for this team is updated as follows:
\begin{align}\label{eq:update-constraint}
  \beta^j_k \triangleq \underset{a^j \in \omega^\ell_k \in \Gamma_k}{\textbf{max}}
  \{n_j\},\;  \forall a^j \in \mathcal{A};
\end{align}
i.e., the maximum number of robots~$\beta^j_k$ required for each action
$a^j$ across the sequence of tasks~$\omega^{L_k}_k$. The related ending time
is updated by:
\begin{align}\label{eq:estimated_time}
  \begin{split}
  t_{\texttt{e}}(\omega_k^\ell)  \triangleq
  &\underset{\omega_j \in
  \texttt{Pre}({\omega_k^\ell})}{\textbf{max}}
  \big\{t_{\texttt{e}}(\omega_j)\big\}\\
  &+ T_{\texttt{nav}}(S_k^{\ell-1}, S_k^{\ell})
  + T_{\texttt{exec}}(\omega_k^\ell),\\
\end{split}
\end{align}
where $t_{\texttt{e}}(\omega^\ell_k)$ is the estimated ending time of
$\omega^\ell_k \in \Gamma_k$; $\texttt{Pre}({\omega_k^\ell})$ is the set of
assigned tasks;
and $T_{\texttt{nav}}(S_k^{\ell-1}, S_k^{\ell})$ is the estimated navigation
duration between regions.

\begin{algorithm}[t!]
  \caption{Simultaneous Task Decomposition and Team Assignment}
  \label{alg:ss-TaskAssign}
  \SetAlgoLined
  \KwIn{$\Phi_t=\{\varphi_1,\cdots,\varphi_M\}$, automata $\{\mathcal{B}_{\varphi_m}\}$,
        fleet $\mathcal{N}$.}
  \KwOut{Teams $K$, capacities $\{\mathcal{C}_k\}$, plans $\{\Gamma_k\}$.}

  Initialize $\mathfrak{T}=(\mathcal{V},\rightarrow)$, $\mathcal{V}=\{\nu_0\}$\;

  \While{not terminated}{
    \tcc{\textbf{Selection}}
    $\mathcal{V}_P \gets \{\nu_1^\star,\cdots,\nu_P^\star\}$ by $\chi(\nu)$ in~\eqref{eq:value}\;
    Initialize $\mathcal{C}_k$ for all $k\in\mathcal{K}$ via~\eqref{eq:capacity-team}\;

    \tcc{\textbf{Expansion}}
    \ForEach{$\nu_h^\star \in \mathcal{V}_P$}{
      Obtain $\Omega^-_{\nu_h^\star}$ by~\eqref{eq:task-candidate}\;
      \ForEach{$\omega \in \Omega^-_{\nu_h^\star}$}{
        \If{add new team}{
          Create $\mathcal{C}_{K+1}$, set $\Gamma_{K+1}=\emptyset$\;
          Update~$\mathcal{K}$\;}
        \ForEach{team $k \in \mathcal{K}$}{
          Generate $\nu^+$ by adding $\omega$ to $\Gamma_k$ as~\eqref{eq:expand}\;
          Update $\widehat{Q}_m^+$ for all $m$\;
          Add $(\nu_h^\star,\nu^+)$ to $\mathfrak{T}$\;
          Update $\mathcal{C}_k$ via~\eqref{eq:update-constraint}\;
          Update $t_{\texttt{e}}(\omega_k^\ell)$ via~\eqref{eq:estimated_time}\;
          \lIf{\eqref{eq:capacity-bound} violated}{mark $\nu^+$ infeasible}
        }
      }
    }

    \tcc{\textbf{Bounding}}
    \ForEach{$\nu \in \mathcal{V}$}{
      Compute $\zeta(\nu)$ by~\eqref{eq:profile}\;
      Prune dominated nodes by~\eqref{eq:dominance}\;
    }
    \BlankLine
    Update~$\overline{\mathcal{V}}$ by~\eqref{eq:frontier}\;
  }

  \tcc{\textbf{Termination}}
  Compute~$\overline{\mathcal{V}}^\star$ by~\eqref{eq:complete-nodes},
  select~$\nu_{K^\star}^\star$ by~\eqref{eq:termination}\;
  \Return Optimal $K$, $\{\mathcal{C}_k\}$, $\{\Gamma_k\}$\;
\end{algorithm}

\emph{(III) Bounding.}
To reduce unnecessary exploration, each node is evaluated through a
\emph{performance profile} that retains detailed information about the plans of
all teams and the progress of all missions, i.e.,
\begin{align}\label{eq:profile}
  \begin{split}
  \zeta(\nu) \triangleq
  \Big[\,
  &\{T_k\},\;
  \{\mathcal{C}_k\},\;\\
  &\Big\{\underset{q_m\in \widehat{Q}_m}{\textbf{min}}
  \psi_m(q_m,\, Q^m_F), m\in\mathcal{M}\Big\}
  \,\Big],
\end{split}
\end{align}
where~$T_k$ is the ending time of $\Gamma_k$;
$C_k$ is the estimated cost of $\Omega_k$;
and the last term~$\psi_m(\cdot)$ measures the minimum distance from the current
automaton states~$\widehat{Q}_m$ to the accepting set~$Q^m_F$ for each mission
$\varphi_m$. The profile $\zeta(\nu)$ has dimension $(2K+M)$ and is
non-negative. Given two nodes $\nu_1$ and $\nu_2$, node $\nu_1$ is said to
\emph{dominate} node $\nu_2$ if the following holds:
\begin{equation}\label{eq:dominance}
  \zeta(\nu_1) \leq \zeta(\nu_2), \quad
  \text{and}\quad
  \zeta(\nu_1) \neq \zeta(\nu_2),
\end{equation}
where the inequality is understood element-wise across the vector
in~\eqref{eq:profile}. Thus, node $\nu_1$ has no larger makespan or cost for
any team, and no less mission progress for any specification, with strict
improvement in at least one entry. In this case, node~$\nu_2$ is marked as
dominated and excluded from expansion. The set of all non-dominated nodes is
the set of {frontiers}, defined as follows:
\begin{equation}\label{eq:frontier}
  \overline{\mathcal{V}} \triangleq
  \big\{\nu \in \mathcal{V}\,\big|\, \nexists\,\nu' \in \mathcal{V} :
  \nu' \text{ dominates } \nu\big\},
\end{equation}
which is updated whenever a new node is added. This bounding procedure
maintains only $\overline{\mathcal{V}}$ as candidates for expansion, ensuring
that strictly inferior nodes are pruned and that the search focuses on
promising branches of the tree~$\mathfrak{T}$.

\emph{(IV) Termination.}
The stages of selection, expansion, and bounding are repeated until the
computation budget is exhausted or no new non-dominated nodes emerge. The
current set of frontier is $\overline{\mathcal{V}}$ from~\eqref{eq:frontier},
and the subset of complete assignments is given by:
\begin{equation}\label{eq:complete-nodes}
  \overline{\mathcal{V}}^\star \triangleq
  \big\{\nu \in \overline{\mathcal{V}} \,\big|\,
  (\widehat{Q}_m \cap Q^m_F) \neq \emptyset,\;
  \forall m\in\mathcal{M}\big\},
\end{equation}
where the accepting set is reached for all missions.
Thus, the optimal node is selected among these assignments, i.e.,
\begin{equation}\label{eq:termination}
  \nu_{K^\star}^\star \triangleq
  {\textbf{argmin}}_{\nu \in \overline{\mathcal{V}}^\star}\,
  \{\chi(\nu)\},
\end{equation}
with $\chi(\nu)$ from~\eqref{eq:value}. The associated capacity constraints
$\{\mathcal{C}_k\}$ and the global assignment are specified by the resulting
plans~$\{\Gamma_k\}$ together with execution times.
{

\begin{example}
  \label{example:task-decomp}
  As illustrated in Fig.~\ref{fig:Tree},
  given~$\omega_1,\omega_2,\omega_3,\omega_4,\omega_5$ in the
  mission automata~$\{\mathcal{B}_{\varphi_m}\}$,
  the node~$\nu^\star_{K^\star}$ assigns~$5$ tasks to~$3$ subteams and the optimal $K^\star=3$.
  Thus the local plans are given by executing~$\omega_3$ after $\omega_1$ for team one;
  $\omega_4$ after $\omega_2$ for team two; and task~$\omega_5$ for team three.
  \hfill $\blacksquare$
\end{example}

}

{Last but not least, due to the dynamic nature of the environment, a receding-horizon
strategy is adopted for the task assignment.
Thus, the search is paused when the number of assigned tasks for the
fleet reaches the horizon $H>0$ as a user-defined hyper-parameter,
and is resumed at the next planning cycle
as described in the sequel.
Note that the horizon~$H$ balances the batch-assignment efficiency and the online responsiveness.
A small $H$ results in fragmented task sequences and frequent replanning,
which under-utilizes its capability to optimize large task sets.
Conversely, an excessively large~$H$ increases computational redundancy
and may induce oscillation in dynamic environments,
as long-term commitments are often invalidated by online mission updates.
This trade-off is evaluated empirically in Sec.~\ref{sec:experiments}.}

\begin{remark}\label{remark:task-tree}
{The framework presented above offers several advantages over existing
approaches~\cite{kantaros2020stylus, schillinger2018simultaneous,
luo2021temporal, luo2025simultaneous, liu2024time, luo2025hulk}:}
(I) it circumvents the explicit construction of the synchronous product
between the B\"uchi automaton, the robot models as transition systems, and the
global product, which is prohibitively large in multi-robot settings;
(II) the search is both anytime and complete, in contrast to mixed-integer linear
programming methods that often require long solving times without intermediate
feasible outputs;
(III) it is well suited to partially known and dynamic
environments where missions are triggered online, since a new specification can
be incorporated by introducing its reachable state set $\widehat{Q}_k$ without
interfering with existing missions;
(IV) the stages of node selection and expansion can be carried out in parallel,
enabling efficient scaling to large teams and complex mission sets;
{(V) the node construction and expansion above differ from the poset-based methods
in~\cite{liu2024time, luo2025hulk}. Instead of using partial orders,
each node~$\nu$ above maintains the sets of reachable states~$\widehat{Q}_m$ for each mission automaton,
allowing task decomposition and team assignment simultaneously.
This approach eliminates the re-computation for global posets, enabling faster integration
of online missions and reducing overhead pre-processing.
More numerical comparisons are given in Sec.~\ref{sec:experiments}.}
\hfill $\blacksquare$
\end{remark}

{
Correctness and completeness of the simultaneous task decomposition and team assignment
algorithm above is provided below.
Particularly, it is shown that the completeness and optimality hold in the static
and known case with the full horizon. Moreover, in the online receding-horizon case,
the algorithm can still ensure correctness and feasibility,
while scalability and adaptability are evaluated empirically in Sec.~\ref{sec:experiments}.
Proofs are provided in the Appendix.}

{
\begin{theorem}\label{thm:correctness}
Consider an instance of Problem~\ref{qs:assign_known_K} with fleet $\mathcal{N}$
and mission set $\Phi_t$. Suppose all missions are known \emph{a priori} and the
planning horizon $H$ is larger than the total number of admissible tasks. Then,
the team plans $\{\Gamma_k\}$ generated by
Alg.~\ref{alg:ss-TaskAssign} satisfy:
(I) all temporal constraints encoded in the
mission automata $\{\mathcal{B}_{\varphi_m}\}$ hence all mission specifications;
and (II) all capacity constraints
in~\eqref{eq:capacity-team}--\eqref{eq:capacity-bound}.
Moreover, whenever feasible plans exist,
Alg.~\ref{alg:ss-TaskAssign} returns the feasible plan
that minimizes the makespan objective in~\eqref{eq:value}.
\end{theorem}
}

{
\begin{lemma}\label{lemma:online}
Under online mission release and the receding-horizon replanning with $H<|\Omega|$,
Alg.~\ref{alg:ss-TaskAssign} guarantees that the
computed partial plan~$\nu$ is consistent with the active missions.
Moreover, each time a new mission is added,
the updated partial plan after adaptation
can still preserve satisfaction under the event-triggered adaptation scheme.
\end{lemma}
}

\subsubsection{Capacity-based and Redundancy-aware Team Formation}\label{subsec:team}

Given the optimal team–task assignment $\nu_{K^\star}^\star$, the composition
of each team remains to be determined. In particular, the capacity constraints
derived earlier specify only the aggregate requirements for each team, i.e.,
the minimum numbers of robots with specific capabilities. The actual formation
of teams requires allocating individual robots to these abstract capacities
while respecting disjoint membership and redundancy margins. This problem is
formulated as follows.

\begin{problem}\label{qs:Team_Formation}
  Given the assignment $\nu_{K^\star}^\star$ and the fleet $\mathcal{N}$,
  determine the team formation
  $\overline{\mathcal{N}}=\{\mathcal{N}_1,\cdots,\mathcal{N}_{K^\star}\}$,
  where $\mathcal{N}_k$ is the set of robots assigned to team $\mathcal{C}_k$.
  Each team must satisfy the required capacities, and membership must be
  disjoint, i.e., $\mathcal{N}_{k_1}\cap\mathcal{N}_{k_2}=\emptyset$ for
  $k_1\neq k_2$. The objective is to {min}imize the overall response
  time as defined in~\eqref{eq:objective}, thereby ensuring efficient
  execution of all assigned tasks.
  \hfill $\blacksquare$
\end{problem}

The team formation problem is addressed using a redundancy-aware mixed-integer
linear programming (MILP) formulation. This formulation explicitly enforces the
capacity constraints of each team, while also incorporating redundancy margins
that allow flexibility in resource allocation. The approach can be summarized
in two key components.

\begin{figure}[t!]
  \centering
  \includegraphics[width=0.9\hsize]{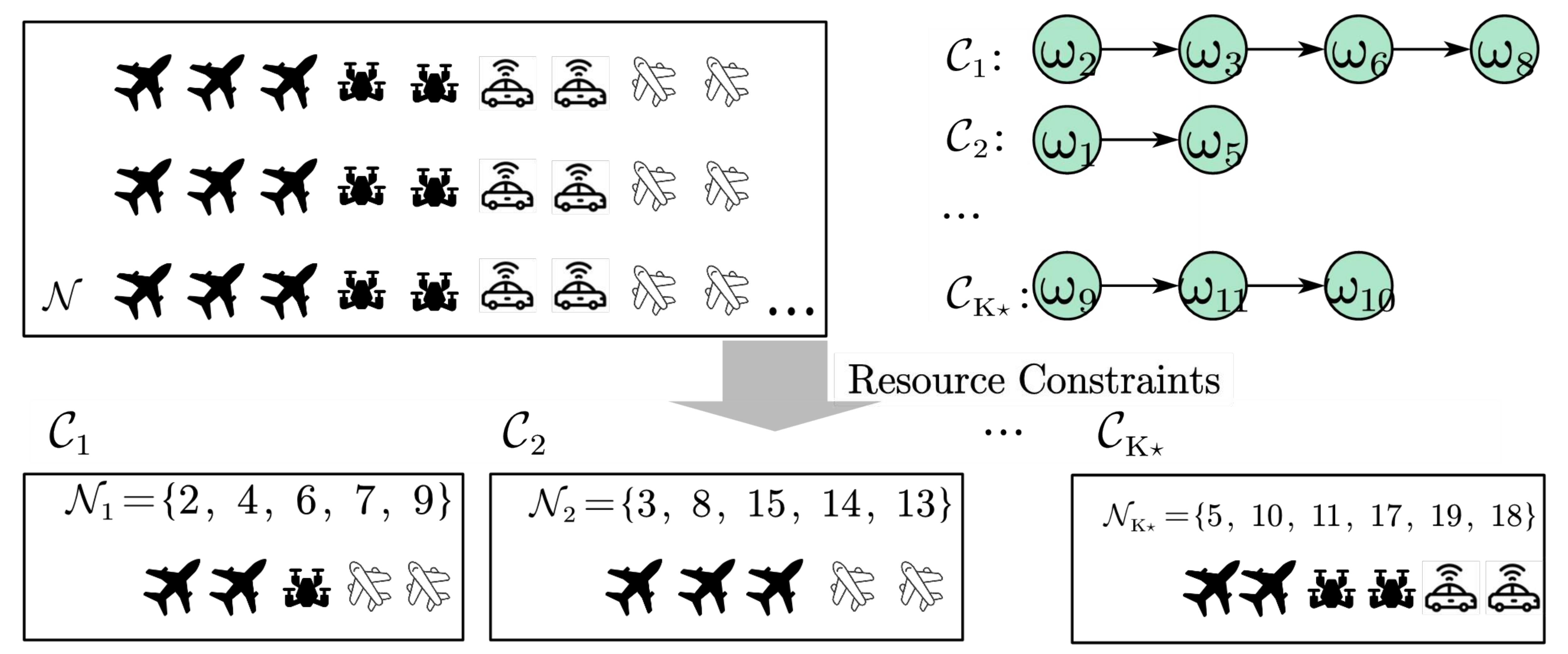}
  \vspace{-0.1in}
  \caption{{Illustration of the capacity-based and redundancy-aware team formation.
    Given the task assignment~$\nu_{K^\star}^\star$,
    the team formation~$\{\mathcal{N}_1,\cdots,\mathcal{N}_{K^\star}\}$
    are the determined optimally under the redundancy margin
    in~\eqref{eq:bound_form}}.}
  \label{fig:formation}
  \vspace{-0.1in}
\end{figure}

(I) \emph{Capacity constraints.}
Binary decision variables $b_{ik}\in\{0,1\}$ are introduced to indicate whether
robot $i\in\mathcal{N}$ is assigned to team $\mathcal{C}_k$. For each action
$a^j\in\mathcal{A}$ required by team $\mathcal{C}_k$, the number of assigned
robots must satisfy a lower bound and an upper bound as follows:
\begin{equation}\label{eq:bound_form}
\beta^j_k \; \leq \sum_{i\in\mathcal{N}} b_{ik}\,\mathds{1}(a_i=a^j)\;\leq\; \beta^j_k\alpha_j,
\end{equation}
where $\beta^j_k$ is the minimum number of robots required to perform action
$a^j$, and $\alpha_j\geq 1$ is a \emph{redundancy margin} introduced as in
Remark~\ref{rm:redundancy-margin}. The inequality above
ensures that team $\mathcal{C}_k$ is always sufficiently staffed to satisfy
the task requirements,
while still allowing limited redundancy given available resources within the fleet.
Together, these constraints provide a trade-off between
feasibility, efficiency and robustness for the team formation.

{

\begin{remark}\label{rm:redundancy-margin}
The margin $\{\alpha_j\}$ above accounts
for (i) uncertainty in the effective workload of tasks such as unknown subtasks revealed online,
and (ii) robot failure or unavailability during execution. In practice, they are selected
as a small safety factor above 1 and increased when subtask-count uncertainty or failure rate
is higher. Larger $\{\alpha_j\}$ improves robustness but may reduce efficiency by reserving
additional robots; smaller $\{\alpha_j\}$ yields higher utilization but can be less resilient to
uncertainty and failures.
\hfill $\blacksquare$
\end{remark}

}

\emph{(II) Min-max objective.}
The second component is to {min}imize the response time across all
teams. For each robot $i$ potentially assigned to team $\mathcal{C}_k$, the
expected arrival time at the first task region $S_k^1$ is
$t_{ik}\triangleq \widehat{t}_i+T_{\texttt{nav}}(\widehat{x}_i,S_k^1)$,
where $\widehat{t}_i$ is the time when robot $i$ becomes available and
$\widehat{x}_i$ is its position at that time. Once robots are assigned, the
execution cost of team $\mathcal{N}_k$ is given by
$J(k)\triangleq \textbf{max}_{i\in\mathcal{N}_k}\{t_{ik}\}
+\sum_{\omega \in \Gamma_k}T_{\texttt{exec}}(\omega)$,
where the first term captures the synchronization delay of the slowest robot,
and the second term aggregates execution times. The global objective is then
$\textbf{min}\,\{\textbf{max}_{k\in\mathcal{K}} J(k)\}$, which minimizes the
worst-case response time across all teams. This reflects the fact that overall
mission efficiency is determined by the slowest team to complete its tasks.

The above formulation constitutes a MILP problem, which can be solved by
off-the-shelf solvers such as \texttt{GLOP}~\cite{glop}. The solution provides the optimal
binary assignment $\{b_{ik}\}$, from which the team formation
$\overline{\mathcal{N}}$ is derived. Specifically,
the team formation is given by~$\mathcal{N}_k\triangleq \{i \in \mathcal{N}\mid b_{ik}=1\}$,
which defines the robots in team $\mathcal{C}_k$. In
addition, the local plan of each robot $i\in\mathcal{N}_k$ can be generated as
a timed sequence of tasks, i.e.,
\begin{equation}\label{eq:robot-plan}
  \xi_i\triangleq (S^1_k,\omega^1_k)(S^2_k,\omega^2_k)\cdots
  (S^{L_k}_k,\omega^{L_k}_k),\;\forall i\in\mathcal{N}_k;
\end{equation}
where $\omega^\ell_k$ is the $\ell$-th task in team $\mathcal{C}_k$;
$S^\ell_k$ is the associated region; and $L_k$ is the number of tasks.
This ensures that each robot has a concrete plan consistent with both the
task assignment of its team and the global mission specification.

{The redundancy-aware team formation is formulated as a MILP with binary membership
variables $b_{ik}\in\{0,1\}$, where~$i\in\mathcal{N}$ and~$k\in \{1,\cdots,K^\star\}$.
This formulation introduces~$\mathcal{O}(|\mathcal{N}|K^\star)$ integer variables and $\mathcal{O}(|\mathcal{A}|K^\star)$
capacity constraints in~\eqref{eq:bound_form}. While NP-hard in the worst case, the value
of~$K^\star$ is small in this context, ensuring that the solve time remains manageable.
The best incumbent feasible formation can be returned under a limited planning time.
In case of no feasible solutions, a greedy capacity-filling
heuristic is employed as a fallback.
Note that this team-formation MILP is significantly smaller than the
direct robot-to-subtask assignment method
without the hierarchical structure~\cite{sahin2019multirobot, kurtz2021more}.
More numerical comparisons can be found in Sec.~\ref{sec:experiments}.}

{

\begin{example}\label{example:formation}
Fig.~\ref{fig:formation} shows the team formation induced by the optimal task
assignment $\nu^\star_{K^\star}$, i.e., coalitions $\{\mathcal{C}_k\}_{k=1}^{K^\star}$.
In particular, the team $\mathcal{N}_1=\{2,4,6,7,9\}$ executes $\omega_2\omega_3\omega_6\omega_8$;
the team $\mathcal{N}_2=\{3,8,15,14,13\}$ executes $\omega_1\omega_5$; and
the team $\mathcal{N}_{K^\star}=\{5,10,11,17,19,18\}$ executes $\omega_9\omega_{11}\omega_{10}$.
\hfill $\blacksquare$
\end{example}

}

\subsection{Local Task and Trajectory Coordination within Teams}\label{subsec:act}

Given the optimal assignment~$\nu_{K^\star}^\star$, the local task plan~$\boldsymbol{\tau}_i$
of each robot~$i \in \mathcal{N}_k$ is derived as in~\eqref{eq:robot-plan}. Each robot~$i$
navigates to region~$S^\ell_k$ to start executing its $\ell$-th task~$(S^\ell_k,\, \omega^\ell_k)$.
As denoted in~\eqref{eq:task}, subtasks~$\mathcal{J}^\ell_k \triangleq \big\{(n_j,\,a_j,\,s_j),
j=1,\cdots,J^\ell_k\big\}$ must be considered to perform the task~$\omega^\ell_k$.
For task~$\omega^\ell_k$ assigned to team~$\mathcal{N}_k$, the local plan of each robot~$i \in \mathcal{N}_k$ is
given by~$\boldsymbol{\tau}_i \triangleq (t^1_i,\, p^1_i,\, a^1_i)(t^2_i,\, p^2_i,\, a^2_i)\cdots$, which is a sequence of timed
goal positions and actions. The collective plan of the team is given by
$\boldsymbol{\tau}^\ell_k \triangleq \{\boldsymbol{\tau}_i \mid i \in \mathcal{N}_k\}$.

However, the assignment of subtasks is interdependent with the optimization of the
associated trajectories, as the trajectory affects the time taken for each robot to reach
subtask positions and execute subtasks. The optimization problem must consider the trajectory of each
robot~$\mathbf{x}_i$ alongside subtask assignment, ensuring that both spatial and temporal
coordination are optimized simultaneously.
Thus, the objective is not only to assign subtasks to the robots but also to determine both the
local plans $\{\boldsymbol{\tau}_i\}$ and trajectories $\{\mathbf{x}_i\}$
of all robots in team~$\mathcal{N}_k$.
This simultaneous coordination guarantees that the overall task completion is time-efficient,
considering both the subtask execution and the robot motion.
Optimizing trajectories is crucial to minimizing
the makespan of task completion.

\begin{figure}[t!]
  \centering
  \includegraphics[width=0.9\hsize]{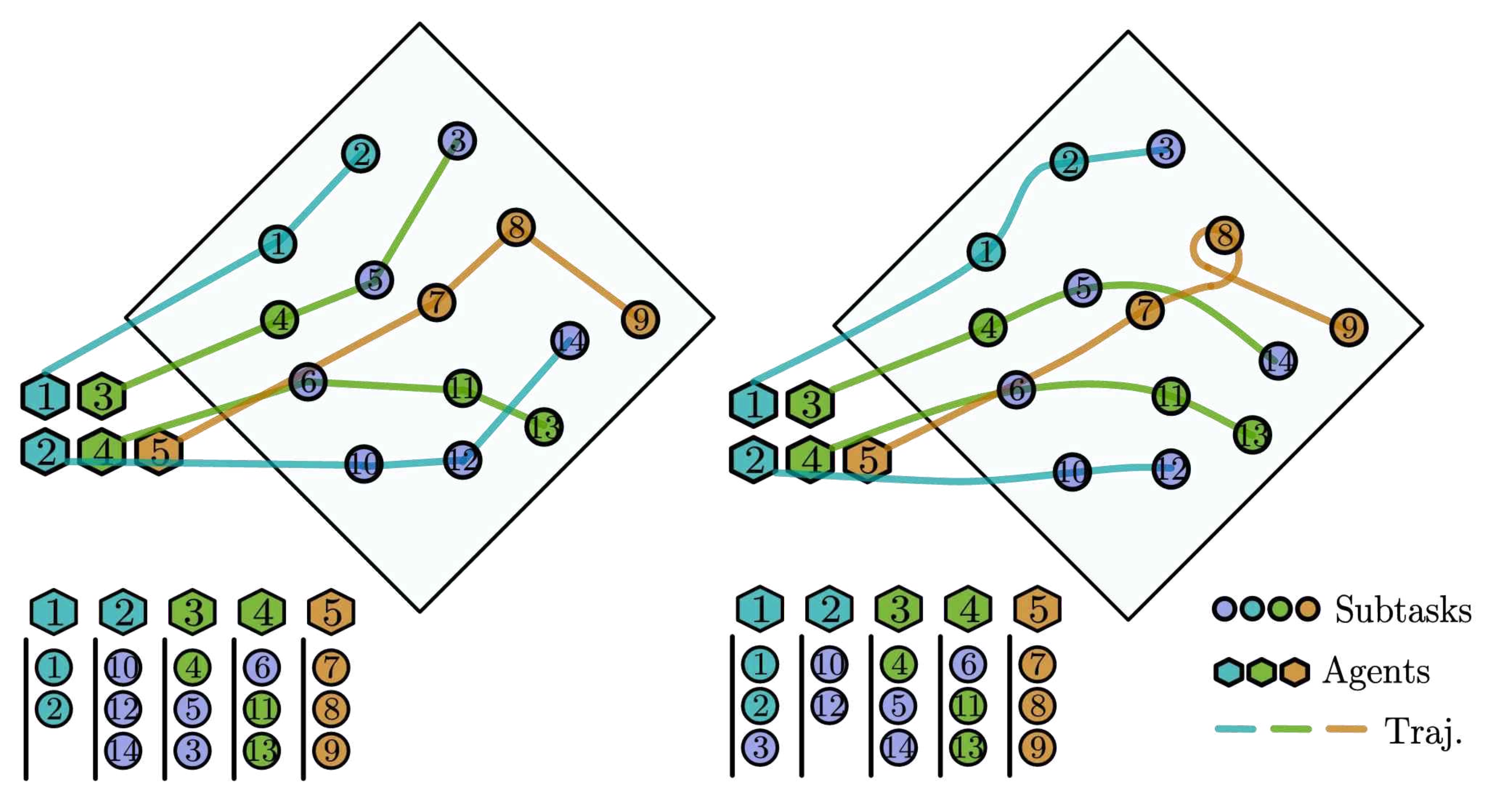}
  \vspace{-0.1in}
  \caption{    {Local coordination results for
    the static and known local tasks described in Sec.~\ref{sec:static-known}.
    In total $14$ delivery subtasks (filled circles)
    are assigned to $5$ robots (filled pentagons),
    of which the trajectories are shown
    for fully-actuated (\textbf{left})
    and non-holonomic (\textbf{right}) robots.}}
  \label{fig:StaticKnown}
  \vspace{-0.1in}
\end{figure}

\begin{problem} \label{qs:LocalTaskProblem}
  Given task~$\omega^\ell_k$ with area~$S_k^\ell$, subtasks~$\mathcal{J}^\ell_k$ as in~\eqref{eq:task}
  and the assigned team~$\mathcal{N}_k$, the goal is to determine the optimal
  local plans~$\{\boldsymbol{\tau}_i\}$ and robot trajectories~$\{\mathbf{x}_i\}$ for each robot~$i \in \mathcal{N}_k$.
  The objective is to minimize the makespan~$T^\ell_k$ as the collective execution time
  for task~$\omega^\ell_k$, while simultaneously optimizing trajectories and task assignments.
  \hfill $\blacksquare$
\end{problem}

Uncertainty in the number of subtasks~$J_k^\ell$ and their locations~$\{s_j\}$
for task~$\omega_k^\ell$ requires coordination strategies conditioned on task
properties. Static and known tasks reduce to allocation and trajectory
optimization for makespan minimization. Static but unknown tasks require
exploration with online insertion of discovered subtasks. Dynamic but known
tasks require continual online reassignment and trajectory updates. These
cases defined by the static--dynamic and known--unknown axes motivate
tailored coordination strategies below.

\subsubsection{Static and Known Subtasks}\label{sec:static-known}

In the static and known case, both the number of subtasks~$J^\ell_k$ and their
locations~$\{s_j\}$ are predetermined. This situation often arises in delivery
or inspection tasks with fixed points of interest. While the setting is
classical, the consideration of robot dynamics is essential: robots with free
holonomic motion can move directly between subtasks, whereas robots with
non-holonomic constraints must follow feasible trajectories that depend on both
position and orientation. Therefore, planning in this case requires different
treatments depending on the underlying motion model.

For holonomic robots, the problem reduces to a variant of the multi-vehicle
routing problem (MVRP) without return. Robots travel along straight-line
segments between subtasks, and the sequence of visits is determined by solving
a combinatorial optimization problem, via off-the-shelf solvers such as~\cite{glop}.
The outcome directly specifies the local
plans~$\{\boldsymbol{\tau}_i\}$, while the trajectories~$\{\mathbf{x}_i\}$ are
straight-line connections.
{However, for non-holonomic robots, the problem becomes a hybrid optimization in which
task sequencing and feasible trajectories are \emph{optimized jointly}.} Each subtask
location is augmented with an orientation, and also the robot state.
The transition between subtasks is evaluated using motion
primitives generated between states,
such as Dubins curves~\cite{ny2011dubins,vavna2015dubins}.
Each primitive is a dynamically feasible trajectory with the associated cost and duration.
{Thus, the overall objective is given by:
\begin{equation}\label{eq:augmented-cost}
  J^\ell_k \triangleq
  \underset{\{\boldsymbol{\tau}_i,\mathbf{x}_i\}}{\textbf{min}}\,
  \Bigg\{\underset{i \in \mathcal{N}_k}{\textbf{max}}\;
  \Big\{\sum_{(s_{j_1},s_{j_2})\in \boldsymbol{\tau}_i}
  \underset{\kappa \in \boldsymbol{\kappa}(x_{j_1},x_{j_2})}{\textbf{min}}\,
  \textsf{Cost}(\kappa)\Big\}\Bigg\},
\end{equation}
where~$\boldsymbol{\kappa}(x_{j_1},x_{j_2})$ denotes the predefined set
of feasible motion primitives connecting the intermediate states~$x_{j_1}$ and $x_{j_2}$.}
In practice, a hybrid-A$^\star$ like search method~\cite{erke2020improved} can
be adopted to jointly optimize the sequence of subtasks and motion primitives,
producing the local plans
$\{\boldsymbol{\tau}_i\}$ and trajectories $\{\mathbf{x}_i\}$
that minimize the makespan.

{

\begin{example}
  \label{example:StaticKnown}
  As shown in Fig.~\ref{fig:StaticKnown},
  in total~$14$ delivery subtasks are assigned to~$5$ robots under different dynamic constraints.
  Note that the subtasks require different robot capabilities.
  The resulting assignments and trajectories are significantly different
  for holonomic and non-holonomic teams.
  \hfill $\blacksquare$
\end{example}

}

\subsubsection{Static and Unknown Subtasks}\label{sec:staticunknown}

In the static and unknown case, the subtasks $\mathcal{J}_k^\ell$ of $\omega_k^\ell$
are initially unknown but become fixed once discovered. This setting is common in
search-and-rescue, where the number and locations of victims or targets in a region
are revealed only during execution. Since subtasks appear online, the static
assignment model in the first case is not directly applicable. To address this, we
propose a simultaneous exploration and coordination (SEC) method with two
components: (I) collaborative exploration of the task area and (II) dynamic
assignment of newly discovered subtasks during exploration.
During exploration, the team must survey the task region~$S_k^\ell$ to reveal
subtasks. Two strategy classes are considered here: (I) \emph{Obstacle-free spaces}:
the goal is collaborative coverage, where standard coverage methods, e.g.,
polygon-based decomposition, apply while accounting for heterogeneous robot
velocities, initial positions, and perception radii. (II) \emph{Cluttered or
structured workspaces}: robots must plan around obstacles; frontier-based
exploration is commonly used~\cite{holz2010evaluating}, with trajectories
respecting the robot dynamic constraints. In both cases, coordination avoids
redundant search and ensures full coverage. The output is the exploration
trajectories $\{\mathbf{x}_i\}$ for all $i\in\mathcal{N}_k$.

\begin{figure}[t!]
  \centering
  \includegraphics[width=0.9\hsize]{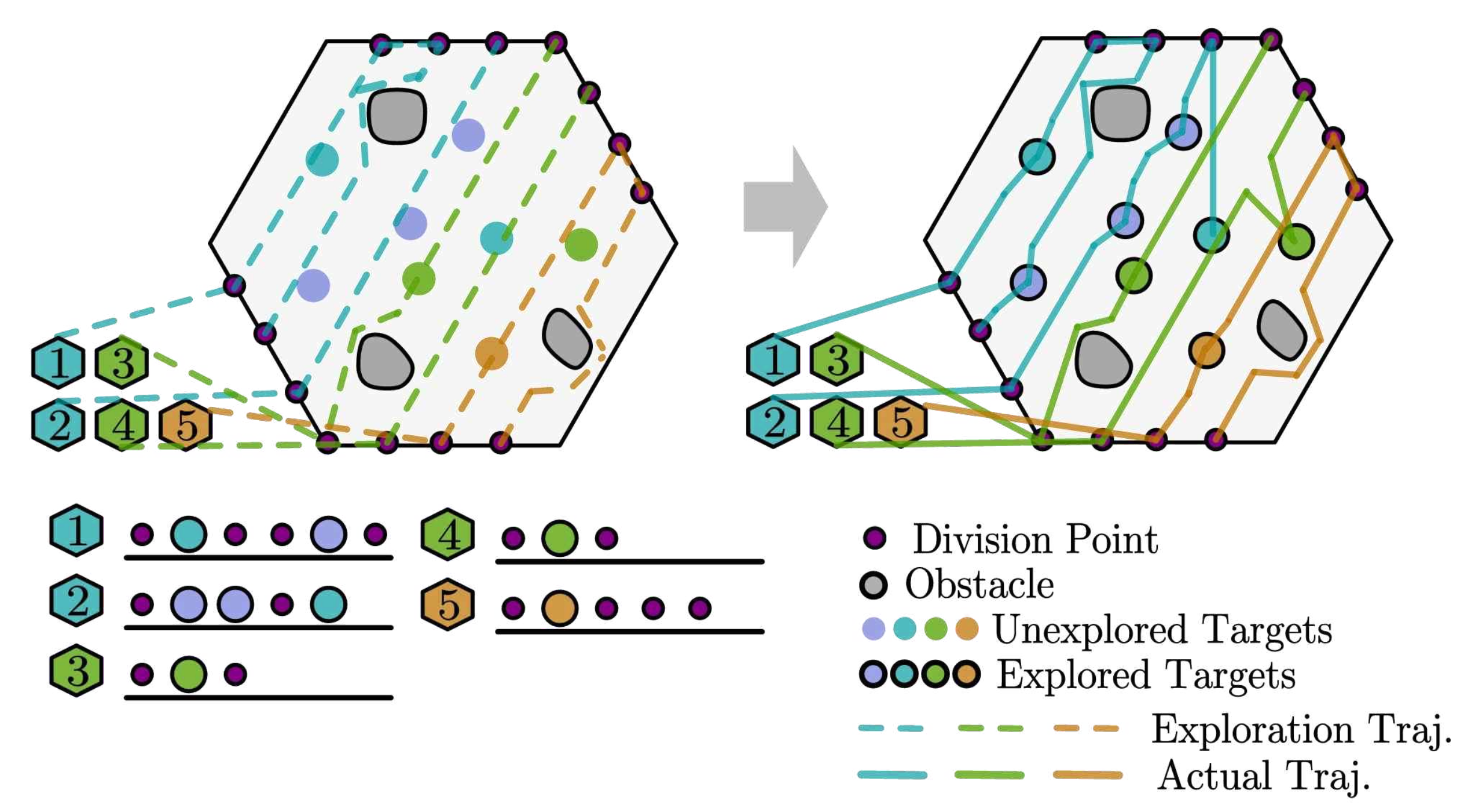}
  \vspace{-0.1in}
  \caption{    {Execution results for the case of static and unknown subtasks
      in Sec.~\ref{sec:staticunknown}, e.g., the ``search and rescue'' task.
      \textbf{Left}: The initial trajectories for coverage by~$5$ robots
      with unknown targets (in circles);
    \textbf{Right}: The actual trajectories after discovering,
    assigning and executing~$8$ rescue subtasks.}}
  \label{fig:StaticUnknown}
  \vspace{-0.1in}
\end{figure}

Furthermore, once the region begins to be explored and new subtasks are discovered, the task of assigning
and fulfilling these subtasks must be done simultaneously with exploration. A minimum-disruption
planning method is proposed, where newly discovered subtasks are inserted into the local plans
of each robot. The objective is to minimize the increase in makespan due to these insertions.
Specifically, once a subtask~$j \in \mathcal{J}_k^\ell$ that meets the robot capabilities~$a_j \in \mathcal{A}_i$
is discovered, the robot updates its local plan to include the subtask.
{The task sequence for each robot~$i \in \mathcal{N}_k$
is then re-optimized to minimize the makespan for the new tasks, i.e.,
\begin{equation}\label{eq:min-disruption}
  \underset{\{(\boldsymbol{\tau}_i,\mathbf{x}_i)\}}{\textbf{min}}\;
  \Big\{\sum_{j=1}^{L_k^\ell} \left(\textsf{Cost}(\boldsymbol{\tau}_i^+-
  \textsf{Cost}(\boldsymbol{\tau}_i))\right)\Big\},
\end{equation}
where $\boldsymbol{\tau}_i^+$ denotes the updated local plan after inserting the
newly discovered subtasks, and the objective quantifies the resulting increase in
execution cost and makespan.} Subtasks are assigned online by selecting the robot
that yields the smallest incremental cost in~\eqref{eq:min-disruption}. The
assignments are recomputed either periodically or after substantial progress on
the current set. {The output is the updated local sequences $\{\boldsymbol{\tau}_i\}$
  and corresponding trajectories $\{\mathbf{x}_i\}$.
  Finally, non-holonomic constraints are handled as in the static and
known case by planning over motion primitives to ensure feasible trajectories.}

{

\begin{example}
  \label{example:StaticUnknown}
  As shown in Fig.~\ref{fig:StaticUnknown},
in total $5$ robots are assigned to execute the task with~$8$ unknown targets.
The initial trajectories of exploration minimizes the coverage time by the computed division points.
More targets are discovered during exploration
and assigned via the proposed scheme.
The resulting trajectories not only explore the task region,
but also fulfill all discovered subtasks.
  \hfill $\blacksquare$
\end{example}

}

\subsubsection{Dynamic and Known Subtasks}\label{sec:dynamicknown}
For the third case, the set of subtasks~$J_k^\ell$ within
$\omega_k^\ell$ and their initial locations~$s_j$ are known, but the
subtasks are \emph{dynamic} moving during execution. Their real-time
positions are assumed to be available through external sensing or tracking
systems~\cite{zhou2018resilient}. Typical examples include collaborative
pursuit or capture tasks, where the robots must form coalitions to surround and
intercept moving targets. In such settings, static plans quickly become
suboptimal or infeasible as the subtasks evolve over time.

\begin{algorithm}[t!]
  \caption{Distributed Local Update for Dynamic and Known Subtasks}
  \label{alg:DynamicKnown}
  \SetAlgoLined
  \KwIn{Subteam $\mathcal{N}_k$, subtasks $\mathcal{J}_k^\ell$,
        cost~$\chi(\cdot)$.}
  \KwOut{Updated~$\{\boldsymbol{\tau}_i(t)\}$ and~$\{\mathbf{x}_i(t)\}$.}
  \nl \textbf{repeat} at each planning iteration for each robot~$i\in \mathcal{N}_k$\;
  \nl \quad Collect neighbor memberships and $\{\chi(\mathcal{R}_j)\}$\;
  \nl \quad Set $j' \gets {\textbf{argmin}_j}\;\chi(\mathcal{R}_j)$\;
  \nl \quad \textbf{if} \eqref{eq:switch-condition} holds
           \textbf{then}\;
           \nl \qquad Send switch intent to robot~$j'$\;
           \nl \quad \textbf{if} no higher-priority conflicts\;
           \qquad Set $j_i\!\gets\!j'$ and synchronize\;
 \nl \textbf{until} no robots can switch\;
 Update $\{\boldsymbol{\tau}_i(t)\}$ and $\{\mathbf{x}_i(t)\}$ given~$\overline{\mathcal{R}}_t$\;
\end{algorithm}

A distributed \emph{dynamic coalition formation} (DCF) method is adopted, in
which each robot~$i \in \mathcal{N}_k$ updates its coalition membership
using local information and peer communication~\cite{dai2024dynamic, guerrero2012multi}.
At any time~$t$, each
subtask~$j \in \mathcal{J}_k^\ell$ is associated with a
coalition~$\mathcal{R}_j(t) \subseteq \mathcal{N}_k$, and the collection of
all coalitions forms a scheme
$\overline{\mathcal{R}}_t \triangleq \{\mathcal{R}_j(t),\, j\in \mathcal{J}_k^\ell\}$.
The coalitions are disjoint and collectively exhaustive, i.e.,
$\mathcal{R}_{j_1}(t)\cap \mathcal{R}_{j_2}(t)=\emptyset$ for $j_1 \neq j_2$,
and $\bigcup_{j\in \mathcal{J}_k^\ell}\mathcal{R}_j(t)=\mathcal{N}_k$.
The local plan~$\boldsymbol{\tau}_i(t)$ is determined by its current
coalition, i.e., $\boldsymbol{\tau}_i(t)=(t_i,p_i,a_i)$ for the currently
assigned subtask, and its trajectory~$\mathbf{x}_i(t)$ follows the coalition
decision. To guide updates, each robot maintains local estimates of the
coalition costs~$\chi(\mathcal{R}_j)$, which depend on the robot–target
distances and the velocities. {The instantaneous
team cost is defined as:
\begin{equation}\label{eq:coalition-cost}
  \chi(\overline{\mathcal{R}}_t)
  \triangleq \underset{j \in \mathcal{J}_k^\ell}{\textbf{max}}\,
  \chi(\mathcal{R}_j)
  + \frac{1}{J_k^\ell}\sum_{j\in \mathcal{J}_k^\ell}\chi(\mathcal{R}_j),
\end{equation}
where the first term reflects the current makespan across all subtasks,
and the second term encourages a balanced workload.}
Based on~\eqref{eq:coalition-cost}, the local switch rule is stated from the
perspective of each robot and is summarized in
Alg.~\ref{alg:DynamicKnown}. At time~$t$, robot~$i$ knows its
assigned task~$j_i$ and the associated coalition~$\mathcal{R}_{j_i}$.
More importantly, it can estimate~$\chi(\mathcal{R}_{j'})$ for nearby
subtasks~$j'\in \mathcal{J}_k^\ell$.
Then, the robot~$i$ can evaluate the following condition locally:
\begin{equation}\label{eq:switch-condition}
  \begin{split}
&\textbf{max}\!\left\{\chi(\mathcal{R}_{j'} \bigcup \{i\}),\,
\chi(\mathcal{R}_{j_i} \backslash \{i\})\right\}\\
&\qquad < \textbf{max}\!\big{\{}\chi(\mathcal{R}_{j_i}),\, \chi(\mathcal{R}_{j'})\big{\}},
\end{split}
\end{equation}
where $\mathcal{R}_{j'}$ is the candidate coalition and $\mathcal{R}_{j_i}$
is the current coalition of robot~$i$. If the inequality holds, an intent to
switch is announced to neighbors in $\mathcal{R}_{j'}$ and $\mathcal{R}_{j_i}$,
a priority rule resolves potential conflicts, and afterwards
a synchronization finalizes the change. Robot~$i$ then updates
its local plan~$\boldsymbol{\tau}_i(t)$ and trajectory~$\mathbf{x}_i(t)$ to track the new subtask.
This local rule uses only the neighborhood costs and memberships, preserves disjoint coalitions, and
drives the scheme toward a locally optimal
$\overline{\mathcal{R}}_t^\star$, yielding time-varying plans
$\{\boldsymbol{\tau}_i(t)\}$ and trajectories $\{\mathbf{x}_i(t)\}$ that
continuously adapt to the moving subtasks.

\begin{figure}[t!]
  \centering
  \includegraphics[width=0.9\hsize]{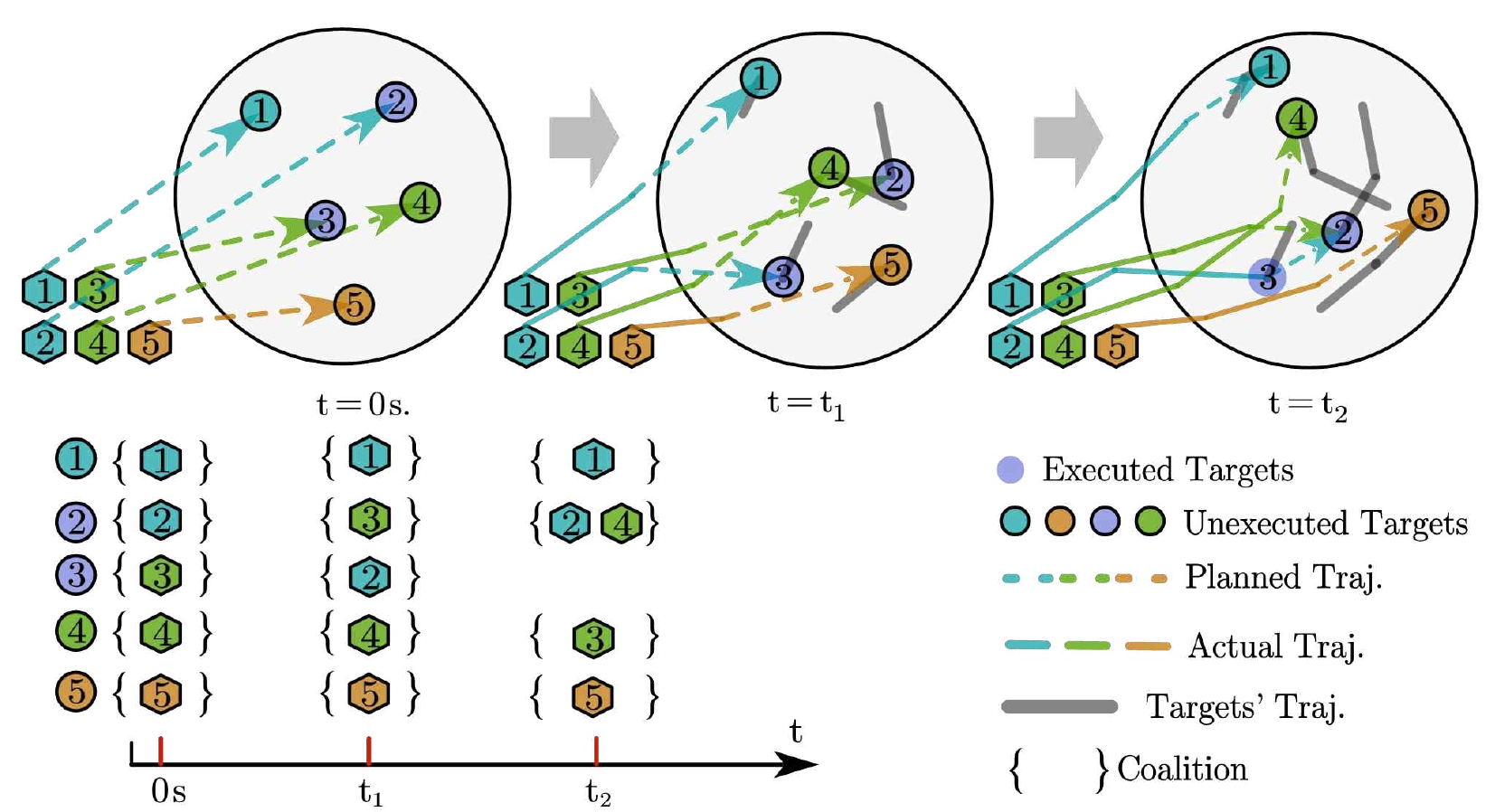}
  \vspace{-0.1in}
  \caption{    {Execution results of the dynamic and known subtasks in Sec.~\ref{sec:dynamicknown}.
    \textbf{Top}: the trajectories of $5$ robots (filled hexagons) and $5$ targets (filled circles)
    at different snapshots;
    \textbf{Bottom}: the evolution of assigned coalitions for each target.
  Note that different colors reflect different capacity constraints.}}
  \label{fig:DynamicKnown}
  \vspace{-0.1in}
\end{figure}

{

\begin{example}
  \label{example:DynamicKnown}
  As shown in Fig.~\ref{fig:DynamicKnown},
$5$ robots are assigned to~$5$ moving targets by forming dynamic coalitions online in Alg.~\ref{alg:DynamicKnown}.
Initially, each robot is assigned to one target.
However, robots~$2$ and~$3$ exchange their targets at~$t_1$.
When target~$3$ is finished at~$t_3$, robots~$2,4$
form into a coalition to execute target~$2$ while robot~$3$ executes target~$4$.
  \hfill $\blacksquare$
\end{example}

}

{

\begin{remark}\label{rm:comm-constraints}
  The local coordination policy in Sec.~\ref{subsec:act} relies on inter-robot
information exchange over the communication graph induced by the communication radius. Connectivity is the key structural requirement: if the
graph is disconnected, coordination messages  cannot
propagate and effective collaboration cannot be realized.
Moreover, latency, packet loss, and bandwidth limitations
primarily slow information flow, increasing the planning time and causing
transient inconsistencies in $\{\boldsymbol{\tau}_i\}$.
However, when the graph is connected sufficiently often
over time windows, delayed or missed updates are corrected once communication
resumes. More realistic models of intermittent communication with joint planning
of task and communication events remain our ongoing work.
\hfill $\blacksquare$
\end{remark}

}

\begin{remark}\label{rm:case}
Note that the case of dynamic and unknown tasks is not considered here,
due to two reasons: (I) without knowing the total number and locations
of subtasks, it is difficult to determine whether the current task is completed,
e.g., subtasks may move dynamically from unexplored areas to explored areas;
(II) if such a task exists, the coordination strategy would be a combination
of the second and third cases above, i.e., to assign the subtasks of exploration and collaboration
via dynamic coalition formation.
    \hfill $\blacksquare$
\end{remark}

{The theorem below presents the correctness guarantee for the local coordination policies
  proposed above. Detailed proof is attached in the Appendix.}

{\begin{theorem}
\label{thm:local-correctness}
For each task~$\omega_k^\ell$ and team~$\mathcal{N}_k$, the strategies in
Sec.~\ref{subsec:act} ensure that the plans~$\{\boldsymbol{\tau}_i\}$
and trajectories~$\{\mathbf{x}_i\}$ are feasible,
and minimize the local task makespan.
\end{theorem}}

\subsection{Online Execution, Interaction and Adaptation}\label{subsec:online}

\subsubsection{Online Execution and Receding-horizon Adaptation}
\label{subsubsec:execute}

\begin{figure}[t!]
  \centering
  \includegraphics[width=0.9\hsize]{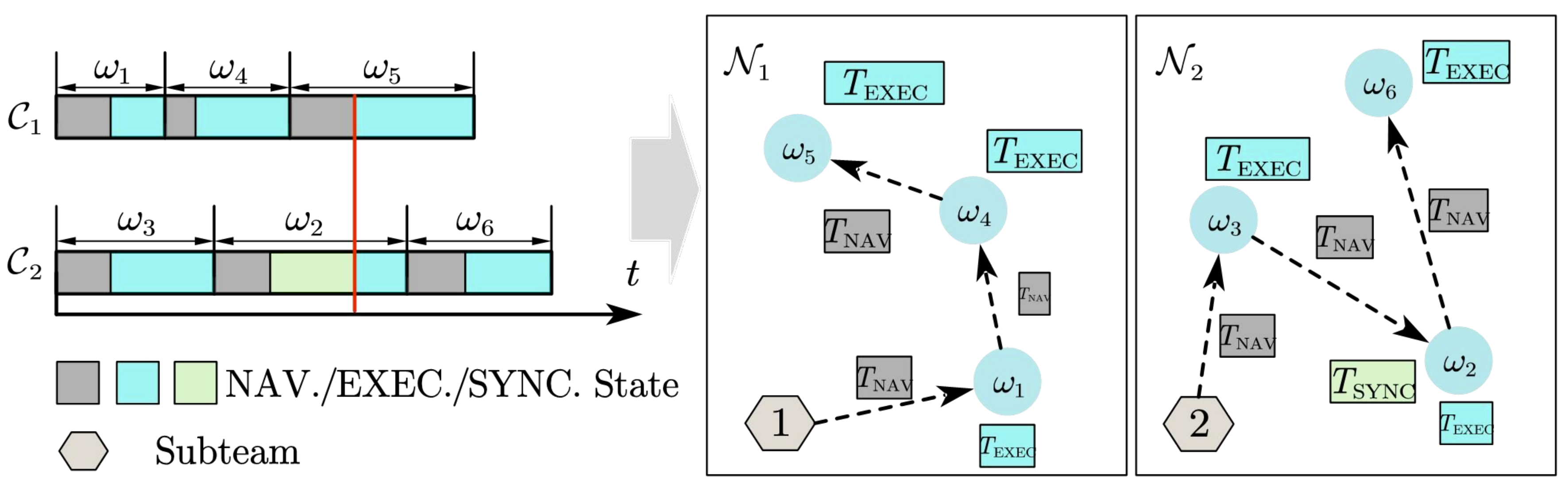}
  \vspace{-0.1in}
  \caption{  {Illustration of the change of execution status
    between navigation, execution and
    synchronization (\textbf{left}),
    for a team~$\mathcal{C}_k$
    given its local plan~$\Gamma_k$ (\textbf{right})
    in~\eqref{eq:termination} and~\eqref{eq:robot-plan}.
    Note that the concurrency constraint between tasks~$\omega_2,\omega_5$
    are enforced by the synchronization state (red line).
    }}
  \label{fig:Assign_Time}
  \vspace{-0.1in}
\end{figure}

Given the initial workspace and mission descriptions, the
missions are decomposed into tasks and the local teams~$\overline{\mathcal{N}}$
are formed. A finite set of~$H$ tasks is then
assigned with redundancy by Alg.~\ref{alg:ss-TaskAssign}, yielding the local
plan~$\Xi_k$ for each team
$\mathcal{N}_k \in \overline{\mathcal{N}}$.
Then, each team executes
$(S^\ell_k,\omega^\ell_k)\in \xi_k$ by navigating to~$S^\ell_k$ and
performing~$\omega^\ell_k$.
{Specifically for robots in teams,
synchronization is required before execution.
Not only the robots within the team should collaborate to perform tasks,
but also there are temporal constraints between tasks including precedence or concurrency.
This yield three distinctive states during online execution: \emph{navigation}, \emph{execution} and \emph{synchronization}.}
When execution begins, the subtasks~$\mathcal{J}^\ell_k$ are
coordinated by three local strategies to derive the action
plans~$\{\boldsymbol{\tau}_i\}$. Execution is concurrent across teams
and parallel within each team.
{This process continues
until all local plans $\{\Xi_k\}$ are completed, or a replanning is triggered.}

{

\begin{example}
  \label{example:Assign_Time}
  As shown in Fig.~\ref{fig:Assign_Time},
  teams~$\mathcal{N}_1$ and $\mathcal{N}_2$
  are constructed from capacities~$\mathcal{C}_1$ and $\mathcal{C}_2$ by team formation.
  After task~$\omega_2$, team~$\mathcal{N}_2$ should wait for team~$\mathcal{N}_1$
  to perform task~$\omega_5$,
due to the concurrency constraint for~$\omega_2,\omega_5$.
  \hfill $\blacksquare$
\end{example}

}

Furthermore,
a receding-horizon scheme handles {replanning for online tasks and operator updates.}
As shown in Fig.~\ref{fig:receding_horizon}, only~$H$ tasks are assigned per cycle,
and the remaining tasks are considered in the next cycle. Replanning is triggered by:
(I) execution progress, when more than half of the planned tasks are accomplished;
(II) new mission specifications, when new missions or modifications arrive;
and (III) feasibility,
when any team reports infeasibility due to execution failures.

Upon any of the triggering conditions,
the task assignment and team formation are recomputed by
Alg.~\ref{alg:ss-TaskAssign} using the current system state. The candidate
pool includes tasks that are assigned but not yet started,
unassigned tasks and newly specified tasks, which enables consistent revision of priorities and team
compositions.
However, tasks currently in execution are not preempted, and the associated
teams continue execution without interruption.
This non-preemption policy preserves
the safety and continuity, which is crucial when the tasks outnumber the required teams.
Moreover, the robot states
are maintained by rolling forward the execution status: for team~$\mathcal{N}_k$ with
the current sequence~$(\omega^0_k,\cdots,\omega^{L_k}_k)$, where~$\omega^{L_k}_k$
is ongoing, it holds that~$\widehat{t}_i \triangleq t_{\texttt{e}}(\omega_k^{L_k})$
and~$\widehat{x}_i \triangleq S_k^{L_k}$ for each~$i \in \mathcal{N}_k$.
This enables consistent future formations under temporal ordering constraints.

\begin{figure}[t!]
  \centering
  \includegraphics[width=0.9\hsize]{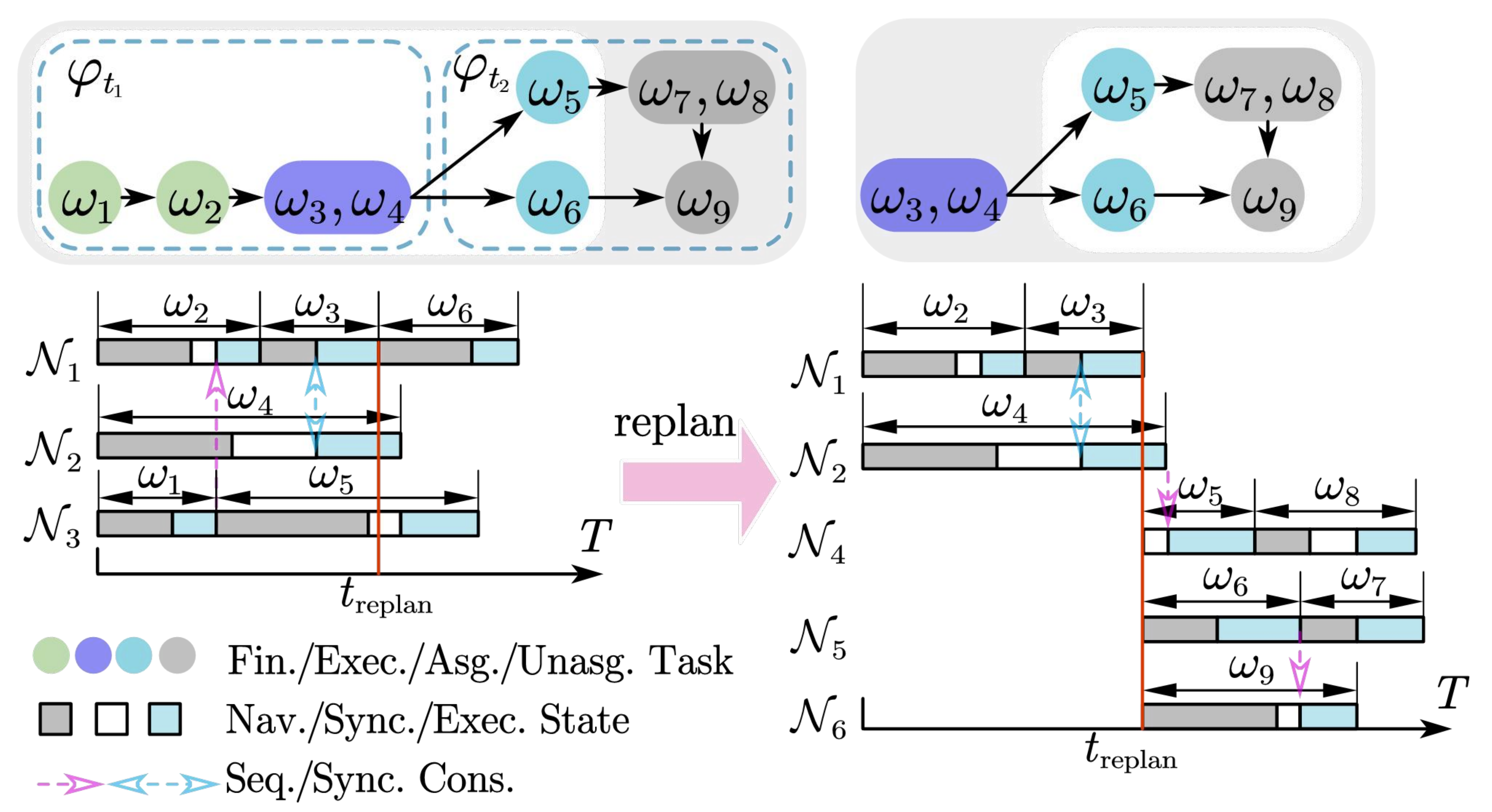}
  \vspace{-0.1in}
  \caption{    {Illustration of the receding-horizon scheme for coordination
    and online adaptation in Sec.~\ref{subsec:online}.
    Note that the number of teams is increased from~$3$ (\textbf{left})
    to~$5$ after the task assignment and team formation (\textbf{right}),
    and the tasks~$\omega_3, \omega_4$ are not preempted after replanning.}}
  \label{fig:receding_horizon}
  \vspace{-0.1in}
\end{figure}

\begin{figure*}[t!]
  \centering
  \includegraphics[width=0.97\hsize]{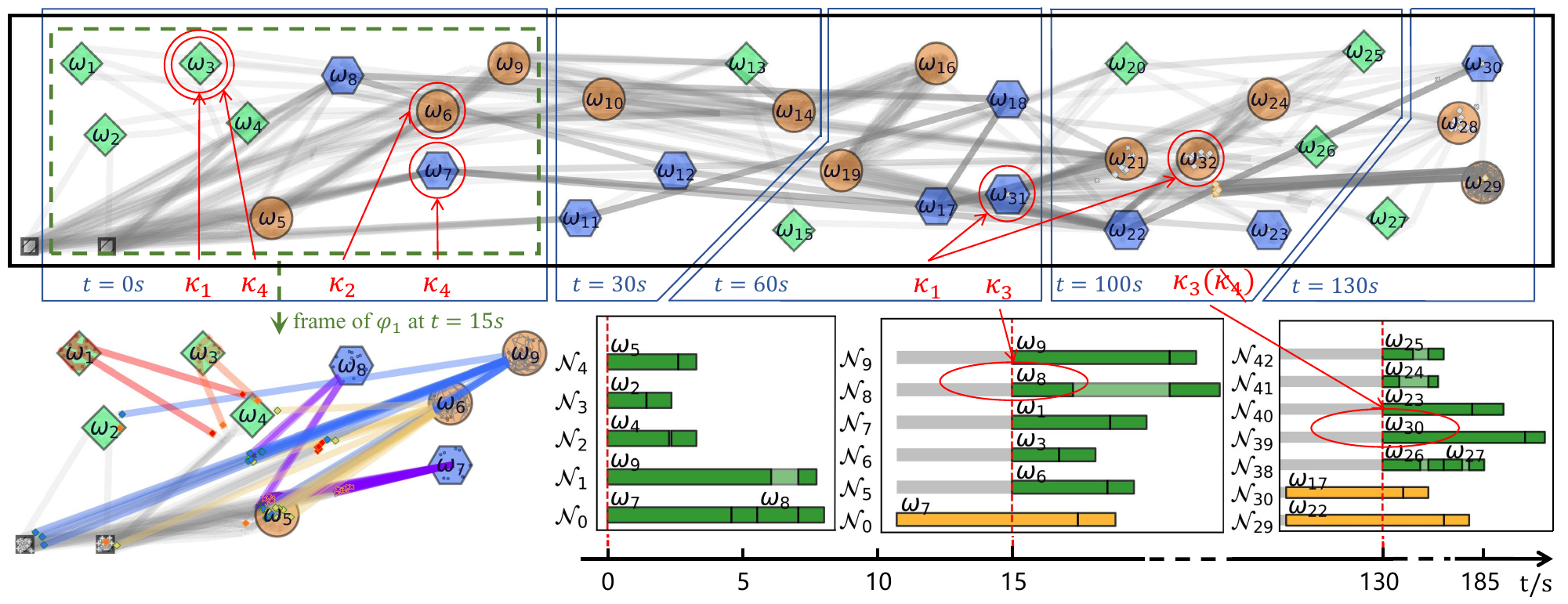}
    \vspace{-0.15in}
    \caption{{Simulated human-in-the-loop scenario of $80$ robots
      in dynamic environment.
      \textbf{Top}: in total $32$ tasks and $482$ subtasks are performed during
      the mission time of $180.4 \,\mathrm{s}$,
      with the time elapsed from left to right.
      Clearly~$\kappa_1$ to~$\kappa_4$ are issued in the process;
      \textbf{Bottom-left}: the snapshots of execution of tasks at~$t=15 \,\mathrm{s}$,
      where different markers indicate different types
      of robots;
      \textbf{Bottom-right}: gantt charts of teams executing assigned tasks at~$t=0 \,\mathrm{s}$,
      $t=15 \,\mathrm{s}$ and~$t=130 \,\mathrm{s}$.
      Note that tasks $\omega_1-\omega_9$ are released at~$t=0 \,\mathrm{s}$,
      $\omega_{10}-\omega_{14}$ at~$t=30 \,\mathrm{s}$,
      $\omega_{15}-\omega_{19}$ at~$t=60 \,\mathrm{s}$,
      $\omega_{20}-\omega_{26}$ at~$t=100 \,\mathrm{s}$,
      $\omega_{27}-\omega_{30}$ at~$t=130 \,\mathrm{s}$,
      and $\omega_{31},\omega_{32}$ are issued by operator at~$t=110 \,\mathrm{s}$.
    }}
    \label{fig:sim_scenario}
  \vspace{-0.1in}
\end{figure*}

{
\begin{remark}\label{rm:parameters}
Note that the choice of the key parameters including the rolling horizon $H$,
the redundancy margin $\alpha\!\ge\!1$, and the replanning trigger, such as
the completion ratio of committed tasks or event-driven updates,
typically depends on the expected mission volatility, e.g., the arrival and
cancellation rate in $\Phi_t$, the uncertainty in task and target estimation,
e.g., location and service-time dispersion, and the anticipated robot failure
and communication drop rate, which together balance look-ahead optimality,
robustness, and responsiveness. Numerical analyses on how their values affect
the overall performance are provided in Sec.~\ref{sec:experiments}.
\hfill $\blacksquare$
\end{remark}
  }

\subsubsection{Fulfillment of Online Human Requests}\label{subsubsec:interact}

Online operator requests are routed from the graphic interface including text, voice and templates, following the protocol for online requests in Sec.~\ref{subsec:hfi}.
These requests are processed by the task decomposition module and the proposed planning stack. Each request induces a bounded update of the search tree
$\mathfrak{T}$ and the frontiers~$\overline{\mathcal{V}}$,
followed by a re-optimization of teams $K$, capacities $\{\mathcal{C}_k\}$, and plans
$\{\Gamma_k\}$ by Alg.~\ref{alg:ss-TaskAssign}. The updated plans are
forwarded to the team execution as in Sec.~\ref{subsec:act},
under the non-preemption and receding-horizon scheme in
Sec.~\ref{subsubsec:execute}, {while the visualization interface immediately
reflects changes in the task allocation, temporal relations among tasks and {Gantt} timelines,
as summarized in Sec.~\ref{subsubsec:interface}.}

More specifically, the four types of requests defined in
Sec.~\ref{subsec:human} are handled as follows. A new mission $\kappa_1$
is parsed and decomposed to sc-LTL, yielding~$\mathcal{B}_{\varphi_m}$
and~$\widehat{Q}^0_m$, with~$\Phi_t \!\leftarrow\! \Phi_t \cup \{\varphi_m\}$;
the planner continues the tree search by evaluating~$\chi(\nu)$
in~\eqref{eq:value}, updating~$\overline{\mathcal{V}}$
via~\eqref{eq:frontier}, and re-selecting the complete assignment
by~\eqref{eq:complete-nodes}. A cancellation~$\kappa_2$ removes~$\varphi_m$ from
$\Phi_t$, prunes nodes whose progress depends on~$\widehat{Q}_m$, and
recomputes~$\overline{\mathcal{V}}$, while preserving tasks already in
execution. A deadline
or priority update~$\kappa_3$ modifies the value function and costs by
re-weighting the progress and adding deadline penalties within~$\chi(\nu)$
and~$C_k$ of~\eqref{eq:value} and~\eqref{eq:capacity-team},
in addition to bounding functions in~\eqref{eq:profile},
thus biasing the selection toward urgent missions.
Lastly, a direct robot reassignment~$\kappa_4$ changes its fleet $\mathcal{N}$
and states~$(\widehat{t}_i,\widehat{x}_i)$. Then, the capacity-based
formation with bounds~\eqref{eq:bound_form}
is re-solved with the min–max objective~$J(k)$ in~\eqref{eq:termination}, producing
the revised~$\{\mathcal{C}_k\}$ and~$\{\Gamma_k\}$. All requests are issued
through the same interface and acknowledged by synchronized feedback.
{
Note that request~$\kappa_4$ imposes linear admissibility
constraints on~$\{b_{ik}\}$ such as the locked and forbidden assignments,
and optional limits on membership changes.
These constraints are reflected in the receding-horizon assignment by restricting which capacity
allocations are allowed within the tree search.
Therefore, feasibility checks are performed across layers.}

{
  \begin{remark}\label{rm:conflicting}
In practice, an operator may issue
contradictory requests, e.g., enforcing immediate completion of a task via~$\kappa_3$
while restricting the only capable robot via~$\kappa_4$.
In case of conflicting operator requests, if the set of current requests~$\mathcal{K}_t$
renders the hard constraints inconsistent,
e.g., capacity bounds~\eqref{eq:bound_form}, non-preemption in
Sec.~\ref{subsubsec:execute}, or priority constraints in $\chi(\nu)$ and
$C_k$ in~\eqref{eq:capacity-team}, then
Alg.~\ref{alg:ss-TaskAssign} and the formation objective~$J(k)$ in~\eqref{eq:termination}
becomes infeasible. This infeasibility is detected at the module
where it arises, thus an explicit warning is issued with the violated constraint
class and implicated missions or robots, which prompts the operator to revise
$\mathcal{K}_t$ before replanning proceeds.
\hfill $\blacksquare$
\end{remark}
  }

\subsubsection{Complexity and Scalability Analysis}\label{subsubsec:complexity}

The computational complexity of the proposed method is analyzed as follows.
{In each iteration of Alg.~\ref{alg:ss-TaskAssign}, selecting a batch of~$P$
nodes and expanding up to~$B \triangleq (K{+}1)\,|\Omega^-_{\nu_h^\star}|$
children incurs $\mathcal{O}(PBK)$ for the automaton-state
updates~$\widehat{Q}^+_m$, together with the capacity and timing
updates in~\eqref{eq:update-constraint} and~\eqref{eq:estimated_time};
maintaining the non-dominated frontier based
on~\eqref{eq:profile}-\eqref{eq:frontier} requires
$\mathcal{O}(|\mathcal{V}|(2K{+}M))$ to reach
the completeness condition in \eqref{eq:complete-nodes}. The capacity-based
team-formation MILP introduces $\mathcal{O}(N K^\star)$ binary variables $b_{ik}$,
which is NP-hard in general but limited by the small number of teams.
For the case of static and known tasks, the
routing problem with subtour elimination is NP-hard,
and its number of variables scales with $|\mathcal{N}_k|(J_k^\ell)^2$.
For static and unknown tasks,
polygonal coverage generation and assignment are near-linear in the region
representation, while rolling insertion of newly discovered subtasks is
$\mathcal{O}(|\mathcal{N}_k|)$ per discovery. Lastly, for dynamic and known tasks, the
distributed coalition updates are linear in
team size and the number of active targets per iteration.
Although the
local MILPs and coalition updates are NP-hard in the worst case, the hierarchical
decomposition confines them to small teams, enabling event-based
replanning with standard solvers.
Detailed empirical runtime evaluations of each component and scaling benchmarks are provided in Sec.~\ref{sec:experiments}.}

\section{Numerical Experiments} \label{sec:experiments}
For numerical validation, the proposed method is implemented in \texttt{Python 3}
and tested on a laptop with an Intel Core Ultra 9 285K CPU.
The solver \texttt{GLOP}~\cite{glop} is used for integer optimization.
{Simulation videos can be found in the supplementary files.}

\begin{figure}[t!]
  \centering
  \includegraphics[width=0.97\hsize]{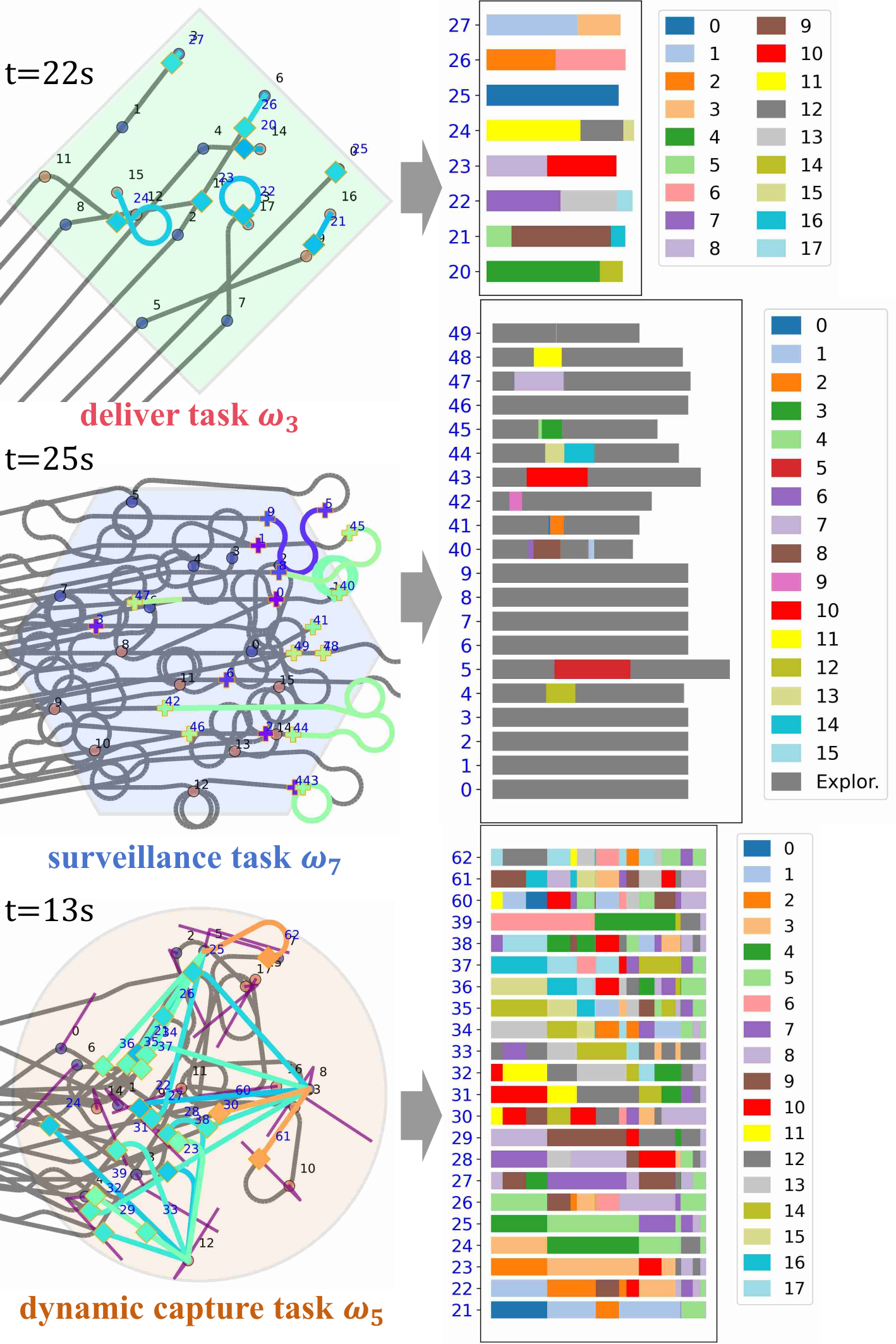}
  \vspace{-0.1in}
  \caption{Final trajectories of local task executions for
    robots under the dynamic constraints of maximum
    curvature~$\kappa \leq 3$.
  \textbf{Top}: delivery task~$\omega_3$ at $t=22 \, \mathrm{s}$ for
  a team of~$8$ robots and~$18$ subtasks;
  \textbf{Middle}: surveillance task~$\omega_7$ at $t=25 \, \mathrm{s}$ for
  a team of~$20$ robots and~$16$ subtasks;
  \textbf{Bottom}: dynamic capture task~$\omega_5$ at $t=13 \, \mathrm{s}$ for
  a team of~$22$ robots and~$18$ subtasks.}
  \label{fig:subteam_sim}
  \vspace{-0.1in}
\end{figure}

\subsection{System Description}\label{subsec:description}
As shown in Fig.~\ref{fig:sim_scenario}, the simulated fleet consists of $N=80$
heterogeneous robots operating in an open environment of size $260\,\mathrm{m}\times
42\,\mathrm{m}$.
There are three robot types with complementary capabilities:
$20$ Type-A robots that perform perception and delivery;
$20$ Type-B robots that perform perception and grasping;
and $40$ Type-C robots that perform delivery and grasping.
Robots are initially distributed evenly over two bases.
Unless otherwise specified, all robots follow a first-order dynamics model with a
maximum speed of $2.5\,\mathrm{m/s}$ in simulation.
To assess the effect of motion feasibility, both robots with curvature
constraints and robots without curvature constraints are considered.
The curvature limit is set to~$3\,\mathrm{m^{-1}}$.

There are $|\boldsymbol{\varphi}_t|=5$ missions released online at random time
instants, with inter-arrival times drawn from a normal distribution with mean
$\mu=32.5\,\mathrm{s}$ and standard deviation $\sigma=5\,\mathrm{s}$; samples are
truncated to positive times.
Upon release, mission locations are placed within the workspace according to a
spatially uniform distribution.
Each mission specified at time $t_i$ follows the sc-LTL template below:
\[
\varphi_i \;=\; \Diamond\!\big(\varphi^i_{\texttt{del}} \wedge
\Diamond\,\varphi^i_{\texttt{surv}}\big)\;\wedge\;\big(\neg\,\varphi^i_{\texttt{cap}}
\,\mathcal{U}\,\varphi^i_{\texttt{surv}}\big),
\]
where the three task types are as follows.
A delivery task requires two distinct subtasks to be completed through delivery or
grasping actions.
A surveillance task requires perception.
A dynamic capture task requires two distinct subtasks to be completed through
delivery or grasping actions while targets move.
Delivery tasks contain on average $13$ subtasks.
Surveillance tasks contain on average $15$ subtasks, and each subtask is initially
unknown with probability $0.5$.
Capture tasks include approximately $17$ moving targets with speed
$0.5\,\mathrm{m/s}$ within the designated region.
The global planning horizon is $H=6$, and replanning is triggered.
Replanning occurs upon sufficient execution progress, upon the release or
modification of missions, and upon detected infeasibility.
{Note that operator commands are issued in real-time throughout the simulation,
via the proposed protocol.}

\subsection{Results}\label{subsec:results}

\subsubsection{Mission Decomposition and Subteam Formation}

As shown in Fig.~\ref{fig:sim_scenario}, the first mission contains $9$ tasks.
At release, tasks $\omega_1$ through $\omega_4$ are delivery, $\omega_7$ and
$\omega_8$ are surveillance, and $\omega_5$, $\omega_6$, and $\omega_9$ are
capture. The associated NBA is computed in $0.20\,\mathrm{s}$ with $7$ states
and $18$ transitions. Given these task automata, Algorithm~\ref{alg:ss-TaskAssign}
optimizes the number of subteams and predicts a makespan of $37.5\,\mathrm{s}$ when $K=5$.
The computation takes $1.36\,\mathrm{s}$.
{During this process, human intervention request
$\kappa_1$ for new mission release is integrated in real time, where a new task $\omega_3$
is added to the set.}
This process confirms the system adaptability to the online
inputs of the operator. This layer validates the top-down design: the global planner reasons on precedence
to size coalitions before any motion planning. Then the subteams are instantiated with capabilities
matched to the next admissible tasks. The resulting formations are: $\mathcal{N}_0$ with~$16$
Type-B robots for surveillance tasks $\omega_7$ and $\omega_8$, $\mathcal{N}_1$ with~$2$ Type-B
and~$14$ Type-C robots for capture task $\omega_9$, $\mathcal{N}_2$ with~$5$ Type-C robots
for delivery task $\omega_4$, $\mathcal{N}_3$ with~$5$ Type-C robots for delivery task
$\omega_2$, and $\mathcal{N}_4$ with~$2$ Type-B and~$14$ Type-C robots for capture task $\omega_5$.
This composition reflects the capability coupling in the specification: surveillance requests perception,
capture requests grasping or delivery in addition to coordination, and delivery requests transport actions.

\begin{figure}[t!]
  \centering
  \includegraphics[width=0.97\hsize]{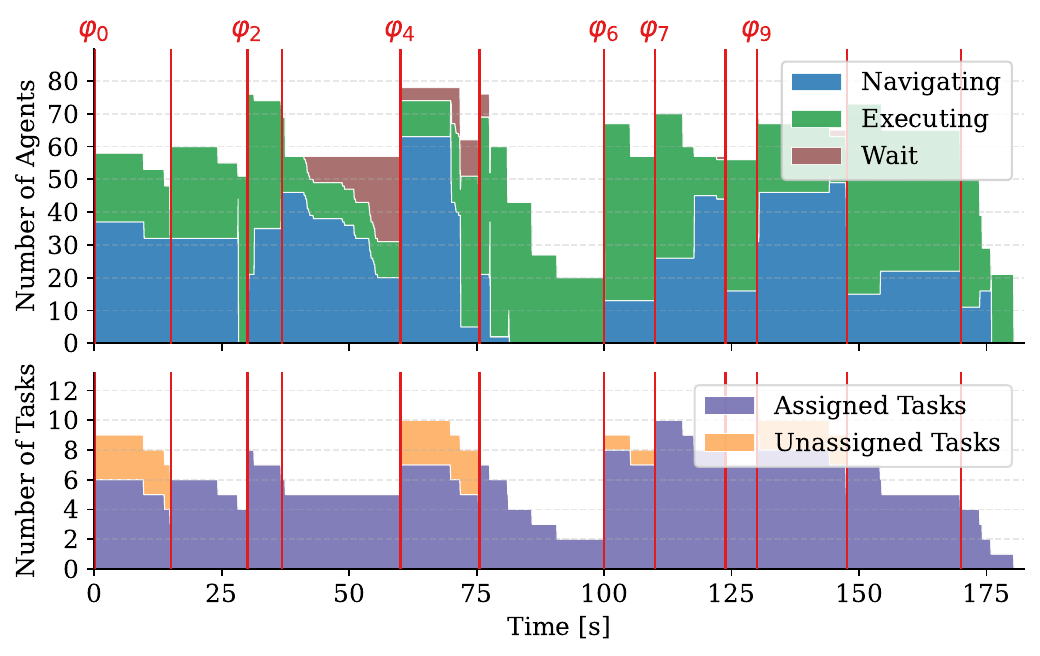}
  \vspace{-0.1in}
  \caption{{Evolution of enrolled robots and tasks accomplished,
    as new missions are specified by $\varphi_{0}-\varphi_{10}$.
    \textbf{Top}: the number of enrolled robots with different status,
    including navigation, waiting,  and task execution;
    \textbf{Bottom}: number of remaining tasks that
    are assigned and unassigned.}}
  \label{fig:dual_stack}
  \vspace{-0.1in}
\end{figure}

Execution begins with $\mathcal{N}_3$ on $\omega_2$ and $\mathcal{N}_4$ on $\omega_5$.
Teams $\mathcal{N}_0$, $\mathcal{N}_1$, and $\mathcal{N}_2$ move to staging locations
because the temporal ordering requires $\omega_2$ before $\omega_4$, $\omega_5$ before $\omega_7$,
and $\omega_4$ before $\omega_8$. The ordering also allows $\omega_8$ and $\omega_9$ to proceed concurrently.
These temporal constraints promote purposeful staging and prevent premature, infeasible starts. At $t=15\,\mathrm{s}$,
teams $\mathcal{N}_3$, $\mathcal{N}_4$, and $\mathcal{N}_2$ complete $\omega_2$, $\omega_5$, and $\omega_4$,
respectively, for a total of~$47$ finished subtasks. Completion triggers replanning, which costs $0.41\,\mathrm{s}$.
{At this stage, request $\kappa_3$ \emph{for priority change} is issued by the operator to execute $\omega_8$
due to a \emph{deadline adjustment}.} Since $\mathcal{N}_0$ is executing $\omega_7$
at that time, the tasks $\omega_8$ and $\omega_9$ are revisited along with the remaining frontier tasks.
Team $\mathcal{N}_0$ continues $\omega_7$ without interruption, and new teams $\mathcal{N}_5$ through
$\mathcal{N}_9$ are formed to execute $\omega_6$, $\omega_3$, $\omega_1$, $\omega_8$, and $\omega_9$, respectively.
This confirms that the online loop preserves continuity for in-progress work while exploiting
newly freed resources.
{At~$t=24\,\mathrm{s}$, the operator issues
request~$\kappa_2$ \emph{for task cancellation}, terminating task~$\omega_6$ to
reflect the updated priorities.
Thus, the respective robots in $\mathcal{N}_5$ are freed up resources to other tasks in the next replanning.}

At $t=30\,\mathrm{s}$, a new mission with~$5$ tasks and~$7$ relations is released.
The global layer replans in $0.44\,\mathrm{s}$, produces~$5$ subteams, and updates
the predicted makespan to $51.3\,\mathrm{s}$.
{
Note that at~$t=110\,\mathrm{s}$, requests~$\kappa_1$ are released to add task~$\omega_{31}$ and~$ \omega_{32}$.
Other missions are released at $t=60\,\mathrm{s}, 100\,\mathrm{s}, 130\,\mathrm{s}$
increase the overall workload to~$32$ tasks and~$482$ subtasks.
Specially, conflicting requests are issued by operator at~$t=140\,\mathrm{s}$,
where $\kappa_3$ requests to raise the priority of task~$\omega_{30}$ execute,
and $\kappa_4$ requests to assign~$10$ Type-C robots to execute~$\omega_{27}$.
Namely, the available Type-C robots can not meet the capacity requirements of task~$\omega_{30}$.
Thus, the system shows the conflict to operator and asks which command to execute.
The operator chooses $\kappa_3$, raising the priority of task~$\omega_{30}$.
In total, the full mission set finishes at $180.4\,\mathrm{s}$ over~$12$ replanning events. The mean task response
time is $37.0\,\mathrm{s}$, and the mean and max replanning time for~$12$ events are $0.36\,\mathrm{s}$, $0.87\,\mathrm{s}$}.
The requests continue to be processed, ensuring that dynamic adjustments
are smoothly integrated into the execution timeline.

\begin{figure}[t!]
  \centering
  \includegraphics[width=0.97\hsize]{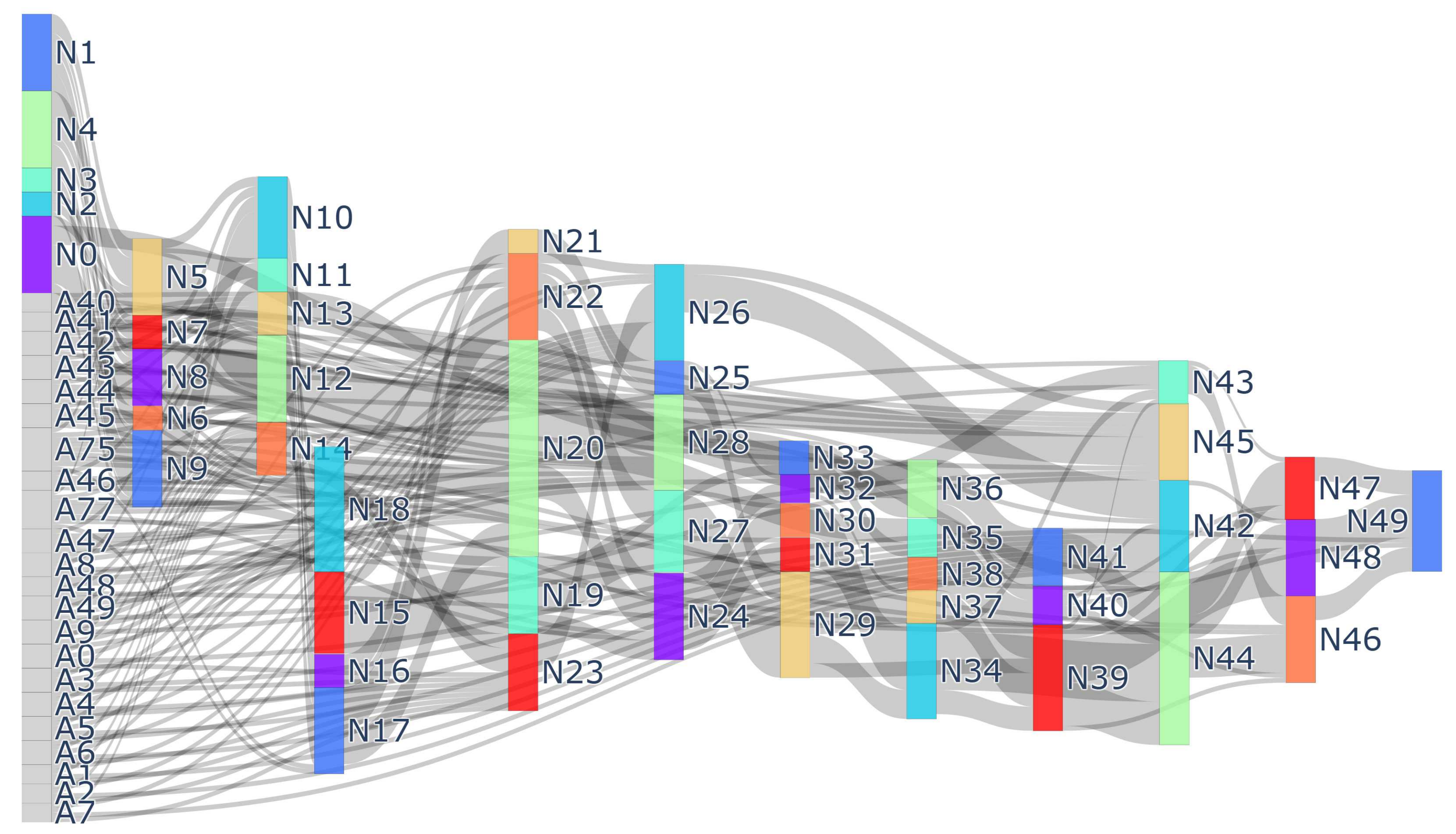}
  \vspace{-0.1in}
  \caption{{Sankey diagram of all robots that participated in the overall mission,
    along with the number of teams and their compositions.
  Note that subteams are labeled $N_0-N_{49}$ and robots are labeled $A_{0}-A_{79}$.}}
  \label{fig:sankey}
  \vspace{-0.1in}
\end{figure}

\subsubsection{Local task execution}
Local execution follows the three strategy classes in Sec.~\ref{subsec:act} and
respects the curvature feasibility shown in Fig.~\ref{fig:subteam_sim}. All robot
trajectories have curvature below $3\,\mathrm{m^{-1}}$. At $t=22\,\mathrm{s}$, the
delivery task $\omega_3$ includes $18$ static and known subtasks. Team
$\mathcal{N}_6$ is assigned $8$ robots using~\eqref{eq:augmented-cost}; each robot
completes $2$ or $3$ subtasks under a fixed robot-to-subtask assignment, yielding
short traversal and limited coalition changes.
{During this phase, the operator issues request~$\kappa_4$
\emph{for robot reassignment}, transferring robot~21 from~$\omega_5$ to support
team $\mathcal{N}_6$ in completing subtasks~$5,9,16$ of $\omega_3$ under resource
reallocation. This intention enables rapid cross-task resource redistribution under
changing priorities.}

At $t=25\,\mathrm{s}$, the surveillance task $\omega_7$ requires exploration with
unknown subtasks. Team $\mathcal{N}_0$ follows region-covering trajectories, and
newly revealed subtasks are allocated using~\eqref{eq:min-disruption} to minimize
route disruption. For example, robot~43 completes subtask~10 and robot~4 completes
subtask~12 before resuming exploration.
{At this point, the operator issues request~$\kappa_4$
\emph{for robot reassignment}, assigning robot~42 to assist team $\mathcal{N}_0$
in completing subtask~9 of $\omega_7$. This timely intervention reallocates available
capacity to the newly revealed workload, improving execution efficiency.}
 In total, $16$ subtasks are discovered and completed, validating that
exploration and execution should interleave to reduce idle time.

At $t=13\,\mathrm{s}$, the capture task $\omega_5$ comprises $18$ moving subtasks
and is executed by team $\mathcal{N}_4$ with $22$ robots. Alg.~\ref{alg:DynamicKnown}
updates local coalitions online as targets move; coalitions dissolve upon subtask
completion and robots are reassigned. For instance, robots~60 and~21 complete
subtask~0 jointly; robot~60 then completes subtask~10 alone, while robot~21
completes subtask~1 alone. Near completion, robots consolidate into two coalitions
to finish subtasks~5 and~8, indicating that additional robots mitigate motion
variability and accelerate completion.
{Across task families, mean response times
are $29.39\,\mathrm{s}$ for delivery, $46.37\,\mathrm{s}$ for surveillance, and
$35.65\,\mathrm{s}$ for dynamic capture, with local planning times of
$0.31\,\mathrm{s}$, $0.22\,\mathrm{s}$, and $0.50\,\mathrm{s}$}.

\subsubsection{Overall Resources Utilization and Adaptation}
As can be seen in Fig.~\ref{fig:dual_stack},
the status of task execution takes the largest proportion of robot states for whole progress.
The proportion of navigation states is acceptable since the whole scene is in shape of rectangular.
Waiting states are kept low due to the scheduling strategies.
The unassigned tasks in the lower panel spike to~$5$ after each release and fall
to~$0$ within~$15\,\mathrm{s}$, indicating fast adaptation to new tasks.
Moreover, the proposed scheme continually reshapes the composition of
each team as the mission evolves, as illustrated in Fig.~\ref{fig:sankey}.
{In total~$49$ subteams are formed and robots are reused across teams
such as robots in~$N_{43}$ later are separated out to form new teams~$N_{46}$ and~$N_{47}$
together with other robots.}

\subsection{Comparisons}\label{subsec:comparisons}
{The proposed method is compared against \textbf{eight} baselines}:
(I) \textbf{MILP},
where a complete MILP is formulated for all robots $\mathcal{N}$ and tasks
$\overline{\Omega}_t$, similar to~\cite{torreno2017cooperative,luo2021temporal},
i.e., without the subteam formation;
(II) \textbf{SAMP-Task}, where a sampling-based planner from~\cite{kantaros2020stylus}
is adopted for all robots and tasks;
(III) \textbf{SAMP-Subtask}, which applies the sampling-based planner directly to subtasks;
{(IV) \textbf{ScRATCHeS}, which formulates a complete MILP for all robots
  $\mathcal{N}$ and tasks $\overline{\Omega}_t$, with capacity-based temporal logic formulation~\cite{leahy2021scalable};
(V) \textbf{Hulk}, which follows poset abstraction of mission automata
  and task-graph constrained receding-horizon assignment~\cite{luo2025hulk};
(VI) \textbf{Flow}, which models precedence-aware coalition task allocation
  and solves it via network-flow approximations with online re-allocation~\cite{gosrich2025online};}
(VII) \textbf{Inf-H}, which is the same
as our method but with an infinite horizon $H$, i.e., all known tasks are assigned
in Alg.~\ref{alg:ss-TaskAssign};
(VIII) \textbf{Greedy}, which assigns a maximum of one task to each subteam, i.e., without the horizon $H$.
{The first six baselines are established methods, while the last two are ablation studies.}
Note that the replanning conditions for all baselines are identical to the proposed method.
The compared metrics include the maximum response time for
missions, {average response time of missions}, the average and maximum planning time,
the average number of robots performing navigation, waiting for collaboration, and executing tasks,
and the success rate.

\begin{table}[t!]
\centering
\caption{{Comparison with Baselines $(N{=}80, M{=}30, J{=450})$}}
\label{table:comparasion}
\vspace{-0.05in}
\setlength{\tabcolsep}{1.5pt}
\renewcommand{\arraystretch}{1.3}
\resizebox{\columnwidth}{!}{
\begin{tabular}{l|ccccc}
\toprule
\scriptsize Method & \scriptsize Resp. Time [s] & \scriptsize {Ave. Resp. [s]}& \scriptsize Ave/Max Plan [s] & \scriptsize N/W/E Robots & \scriptsize Succ. Rate [\%] \\
\midrule
{\textbf{Ours}} & \small  {{184.4}} & \small {\textbf{34.6}} & \small {0.46}/0.76 & \small  {\textbf{19/3/27}} & \small {\textbf{100}} \\
MILP & \small 327.8 & \small {99.1} & \small 90.3/144.0 & \small 23/7/18 & \small 100 \\
SAMP-Task & \small {286.0} & \small {72.6} & \small {0.45/0.74} & \small {24/7/15} & \small 100 \\
SAMP-Subtask & \small 230.1 & \small {45.6} & \small 3.8/8.0 & \small 16/5/30 & \small 86 \\
{ScRATCHeS} & \small {257.4} & \small {64.9} &\small {1.25/7.60} & \small {13/3/19} & \small {100} \\
{Hulk(+poset)} & \small {194.4} & \small {40.5} & \small {4.66/7.59} & \small {24/8/24} & \small {100} \\
{Flow} & \small {183.5} & \small {41.8} & \small {1.79/5.98} & \small {31/1/34} & \small {96} \\
{Inf-H} & \small  {\textbf{182.1}} & \small {{77.1}}& \small  {53.4/415.8} & \small  {23/3/24} & \small  {100} \\
Greedy & \small 325.7 & \small {115.6} & \small \textbf{0.31/0.56} & \small 31/7/30 & \small 100 \\

\bottomrule
\end{tabular}
}
\vspace{-3mm}
\end{table}

\begin{figure}[t!]
    \centering
    \includegraphics[width=0.97\hsize]{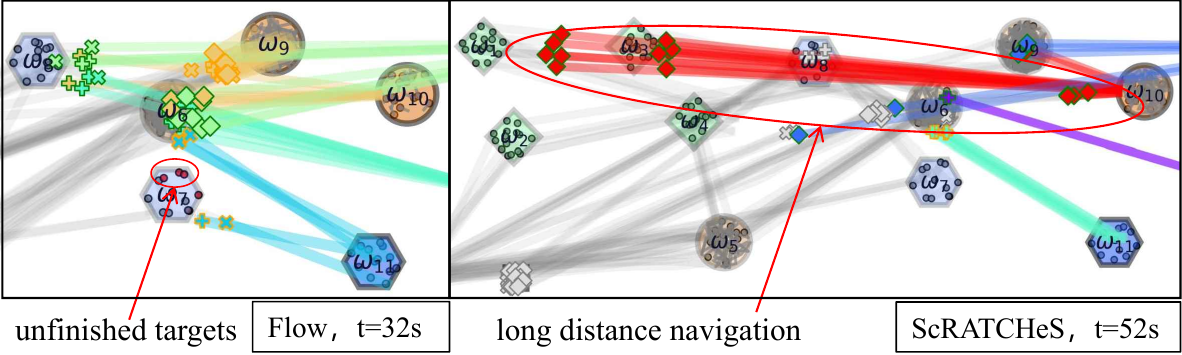}
    \vspace{-0.1in}
    \caption{        {Snapshots of baselines Flow and ScRATCHeS.
        \textbf{Left}: the subtasks in the red circle are unfinished after $t=32\mathrm{s}$;
        \textbf{Right}: robots marked in red have long navigation distance due to inefficient assignments.
        }}\label{fig:baselinefig}
    \vspace{-0.1in}
  \end{figure}

As summarized in Table~\ref{table:comparasion}, the proposed method outperforms all other methods
across most metrics, including response time, planning time, and robot scheduling.
{
\textbf{Efficiency and Response:} {HECTOR} achieves a total response time of $184.4\,\mathrm{s}$,
which is significantly lower than most methods like {MILP} at $327.8\,\mathrm{s}$ and {SAMP-Task} at $286.0\,\mathrm{s}$,
largely owing to the establishment for subteam.
While the {Inf-H} method exhibits a comparable total response time,
its maximum planning time can be prohibitively long at $415.8\,\mathrm{s}$,
demonstrating the inefficiency of an infinite horizon.
The {Flow} based method has a execution success rate of $96\%$
though also exhibiting a comparable response time.
Meanwhile, {HECTOR} maintains a highly stable and efficient planning,
with the average and max planning time of $0.46/0.76\,\mathrm{s}$.
While some methods exhibit comparable computational efficiency in terms of planning time
with {SAMP-Task} at $0.45/0.74\,\mathrm{s}$
and {Greedy} at $0.31/0.56\,\mathrm{s}$,
they suffer from inefficient scheduling of robots or prolonged response times.
\textbf{Reliability and Success Rate:} {HECTOR} and {MILP} both achieve a $100\%$ success rate,
but HECTOR does so with significantly less computational overhead.
As shown in Fig.~\ref{fig:baselinefig}, while {ScRATCHeS} exhibits excellent robot scheduling ($13/3/19$),
it falls short in terms of response and planning time.
In contrast, {SAMP-Subtask} and {Flow} exhibit lower success rates at $86\%$ and $96\%$,
as they do not account for uncertainties in subtasks.
Although {Greedy} remains fast,
it results in inefficient robot scheduling ($31/7/30$) and a long response time at $325.7\mathrm{s}$.
\textbf{Robustness in Complex Coordination:}
While {SAMP-Subtask} struggles with the increased search space of $450$ subtasks,
our hierarchical structure effectively decouples the task assignment from intra-team coordination.
Though {Hulk} shows a competitive response time at $194.4\,\mathrm{s}$,
HECTOR avoids the mission poset establishment for complex temporal tasks,
leading to a $10$ times reduction in planning time.}

\begin{table}[t!]
\centering
\caption{{Scalability and Robustness Analysis}}
\label{table:scalability}
\vspace{-0.05in}
\setlength{\tabcolsep}{1.2pt}
\renewcommand{\arraystretch}{1.3}
\resizebox{\columnwidth}{!}{
\begin{threeparttable}
\begin{tabular}{c|c|>{\small}c>{\small}c>{\small}c>{\small}c}
\toprule
\scriptsize $\mathbf{(N, M, J)}$ & \scriptsize Failure $\rho$ & \scriptsize Resp. Time [s] & \scriptsize Ave/Max Plan [s] & \scriptsize N/W/E Robots & \scriptsize Succ. Rate [\%] \\
\midrule
\multirow{2}{*}{\textbf{(120, 50, 750)}} & 0.05 & 378.2 & 0.46/0.82 & 26/5/27 & 100 \\
 & 0.10 & 346.6 & 0.51/0.79 & 29/5/25 & 100 \\
\midrule
\multirow{2}{*}{\textbf{(150, 80, 1202)}} & 0.05 & 699.1 & 0.53/0.88 & 29/3/31 & 100 \\
& 0.10 & 788.3 & 0.55/0.90 & 26/7/36 & 100 \\
\midrule
\multirow{2}{*}{{\textbf{(170, 100, 1509)}}} & {0.05} & {923.1}
& {0.65/1.48} & {30/6/30} & {100} \\
 & {0.10} & {1120} & {0.87/1.65} & {31/3/29} & {100} \\
\bottomrule
\end{tabular}
\end{threeparttable}
}
\vspace{-0.1in}
\end{table}

\begin{figure}[t!]
    \centering
    \includegraphics[width=0.97\hsize]{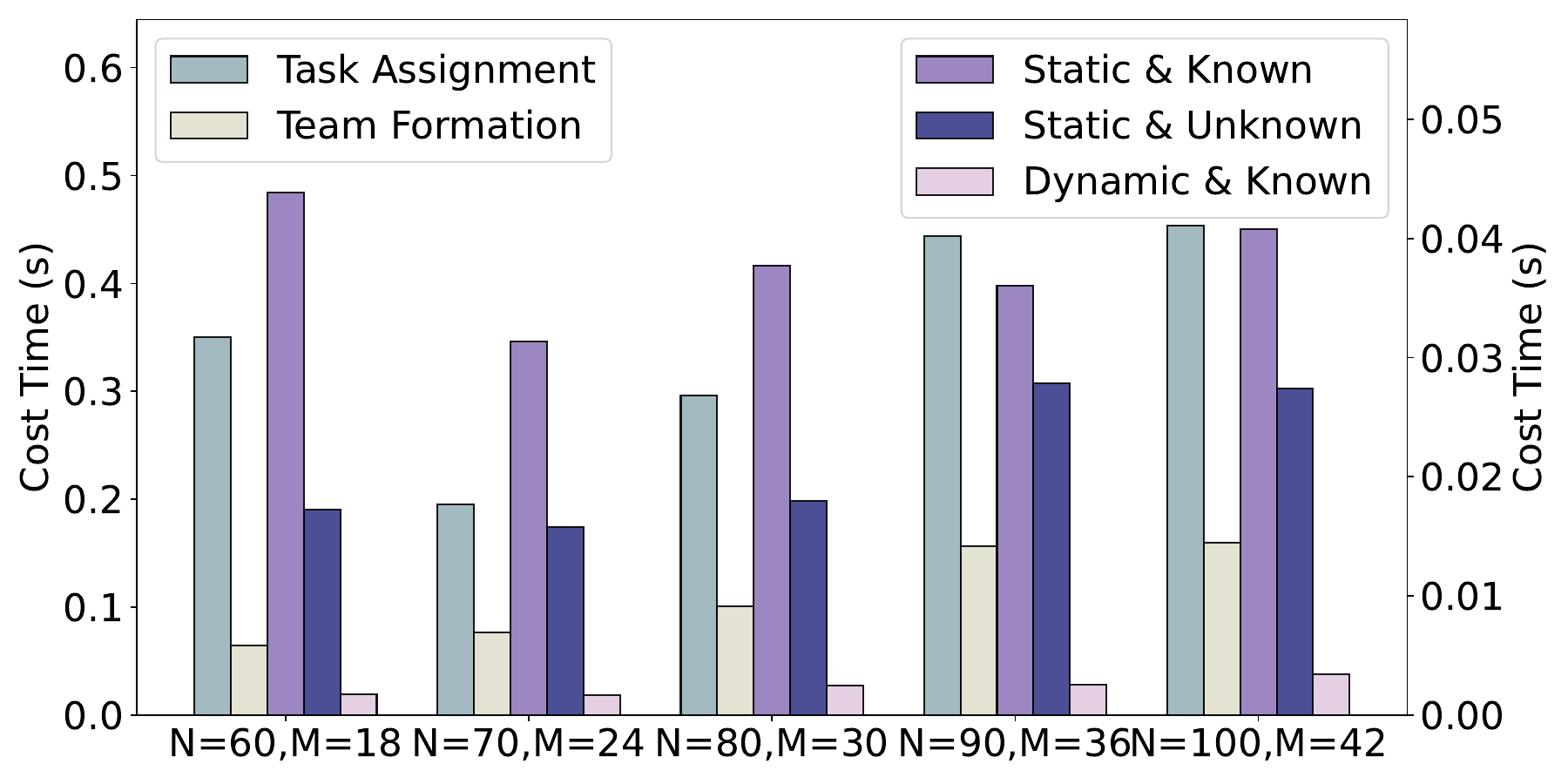}
    \vspace{-0.1in}
    \caption{      {Computational time analysis across
        the \textbf{five core modules} with varying numbers of robots~$N$ and tasks~$M$.
        }}\label{fig:modules_complexity}
    \vspace{-0.1in}
  \end{figure}

\subsection{Generalization}\label{subsec:scalability}

{For further validation,
the scalability and robustness of our method are evaluated
by increasing the fleet size and introducing robot failures with a probability $\rho$.
The scenarios are scaled from a fleet of $120$ robots with $50$ tasks (released within $240\,\mathrm{s}$),
to $150$ robots and $80$ tasks (released within $400\,\mathrm{s}$),
to $170$ robots and $100$ tasks (released within $480\,\mathrm{s}$).}

(I) \emph{Scalability}. As summarized in Table~\ref{table:scalability},
regardless of the failure rate $\rho$ being $0.05$ or $0.1$,
{the average and maximum planning time remain consistently below $2.0\,\mathrm{s}$,
even as the problem scale grows to $170$ robots and $100$ tasks.
This efficiency is primarily attributed to our receding-horizon planning with $H=6$, which prevents computational explosion with increasing fleet size}.
When the fleet size and the number of tasks increase to $120$ and $50$,
the response time increases by $40\%$ to $57\%$ relative to the total horizon of $240$ seconds.
Similarly, with a fleet size of $150$ and $80$ tasks, the response time increases
by $75\%$ to $97\%$ relative to the total horizon of $400$ seconds,
{while the response time increases by $92\%$ to $133\%$ relative to the horizon of $480$ seconds},
clearly indicating that the response time scales with the fleet and task sizes,
{and demonstrating the long-range scheduling of robots in large-scale scenarios}. Furthermore,
the average number of deployed robots increases as the fleet size and task complexity grow.

(II) \emph{Failure Recovery}. {Even with failure probabilities of $\rho=0.05$
and $0.1$, the success rate remains at $100\%$ for fleets of $170$ robots
  with $100$ tasks and $1509$ subtasks.
To recover from these failures, more robots are recruited
and the average response time is further increased.
The average planing time is only increased slightly with higher failure rate.}
This shows that the fleet capacities are sufficient to meet
the task requirements, even under more challenging conditions.

{(III) \emph{Runtime Decomposition}. Fig.~\ref{fig:modules_complexity} reports the runtime of each main module in
the proposed framework, including task assignment, team formation, and local
coordination, under different fleet sizes
$|\mathcal{N}|$ and the numbers of tasks $|\mathcal{T}_{\mathrm{act}}|$.
It can be seen that task assignment consistently requires more computational time than team formation,
with times ranging from approximately $0.35 \mathrm{s}$ to a maximum of $0.55 \mathrm{s}$ as both~$N$ and $M$ increase.
In contrast, team formation remains relatively stable, typically between $0.05 \mathrm{s}$ and a maximum of $0.1 \mathrm{s}$.
The dynamic and known scenario exhibits the lowest computational times for local coordination strategies,
while the static and known case takes longest time at around $0.4 \mathrm{s}$.}

\begin{figure}[t!]
    \centering
    \includegraphics[width=0.97\hsize]{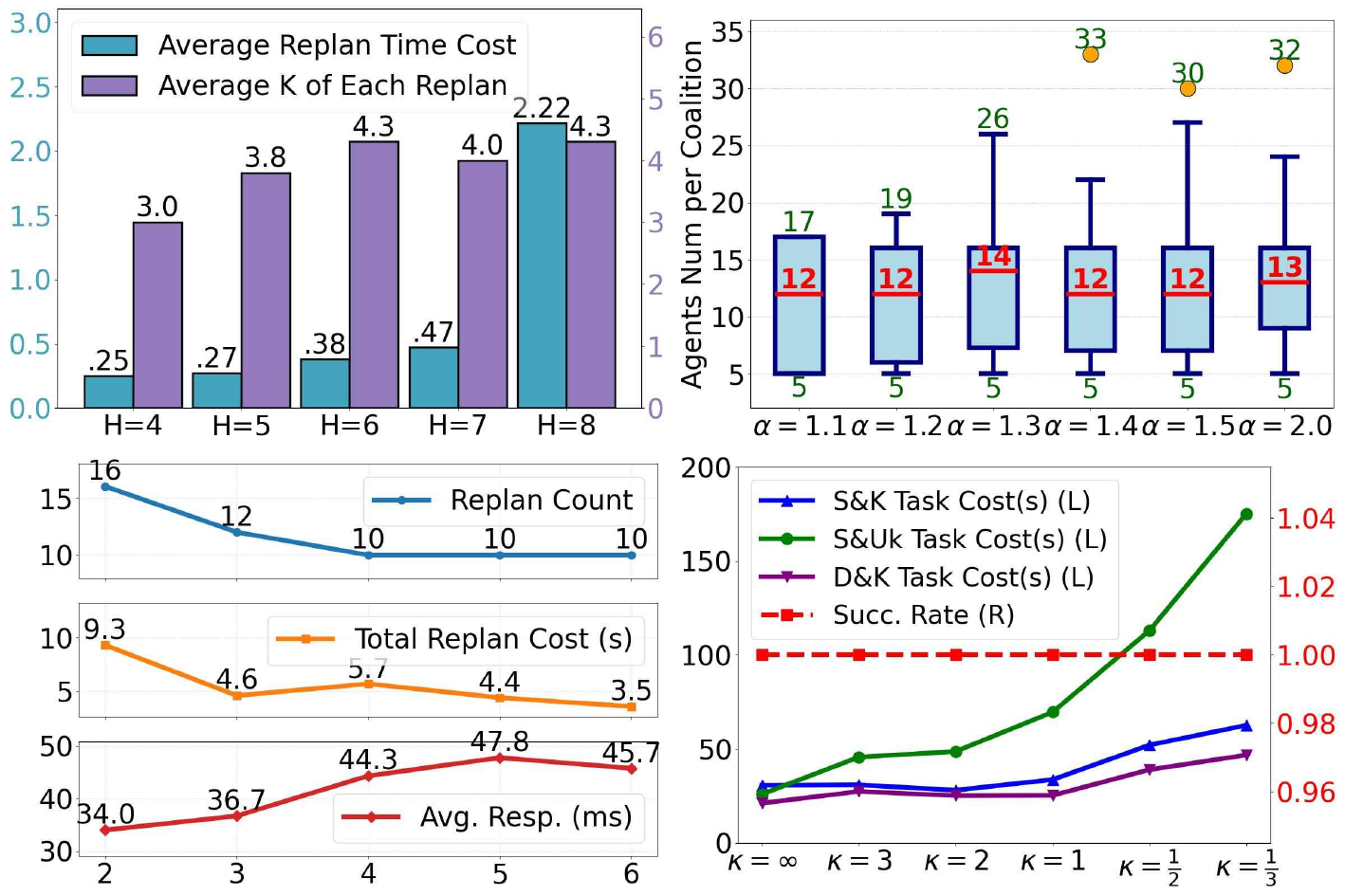}
    \vspace{-0.1in}
    \caption{        {Parameter sensitivity and performance analysis of the HECTOR framework:
parameter horizon~$H$ in assignment (\textbf{top-left}),
redundancy margin~$\alpha$ in team formation (\textbf{top-right}),
the number of finished task in current horizon to trigger the replan (\textbf{bottom-left}),
and the curvature~$\kappa$ of robots to plan the trajectories in task coordination (\textbf{bottom-right}).
        }}\label{fig:analysis}
    \vspace{-0.1in}
  \end{figure}

{(IV) \emph{Sensitivity to Key Parameters}.
As shown in Fig.\ref{fig:analysis},
the sensitivity of the proposed scheme with respect to several
key parameters are analyzed.
In particular, the analysis of the rolling horizon~$H$ reveals that the average replanning time
increases from $0.25 \mathrm{s}$ at $H=4$ to $2.22 \mathrm{s}$ at $H=8$. The average
number of teams during each replan is around~$3.9$,
which remains close across different horizons.
Moreover, the replan count decreases to $10$
as the triggering condition is increased from $2$ to $6$.
The total replanning time decreases from $9.3 \mathrm{s}$ to $3.5 \mathrm{s}$,
while the response time is increased from $34.0 \mathrm{ms}$ and $47.8 \mathrm{ms}$.
Regarding the redundancy margin~$\alpha$, the agents per coalition increase with $\alpha$,
reaching $32$ agents at $\alpha=2.0$, compared to $17$ agents at $\alpha=1.1$.
However, the success rate remains $100\%$ across all choice of $\alpha$.
Lastly, as the curvature constraints~$\kappa$ are tightened, the execution costs of local tasks are
increased due to tighter dynamic constraints,
especially for static and unknown tasks from $25 \mathrm{s}$ to $180 \mathrm{s}$.
The success rate remains $100\%$ across all curvatures and tasks,
demonstrating its versatility, even for non-holonomic teams.}

\begin{figure}[t!]
  \centering
  \includegraphics[width=0.97\hsize]{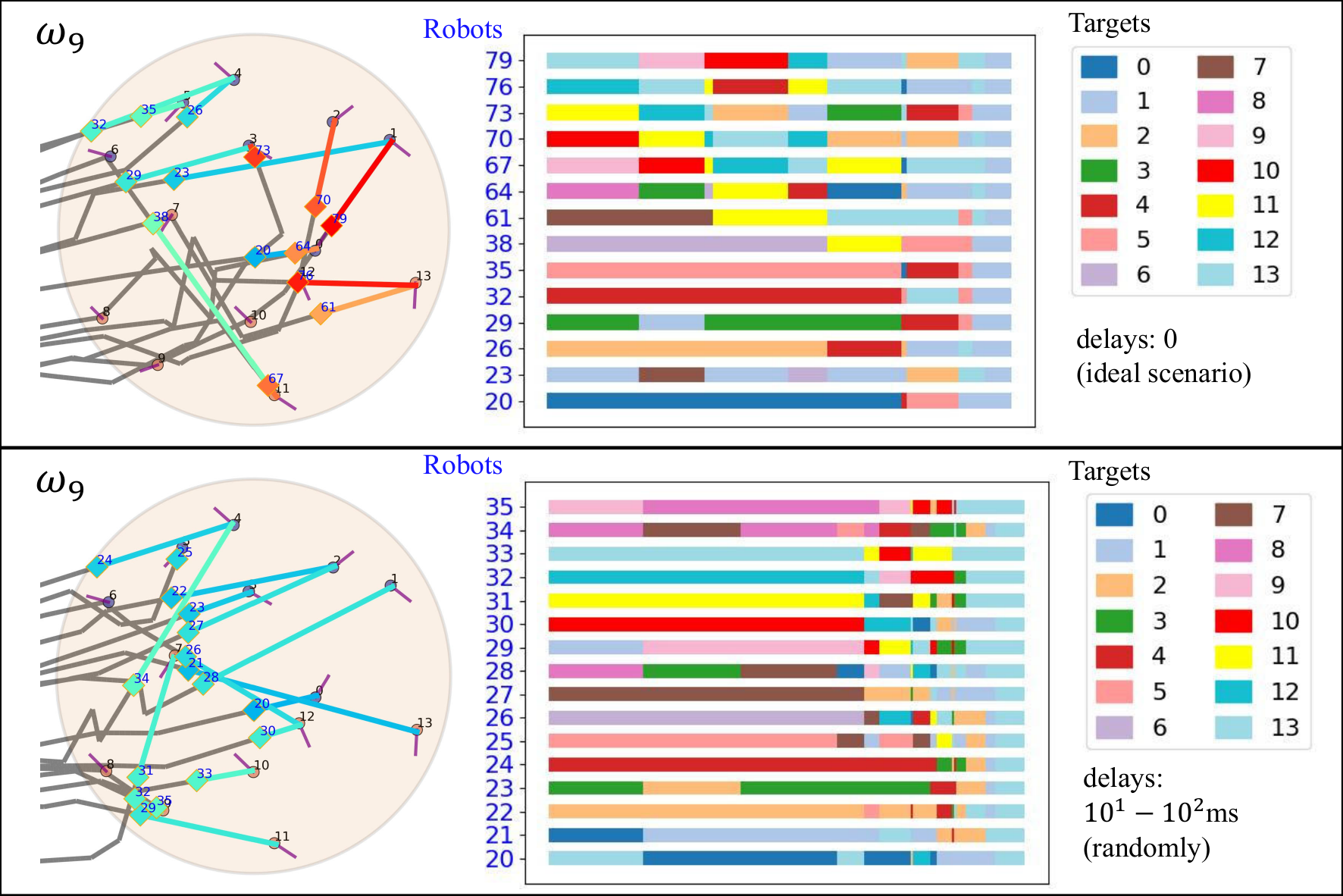}
  \vspace{-0.1in}
  \caption{{Illustration of the distributed dynamic coalition formation
 under different communication delays.
 \textbf{Top}: the ideal scenario without communication delay; \textbf{Bottom}:
 scenario with stochastic delays ranging from~$10^1$ to~$10^2\,\mathrm{ms}$.}}
  \label{fig:delayfig}
  \vspace{-0.1in}
\end{figure}

{(V) \emph{Communication Constraints}.
  The performance of the distributed dynamic coalition formation under varying
communication latencies is evaluated, as shown in Fig.~\ref{fig:delayfig}.
It can be seen that when the communication latency is increased to a stochastic delay
ranging from $10^1$ to $10^2\,\mathrm{ms}$,
the proposed algorithm can still ensuring successful task completion.
Notably, the resulting robot formation, local plans, and trajectories
are drastically different.
For instance, robots~$21$ and~$20$ are assigned to the subtask~$0$,
while subtask~$7$ is finished by the coalition formed by robots~$34$ and $25$
instead of robots~$23$ and~$61$.
The overall task execution time is increased slightly from~$31.0\mathrm{s}$ to ~$41.2\mathrm{s}$.}

\begin{figure*}[t!]
    \centering
    \includegraphics[width=0.9\hsize]{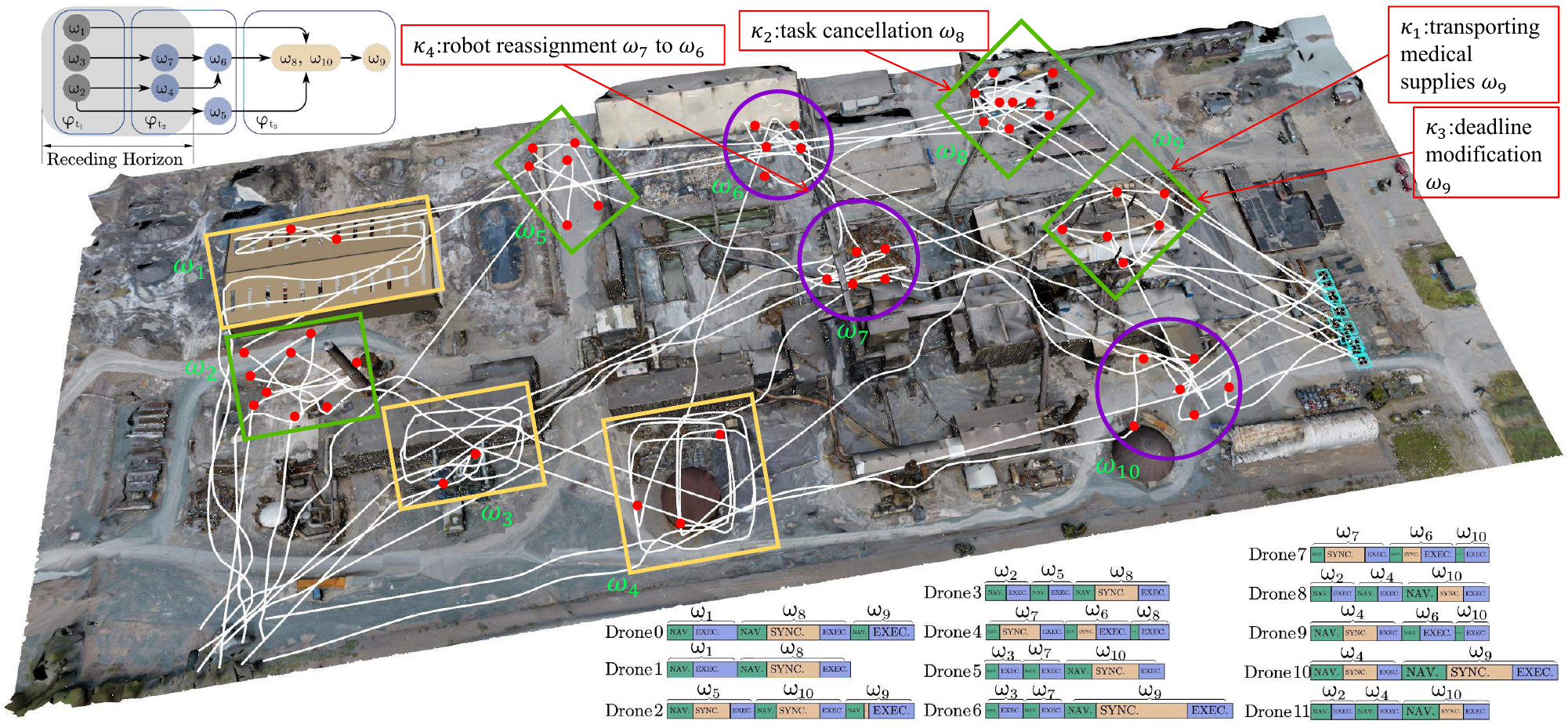}
    \vspace{-0.1in}
    \caption{{High-fidelity ROS simulation of the proposed framework under the scene for disaster relief.
      \textbf{Top-left}: 2 UAVs assigned to static known tasks
      (green),~4 robots to dynamic known tasks (purple), and~6 robots to static unknown tasks (yellow);
      \textbf{Middle}: the task execution trajectories for 10 tasks over different regions
      and $36$ subtasks.
      \textbf{Bottom-right}: the execution timeline of each robot
      completing assigned tasks in 34.6$\,\mathrm{s}$ on average.}}
    \label{fig:ros_sim}
    \vspace{-0.1in}
  \end{figure*}

\subsection{High-fidelity ROS-Simulation}\label{subsec:ros_simulation}

\subsubsection{Mission and Workspace Setup}\label{subsubsec:mission_setup}

To validate the proposed framework under more realistic robot dynamics and environment
interactions, a ROS-based simulation is conducted for~12 UAVs performing~10 tasks as part of the
disaster relief mission, as shown in Fig.~\ref{fig:ros_sim}.
The tasks are released in three stages across~3 different missions. Two UAVs are
assigned to the static tasks, including the searching for the workshop and storage tank.
Four UAVs
are allocated to the dynamic tasks, such as transporting relief packages to moving workers.
The static tasks include delivering goods and seeding messages to fixed delivery locations.
{The request $\kappa_1$ \emph{for new task release}
is issued to add an additional task, $\omega_9$, for transporting medical supplies to a new location,
prompting a reallocation of resources.
When workers make empirical requests to the operator at unexpected locations,
the request $\kappa_4$ \emph{for robot reassignment} is utilized
to reassign UAV~$7$ from $\omega_7$ to assist team $\mathcal{N}_3$ with $\omega_6$.}
These interventions ensure the system remains flexible and responsive
to dynamic missions.

\subsubsection{Simulation Results and System Performance}\label{subsubsec:simulation_results}

As summarized in Fig.~\ref{fig:ros_sim},
the simulation results highlight the effectiveness of the hierarchical coordination framework,
with all~$10$ tasks and~$50$ subtasks successfully completed. The system efficiently handles dynamic task
reassignments, with~4 replanning events performed in less than~1$\,\mathrm{s}$ each. The average response
time for tasks is~34.6$\,\mathrm{s}$, and the mean replanning time for~10 events is~0.46$\,\mathrm{s}$.
For static tasks, the average response time is~28.9$\,\mathrm{s}$, while dynamic capture tasks have a
response time of~29$\,\mathrm{s}$. The multimodal human-fleet interaction protocol allows real-time
operator inputs through the GUI, enabling dynamic task modifications.
{Leveraging external situational awareness,
the operator identifies a critical conflict where a routine supply delivery obstructed urgent rescue operations,
 request~$\kappa_2$ \emph{for task cancellation}
is triggered to cancel~$\omega_8$ due to a priority shift, freeing resources for higher-priority tasks.
Additionally, the request~$\kappa_3$ \emph{for deadline modification} is used to extend deadlines for $\omega_9$ as delays
occur due to unforeseen human-induced circumstances. This flexibility ensures seamless task allocation and continuous
adaptation to changing task conditions, maintaining smooth execution throughout the mission.}

\section{Conclusion} \label{sec:conclusion}
This work presents {HECTOR}, a hierarchical planning and coordination framework that couples global
mission assignment with local subtask and trajectory coordination.
It embeds practical human–fleet protocols and a graphic interface for multimodal
and online interactions under complex temporal tasks.
{Future work includes conflicting requests, adversarial settings via game-theoretic
reasoning, and tighter integration of communication constraints.}

\bibliographystyle{IEEEtran}
\bibliography{contents/references}

@article{schillinger2018simultaneous,
  title={Simultaneous task allocation and planning for temporal logic goals in heterogeneous multi-robot systems},
  author={Schillinger, Philipp and B{\"u}rger, Mathias and Dimarogonas, Dimos V},
  journal={The International Journal of Robotics Research},
  volume={37},
  number={7},
  pages={818--838},
  year={2018},
}

@article{kantaros2020stylus,
  title={Stylus*: A temporal logic optimal control synthesis algorithm for large-scale multi-robot systems},
  author={Kantaros, Yiannis and Zavlanos, Michael M},
  journal={The International Journal of Robotics Research},
  volume={39},
  number={7},
  pages={812--836},
  year={2020},
  publisher={SAGE Publications Sage UK: London, England}
}

@article{luo2021temporal,
  title={Temporal logic task allocation in heterogeneous multirobot systems},
  author={Luo, Xusheng and Zavlanos, Michael M},
  journal={IEEE Transactions on Robotics},
  volume={38},
  number={6},
  pages={3602--3621},
  year={2022},
}

@book{baier2008principles,
  title={Principles of model checking},
  author={Baier, Christel and Katoen, Joost-Pieter},
  year={2008},
  publisher={MIT press}
}

@article{sahin2019multirobot,
  title   = {Multirobot Coordination With Counting Temporal Logics},
  author  = {Sahin, Yunus Emre and Nilsson, Petter and Ozay, Necmiye},
  journal = {IEEE Transactions on Robotics},
  volume  = {36},
  number  = {4},
  pages   = {1189--1206},
  year    = {2020},
  doi     = {10.1109/TRO.2019.2957669}
}

@article{leahy2021scalable,
  title   = {Scalable and Robust Algorithms for Task-Based Coordination From High-Level Specifications (ScRATCHeS)},
  author  = {Leahy, Kevin and Serlin, Zachary and Vasile, Cristian-Ioan and Schoer, Andrew and Jones, Austin M. and Tron, Roberto and Belta, Calin},
  journal = {IEEE Transactions on Robotics},
  volume  = {38},
  number  = {4},
  pages   = {2516--2535},
  year    = {2022},
  doi     = {10.1109/TRO.2021.3130794}
}

@inproceedings{varava2017herding,
  title={Herding by Caging: a Topological Approach towards Guiding Moving Agents via Mobile Robots.},
  author={Varava, Anastasiia and Hang, Kaiyu and Kragic, Danica and Pokorny, Florian T},
  booktitle={Robotics: Science and Systems},
  pages={696--700},
  year={2017}
}

@article{ulusoy2013optimality,
  title={Optimality and robustness in multi-robot path planning with temporal logic constraints},
  author={Ulusoy, Alphan and Smith, Stephen L and Ding, Xu Chu and Belta, Calin and Rus, Daniela},
  journal={The International Journal of Robotics Research},
  volume={32},
  number={8},
  pages={889--911},
  year={2013},
  publisher={SAGE Publications Sage UK: London, England}
}

@article{torreno2017cooperative,
  title={Cooperative multi-agent planning: A survey},
  author={Torre{\~n}o, Alejandro and Onaindia, Eva and Komenda, Anton{\'\i}n and {\v{S}}tolba, Michal},
  journal={ACM Computing Surveys (CSUR)},
  volume={50},
  number={6},
  pages={1--32},
  year={2017},
  publisher={ACM New York, NY, USA}
}

@inproceedings{gini2017multi,
  title={Multi-robot allocation of tasks with temporal and ordering constraints},
  author={Gini, Maria},
  booktitle={AAAI Conference on Artificial Intelligence},
  year={2017}
}

@article{apt2009generic,
  title={A generic approach to coalition formation},
  author={Apt, Krzysztof R and Witzel, Andreas},
  journal={International game theory review},
  volume={11},
  number={03},
  pages={347--367},
  year={2009},
  publisher={World Scientific}
}

@misc{glop,
  title        = {{OR-Tools} Linear Optimization Solver},
  author       = {{Google OR-Tools}},
  howpublished = {\url{https://developers.google.com/optimization/lp}},
  note         = {Accessed: 2026-05-14}
}

@article{liu2024time,
  title={Time minimization and online synchronization for multi-agent systems under collaborative temporal logic tasks},
  author={Liu, Zesen and Guo, Meng and Li, Zhongkui},
  journal={Automatica},
  volume={159},
  pages={111377},
  year={2024},
  publisher={Elsevier}
}

@inproceedings{holz2010evaluating,
  title     = {Evaluating the Efficiency of Frontier-Based Exploration Strategies},
  author    = {Holz, Dirk and Basilico, Nicola and Amigoni, Francesco and Behnke, Sven},
  booktitle = {Proceedings of ISR/ROBOTIK 2010},
  pages     = {36--43},
  year      = {2010},
  publisher = {VDE Verlag}
}

@article{tuci2018cooperative_transport,
  title   = {Cooperative Object Transport in Multi-Robot Systems: A Review of the State-of-the-Art},
  author  = {Tuci, Elio and Alkilabi, Muhanad H. M. and Akanyeti, Otar},
  journal = {Frontiers in Robotics and AI},
  volume  = {5},
  pages   = {59},
  year    = {2018},
  doi     = {10.3389/frobt.2018.00059}
}

@article{smith2009dynamic,
  title   = {Monotonic Target Assignment for Robotic Networks},
  author  = {Smith, Stephen L. and Bullo, Francesco},
  journal = {IEEE Transactions on Automatic Control},
  volume  = {54},
  number  = {9},
  pages   = {2042--2057},
  year    = {2009},
  doi     = {10.1109/TAC.2009.2026926}
}

@article{choudhury2017dynamics,
  title   = {Dynamic Multi-Robot Task Allocation under Uncertainty and Temporal Constraints},
  author  = {Choudhury, Shushman and Gupta, Jayesh K. and Kochenderfer, Mykel J. and Sadigh, Dorsa and Bohg, Jeannette},
  journal = {Autonomous Robots},
  volume  = {46},
  number  = {1},
  pages   = {231--247},
  year    = {2022},
  doi     = {10.1007/s10514-021-10022-9}
}

@article{you2021human,
  title   = {Computational Human--Robot Interaction},
  author  = {Thomaz, Andrea and Hoffman, Guy and Cakmak, Maya},
  journal = {Foundations and Trends in Robotics},
  volume  = {4},
  number  = {2--3},
  pages   = {105--223},
  year    = {2016},
  doi     = {10.1561/2300000049}
}

@article{ferreira2021survey,
  title   = {Human--Robot Interaction: A Survey},
  author  = {Goodrich, Michael A. and Schultz, Alan C.},
  journal = {Foundations and Trends in Human--Computer Interaction},
  volume  = {1},
  number  = {3},
  pages   = {203--275},
  year    = {2007},
  doi     = {10.1561/1100000005}
}

@article{krizmancic2022cooperative_ag,
  title   = {Cooperative Aerial--Ground Multi-Robot System for Automated Construction Tasks},
  author  = {Kri{\v{z}}man{\v{c}}i{\'c}, Marko and Arbanas, Barbara and Petrovi{\'c}, Tamara and Petri{\'c}, Frano and Bogdan, Stjepan},
  journal = {IEEE Robotics and Automation Letters},
  volume  = {5},
  number  = {2},
  pages   = {798--805},
  year    = {2020},
  doi     = {10.1109/LRA.2020.2965855}
}

@article{chen2021decentralized,
  title   = {Decentralized Task and Path Planning for Multi-Robot Systems},
  author  = {Chen, Yuxiao and Rosolia, Ugo and Ames, Aaron D.},
  journal = {IEEE Robotics and Automation Letters},
  volume  = {6},
  number  = {3},
  pages   = {4337--4344},
  year    = {2021},
  doi     = {10.1109/LRA.2021.3068103}
}

@article{dahiya2023survey,
  title   = {A Survey of Multi-Agent Human--Robot Interaction Systems},
  author  = {Dahiya, Abhinav and Aroyo, Alexander M. and Dautenhahn,
             Kerstin and Smith, Stephen L.},
  journal = {Robotics and Autonomous Systems},
  volume  = {161},
  pages   = {104335},
  year    = {2023},
  doi     = {10.1016/j.robot.2022.104335}
}

@inproceedings{ji2021multi,
  title     = {A Mixed-Integer Linear Programming Formulation for Human Multi-Robot Task Allocation},
  author    = {Lippi, Martina and Marino, Alessandro},
  booktitle = {IEEE International Conference on Robot
               and Human Interactive Communication},
  pages     = {1017--1023},
  year      = {2021},
  doi       = {10.1109/RO-MAN50785.2021.9515362}
}

@article{zhao2021market,
  title   = {Consensus-Based Decentralized Auctions for Robust Task Allocation},
  author  = {Choi, Han-Lim and Brunet, Luc and How, Jonathan P.},
  journal = {IEEE Transactions on Robotics},
  volume  = {25},
  number  = {4},
  pages   = {912--926},
  year    = {2009},
  doi     = {10.1109/TRO.2009.2022423}
}

@article{fioretto2018distributed,
  title   = {Distributed Constraint Optimization Problems and Applications: A Survey},
  author  = {Fioretto, Ferdinando and Pontelli, Enrico and Yeoh, William},
  journal = {Journal of Artificial Intelligence Research},
  volume  = {61},
  pages   = {623--698},
  year    = {2018},
  doi     = {10.1613/jair.5735}
}

@inproceedings{liu2019multi,
  title     = {An Evolutionary Traveling Salesman Approach for Multi-Robot Task Allocation},
  author    = {Arif, Muhammad Usman and Haider, Sajjad},
  booktitle = {International Conference on Agents and
               Artificial Intelligence},
  volume    = {2},
  pages     = {567--574},
  year      = {2017},
  doi       = {10.5220/0006197305670574}
}

@article{lahijanian2015temporal,
  title   = {Temporal Logic Motion Planning and Control With Probabilistic Satisfaction Guarantees},
  author  = {Lahijanian, Morteza and Andersson, Sean B. and Belta, Calin},
  journal = {IEEE Transactions on Robotics},
  volume  = {28},
  number  = {2},
  pages   = {396--409},
  year    = {2012},
  doi     = {10.1109/TRO.2011.2172150}
}

@article{martin2021mrta,
  title   = {Multi-robot task allocation problem with multiple nonlinear criteria using branch and bound and genetic algorithms},
  author  = {Martin, J. G. and Frejo, J. R. D. and García, R. A. and Camacho, E. F.},
  journal = {Intelligent Service Robotics},
  volume  = {14},
  number  = {3},
  pages   = {707--727},
  year    = {2021},
  doi     = {10.1007/s11370-021-00393-4}
}

@article{kolling2016human,
  title   = {Human Interaction With Robot Swarms: A Survey},
  author  = {Kolling, Andreas and Walker, Peter and Chakraborty, Nilanjan
             and Sycara, Katia and Lewis, Michael},
  journal = {IEEE Transactions on Human-Machine Systems},
  volume  = {46},
  number  = {1},
  pages   = {9--26},
  year    = {2016},
  doi     = {10.1109/THMS.2015.2480801}
}

@article{gombolay2017computational,
  title   = {Computational Design of Mixed-Initiative Human--Robot Teaming That Considers Human Factors: Situational Awareness, Workload, and Workflow Preferences},
  author  = {Gombolay, Matthew C. and Bair, Anna and Huang, Cindy and Shah, Julie A.},
  journal = {The International Journal of Robotics Research},
  volume  = {36},
  number  = {5--7},
  pages   = {597--617},
  year    = {2017},
  doi     = {10.1177/0278364916688255}
}

@inproceedings{swamy2020scaledautonomy,
  title     = {Scaled Autonomy: Enabling Human Operators to Control Robot Fleets},
  author    = {Swamy, Gokul and Reddy, Siddharth and Levine, Sergey and
               Dragan, Anca D.},
  booktitle = {IEEE International Conference on Robotics
               and Automation},
  pages     = {5942--5948},
  year      = {2020},
  doi       = {10.1109/ICRA40945.2020.9196792}
}

@inproceedings{hoque2022fleetdagger,
  title     = {Fleet-DAgger: Interactive Robot Fleet Learning with Scalable Human Supervision},
  author    = {Hoque, Ryan and Chen, Lawrence Yunliang and Sharma, Satvik and Dharmarajan, Karthik and Thananjeyan, Brijen and Abbeel, Pieter and Goldberg, Ken},
  booktitle = {Conference on Robot Learning (CoRL)},
  year      = {2022},
}

@article{erke2020improved,
  title={An improved A-Star based path planning algorithm for autonomous land vehicles},
  author={Erke, Shang and Bin, Dai and Yiming, Nie and Qi, Zhu and Liang, Xiao and Dawei, Zhao},
  journal={International Journal of Advanced Robotic Systems},
  volume={17},
  number={5},
  year={2020},
  publisher={SAGE Publications Sage UK: London, England}
}

@article{ny2011dubins,
  title   = {On the {Dubins} Traveling Salesman Problem},
  author  = {Le Ny, Jerome and Feron, Eric and Frazzoli, Emilio},
  journal = {IEEE Transactions on Automatic Control},
  volume  = {57},
  number  = {1},
  pages   = {265--270},
  year    = {2012},
  doi     = {10.1109/TAC.2011.2166311}
}

@inproceedings{vavna2015dubins,
  title={On the dubins traveling salesman problem with neighborhoods},
  author={V{\'a}{\v{n}}a, Petr and Faigl, Jan},
  booktitle={IEEE/RSJ International Conference on Intelligent Robots and Systems (IROS)},
  pages={4029--4034},
  year={2015},
}

@article{zhou2018resilient,
  title   = {Resilient Active Target Tracking With Multiple Robots},
  author  = {Zhou, Lifeng and Tzoumas, Vasileios and Pappas, George J. and Tokekar, Pratap},
  journal = {IEEE Robotics and Automation Letters},
  volume  = {4},
  number  = {1},
  pages   = {129--136},
  year    = {2019},
  doi     = {10.1109/LRA.2018.2881296}
}

@article{guerrero2012multi,
  title={Multi-robot coalition formation in real-time scenarios},
  author={Guerrero, Jose and Oliver, Gabriel},
  journal={Robotics and Autonomous Systems},
  volume={60},
  number={10},
  pages={1295--1307},
  year={2012},
  publisher={Elsevier}
}

@inproceedings{dai2024dynamic,
  title={Dynamic coalition formation and routing for multirobot task allocation via reinforcement learning},
  author={Dai, Weiheng and Bidwai, Aditya and Sartoretti, Guillaume},
  booktitle={IEEE International Conference on Robotics and Automation (ICRA)},
  pages={16567--16573},
  year={2024},
}

@article{kurtz2021more,
  title={A more scalable mixed-integer encoding for metric temporal logic},
  author={Kurtz, Vince and Lin, Hai},
  journal={IEEE Control Systems Letters},
  volume={6},
  pages={1718--1723},
  year={2021},
  publisher={IEEE}
}

@article{li2025task,
  title={Task Allocation of Heterogeneous Robots under Temporal Logic Specifications with Inter-Task Constraints and Variable Capabilities},
  author={Li, Lin and Chen, Ziyang and Wang, Hao and Kan, Zhen},
  journal={IEEE Transactions on Automation Science and Engineering},
  year={2025},
  publisher={IEEE}
}

@article{luo2024decomposition,
  title={Decomposition-based hierarchical task allocation and planning for multi-robots under hierarchical temporal logic specifications},
  author={Luo, Xusheng and Xu, Shaojun and Liu, Ruixuan and Liu, Changliu},
  journal={IEEE Robotics and Automation Letters},
  volume={9},
  number={8},
  pages={7182--7189},
  year={2024},
  publisher={IEEE}
}

@article{chen2025real,
  title={Real-time reactive task allocation and planning of large heterogeneous multi-robot systems with temporal logic specifications},
  author={Chen, Ziyang and Kan, Zhen},
  journal={The International Journal of Robotics Research},
  volume={44},
  number={4},
  pages={640--664},
  year={2025},
  publisher={SAGE Publications Sage UK: London, England}
}

@article{luo2025simultaneous,
  title={Simultaneous task allocation and planning for multi-robots under hierarchical temporal logic specifications},
  author={Luo, Xusheng and Liu, Changliu},
  journal={IEEE Transactions on Robotics},
  year={2025},
  publisher={IEEE}
}

@article{lindemann2019coupled,
  title   = {Coupled Multi-Robot Systems Under Linear Temporal Logic and Signal Temporal Logic Tasks},
  author  = {Lindemann, Lars and Nowak, Jakub and Sch{\"o}nb{\"a}chler, Lukas and Guo, Meng and Tumova, Jana and Dimarogonas, Dimos V.},
  journal = {IEEE Transactions on Control Systems Technology},
  volume  = {29},
  number  = {2},
  pages   = {858--865},
  year    = {2021},
  doi     = {10.1109/TCST.2019.2955628}
}

@article{cardona2024planning,
  title={Planning for heterogeneous teams of robots with temporal logic, capability, and resource constraints},
  author={Cardona, Gustavo A and Vasile, Cristian-Ioan},
  journal={The International Journal of Robotics Research},
  volume={43},
  number={13},
  pages={2089--2111},
  year={2024},
  publisher={SAGE Publications Sage UK: London, England}
}

@article{kantaros2022perception,
  title={Perception-based temporal logic planning in uncertain semantic maps},
  author={Kantaros, Yiannis and Kalluraya, Samarth and Jin, Qi and Pappas, George J},
  journal={IEEE Transactions on Robotics},
  volume={38},
  number={4},
  pages={2536--2556},
  year={2022},
  publisher={IEEE}
}

@inproceedings{luo2025hulk,
  author    = {Luo, Qingyuan and Li, Jie and Guo, Meng},
  title     = {HULK: Large-Scale Hierarchical Coordination Under Continual and
               Uncertain Temporal Tasks},
  booktitle = {IEEE International Conference on Robotics and
               Automation (ICRA)},
  year      = {2025}
}

@inproceedings{tumova2014maximally,
  author    = {Tumova, Jana and Marzinotto, Alejandro and Dimarogonas, Dimos V.
               and Kragic, Danica},
  title     = {Maximally Satisfying LTL Action Planning},
  booktitle = {IEEE/RSJ International Conference on Intelligent Robots and
               Systems (IROS)},
  year      = {2014},
  pages     = {1503--1510},
  publisher = {IEEE},
}

@article{gosrich2025online,
  title={Online multi-robot coordination and cooperation with task precedence relationships},
  author={Gosrich, Walker and Agarwal, Saurav and Garg, Kashish and Mayya, Siddharth and Malencia, Matthew and Yim, Mark and Kumar, Vijay},
  journal={IEEE Transactions on robotics},
  year={2025},
  publisher={IEEE}
}

\appendix

\begin{table}[t!]
\centering
{
\caption{Nomenclature of Key Variables and Definitions}
\label{table:nomenclature}
\vspace{-0.05in}
\renewcommand{\arraystretch}{1.4}
\resizebox{\columnwidth}{!}{
\begin{tabular}{lp{5cm}l}
\toprule
\scriptsize \textbf{Term} & \scriptsize \textbf{Definition} & \scriptsize \textbf{Reference} \\
\midrule
$\Phi_t$ & Set of missions known at time $t > 0$. & \smash{\mbox{Sec.~\ref{subsec:task}}} \\
$\tau_i$ & Local plan of robot~$i\in \mathcal{N}$. & \smash{\mbox{Sec.~\ref{subsec:multi-agent}}}\\
$\mathcal{K}=\{\kappa_{1,2,3,4}\}$ & Operator requests. & \smash{\mbox{Sec.~\ref{subsec:human}}}\\
$\mathcal{B}_{\varphi_m}$ & B\"{u}chi automaton for mission $\varphi_m$. & \smash{\mbox{Sec.~\ref{subsec:task}}} \\
$\widehat{Q}_m$ & Set of reachable states in automaton $\mathcal{B}_{\varphi_m}$. & \smash{\mbox{Sec.~\ref{subsec:task}}} \\
$\mathfrak{T}$ & Search tree structure defined as $(\mathcal{V}, \rightarrow)$. & \smash{\mbox{Sec.~\ref{subsec:task}}} \\
$\chi(\nu)$ & Value function for node selection. & \smash{\mbox{Eq.~\eqref{eq:value}}} \\
$\mathcal{C}_k$ & Capacity constraints for team~$k$. & \smash{\mbox{Sec.~\ref{subsec:task}}} \\
$\beta^j_k$ & Min. robots required to perform action $a^j$ for team $k$. & \smash{\mbox{Eq.~\eqref{eq:capacity-team}}} \\
$\alpha_j$ & Redundancy margin for workload uncertainty and failures. & \smash{\mbox{Sec.~\ref{subsec:task}}} \\
$\zeta(\nu)$ & Performance profile for node evaluation and pruning. & \smash{\mbox{Eq.~\eqref{eq:profile}}} \\
$\overline{\mathcal{V}}$ & Set of non-dominated frontier nodes in the search tree. & \smash{\mbox{Eq.~\eqref{eq:frontier}}} \\
$H$ & Planning horizon for task assignment. & \smash{\mbox{Sec.~\ref{subsec:task}}} \\
$b_{ik}$ & Robot $i$ assigned to team $\mathcal{N}_k$. & \smash{\mbox{Sec.~\ref{subsec:task}}} \\
$J(k)$ & Execution cost of team $k$. & \smash{\mbox{Sec.~\ref{subsec:task}}}  \\
$N_k$ & Robots in team $k$. & \smash{\mbox{Sec.~\ref{subsec:task}}}  \\
$\mathcal{J}^\ell_k$ & Set of local tasks for the $\ell$-th task of team~$k$. & \smash{\mbox{Sec.~\ref{subsec:act}}} \\
$\mathbf{x}_i$ & Trajectory of robot~$i\in \mathcal{N}$. & \smash{\mbox{Sec.~\ref{subsec:act}}}\\
\bottomrule
\end{tabular}
}}
\vspace{-3mm}
\end{table}

\subsection{{Proof of Lemmas and Theorems}}\label{app:proof}

\begin{proof}
{  \textbf{of Theorem~\ref{thm:correctness}}.
    Temporal correctness follows from the update of the reachable state sets
$\widehat{Q}_m$ along enabled automaton transitions
$q_m^{\ell+1} \in \delta^m(q_m^\ell,\omega^{\ell+1})$ and from the
completeness condition $\widehat{Q}_m \cap Q_F^m \neq \emptyset$ in
\eqref{eq:complete-nodes}. Any root to leaf path that ends at a node
$\nu \in \overline{\mathcal{V}}^\star$ therefore induces, for every
$m \in \mathcal{M}$, an accepting run of $\mathcal{B}_{\varphi_m}$, so
the joint plans $\{\Gamma_k\}$ satisfy all temporal constraints encoded
by $\{\mathcal{B}_{\varphi_m}\}$.
Moreover, the capacity feasibility and finite time convergence follow from local
pruning and finiteness of the search space. Each expansion step updates~$\mathcal{C}_k$
and enforces the fleet constraint
$\sum_{k\in\mathcal{K}}\beta_k^j \leq
\sum_{i\in\mathcal{N}}\mathds{1}(a^j\in\mathcal{A}_i)$, hence any node
in $\overline{\mathcal{V}}^\star$ respects all capacity bounds. The set
$\Phi_t$, the automaton state sets $Q^m$, the alphabets $\Sigma^m$, the total number of tasks~$\Omega$,
and the admissible team compositions are finite, so the set
of feasible nodes is finite. Dominance pruning does not remove all
representatives of any feasible solution class, and a fair selection
rule eventually expands every feasible {non-dominated} node. For a
feasible problem, at least one complete {non-dominated} node
$\nu \in \overline{\mathcal{V}}^\star$ is therefore generated in finite
time, which yields plans that satisfy all missions and all capacities.
Lastly, given the objective function~\eqref{eq:value} and enough planning time,
the node with the minimum cost as in~\eqref{eq:objective} would be returned,
which also satisfies each mission requirement at the accepting states.
    This completes the proof. }
  \end{proof}

\begin{proof}
{  \textbf{of Lemma~\ref{lemma:online}}.
Similar to the previous case of known tasks,
since each node $\nu$ explicitly tracks reachable states $\widehat{Q}_m$
and transitions $\delta^m$,
and expansion strictly enforces the capacity bounds
in \eqref{eq:capacity-bound},
the partial plans are inherently consistent with the mission specification and capability constraints.
Moreover, the receding-horizon handover during re-planning uses the final state of $H$
as the next root $\nu_0$, ensuring the execution sequence remains a valid prefix
toward the accepting sets $\{Q_F^m\}$.
Thus, the accumulated trace of all robots satisfies the mission requirement~$\varphi_m$ once an accepting state~$q_k\in Q^m_F$ is reached.
However, it is worth noting that global optimality is lost because the limited horizon~$H$
and online mission updates prevent the tree search from considering the complete set of all future tasks.
}
\end{proof}

\begin{proof}
  {\textbf{of Theorem~\ref{thm:local-correctness}}.
    In the static and known case, the task reduces to an MVRP, where the trajectory of each
robot $i \in \mathcal{N}_k$ is optimized to minimize the makespan $J_k^\ell$. For non-holonomic
robots, the problem is augmented with motion primitives $\kappa(x_{j_1}, x_{j_2})$ to ensure
feasible trajectories, with the overall objective to minimize $J_k^\ell$ under trajectory constraints.
In the static and unknown case, exploration ensures all subtasks $\mathcal{J}_k^\ell$ are discovered,
and the makespan $\textbf{max}_{i \in \mathcal{N}_k} T_i$ is minimized as new subtasks are dynamically
inserted. In the dynamic and known case, coalition updates are driven by the cost function
$\chi(\mathcal{R}_j)$ for each coalition $\mathcal{R}_j$, and the switch condition ensures that robots
only switch coalitions if it reduces the overall team cost $\chi(\overline{\mathcal{R}}_t)$. The process
converges to a locally optimal configuration, minimizing the makespan $\textbf{max}_{k \in \mathcal{K}} J(k)$.}
\end{proof}

\end{document}